\documentclass{article}

\usepackage{arxiv}

\usepackage[utf8]{inputenc} 
\usepackage[T1]{fontenc}    
\usepackage{hyperref}       
\usepackage{url}            
\usepackage{booktabs}       
\usepackage{amsfonts}       
\usepackage{nicefrac}       
\usepackage{microtype}      
\usepackage{amsthm}
\usepackage{mathtools}
\usepackage{amssymb}
\usepackage{todonotes}

\usepackage[shortlabels]{enumitem}
\usepackage{color}
\usepackage{comment}
\usepackage{subcaption}
\usepackage{graphicx}

\newtheorem{theorem}{Theorem}
\newtheorem{definition}[theorem]{Definition}
\newtheorem{proposition}[theorem]{Proposition}
\newtheorem{corollary}[theorem]{Corollary}
\newtheorem{lemma}[theorem]{Lemma}
\newtheorem{remark}[theorem]{Remark}

\numberwithin{theorem}{section}

\newcommand{\Var}{\operatorname{Var}}
\newcommand{\Lip}{\text{Lip}}
\newcommand{\range}{\text{Range}}

\newcommand{\dom}{\text{dom}}
\newcommand{\sgn}{\text{sgn}}

\DeclareMathOperator*{\argmax}{argmax} 

\newcommand{\req}[1]{Eq.\,(\ref{#1})}

\title{$(f,\Gamma)$-Divergences: Interpolating between $f$-Divergences and Integral Probability Metrics}

\author{  Jeremiah Birrell\\
  TRIPODS Institute for Theoretical Foundations of Data Science\\
  University of Massachusetts Amherst\\
  Amherst, MA 01003,  USA \\
  \texttt{birrell@math.umass.edu} \\
   \And
 Paul Dupuis\\
  Division of Applied Mathematics\\
  Brown University\\
  Providence, RI 02912, USA \\
  \texttt{dupuis@dam.brown.edu} \\
  \And
    Markos A. Katsoulakis\\
    Department of Mathematics and Statistics\\
  University of Massachusetts Amherst\\
  Amherst, MA 01003,  USA \\
  \texttt{markos@math.umass.edu} \\
     \And
 Yannis Pantazis \\
  Institute of Applied and Computational Mathematics\\
  Foundation for Research and Technology - Hellas\\
  Heraklion, GR-70013, Greece \\
  \texttt{pantazis@iacm.forth.gr}\\
\And
    Luc Rey-Bellet\\
    Department of Mathematics and Statistics\\
  University of Massachusetts Amherst\\
  Amherst, MA 01003,  USA \\
  \texttt{luc@math.umass.edu} 
}

\begin{document}
\maketitle

\begin{abstract}
We develop a rigorous and general framework for  constructing  information-theoretic divergences that subsume both $f$-divergences and integral probability metrics (IPMs),  such as  the  $1$-Wasserstein distance. We prove under which assumptions these divergences, hereafter referred to as $(f,\Gamma)$-divergences,  provide a notion of `distance' between probability measures and show that they can be expressed as a two-stage mass-redistribution/mass-transport  process. The  $(f,\Gamma)$-divergences inherit features  from IPMs,   such as   the ability  to compare distributions which are not absolutely continuous, as well as   from $f$-divergences, namely   the strict concavity of their variational representations and the ability to control heavy-tailed distributions  for particular choices of $f$. When combined, these features  establish a divergence with improved properties for estimation, statistical learning, and uncertainty quantification applications. Using statistical learning as an example, we demonstrate their advantage in training generative adversarial networks (GANs) for heavy-tailed, not-absolutely continuous sample distributions. We also show improved performance and stability over gradient-penalized Wasserstein GAN in image generation. 
\end{abstract}


\section{Introduction}

Divergences and metrics provide a notion of `distance' between multivariate probability distributions, thus allowing for comparison of models with one another and with data. Divergences are used in many theoretical and practical problems in mathematics, engineering, and the natural sciences, ranging from statistical physics, large deviations theory, uncertainty quantification and statistics to   information theory, communication theory, and machine learning.    In this work, we introduce and study what we term the $(f,\Gamma)$-divergences, denoted by $D_f^\Gamma$ and defined by the variational expression
\begin{align}
D_f^\Gamma(Q\|P) \equiv& \sup_{ g\in\Gamma}\left\{E_Q[  g]-\Lambda_f^P[g]\right\}\,,\label{eq:Df_Gamma_intro}\\
\Lambda_f^P[g]\equiv&\inf_{\nu\in\mathbb{R}}\left\{\nu+E_P[f^*( g-\nu)]\right\}\,,\label{eq:Lambda_f_def_intro}
\end{align}
where $Q$ and $P$ are probability measures, $f$ is a convex function with $f(1)=0$, $f^*$ denotes the Legendre Transform (LT) of $f$, and  $\Gamma\subset \mathcal M_b(\Omega)$ is an appropriate function space\footnote{$\mathcal M_b(\Omega)$ denotes the set of all measurable and bounded real-valued functions on $\Omega$.}. The resemblance to the variational representation of the $f$-divergence is evident (see \req{eq:Df_variational_bounded_intro} below), however, the additional optimization over shifts $\nu$ in \eqref{eq:Lambda_f_def_intro}, 
which is motivated by the Gibbs variational principle \cite{BenTal2007}, will enable the derivation of many theoretical properties of the $(f,\Gamma)$-divergence.
In the special case of the Kullback-Leibler (KL)  divergence, 
$\Lambda_f^P[g]$ is exactly the cumulant generating function that arises in  the Donsker-Varadhan variational formula \cite{Dupuis_Ellis}.
We will show that the $(f,\Gamma)$-divergences are related to, interpolate between, and inherit key properties from  both the $f$-divergences 
and   the integral probability metrics (IPMs). 
To motivate the definition in \eqref{eq:Df_Gamma_intro}, we first recall  the definition and basic properties of $f$-divergences and IPMs.

The family of $f$-divergences includes among others the KL divergence \cite{kullback1951}, the total variation distance, the $\chi^2$-divergence, the Hellinger distance, and the Jensen-Shannon divergence   \cite{Ali1966,csiszar1967}.  The $f$-divergence  between two probability measures $Q$ and $P$ induced by a convex function $f$ satisfying $f(1)=0$ is defined by
\begin{equation}\label{eq:Df_def_intro}
D_f(Q\|P)\equiv
     E_P[f(dQ/dP)]\, . 
\end{equation}
This definition assumes absolute continuity between $Q$ and $P$, $Q\ll P$, which in particular means that the support of $Q$ is included in the support of $P$.
The estimation of an $f$-divergence directly from \eqref{eq:Df_def_intro} is challenging since it requires  knowledge  of  the likelihood ratio (i.e., Radon-Nikodym derivative) $dQ/dP$, such as when working within a parametric family, or  of a reasonable approximation to $dQ/dP$, usually through histogram binning, kernel density estimation \cite{Wang2005, KandasamyEtAl:2015vm}, or through $k$-nearest neighbor approximation \cite{Wang2006}. However, parametric methods greatly restrict the  collection of allowed models, resulting in reduced expressivity,  whereas non-parametric likelihood-ratio methods do not scale efficiently with the dimension of the data \cite{KrishnamurthyEtAl:2014rd}. To address such challenges, statistical estimators which are based on variational representations of divergences have recently been  introduced  \cite{Nguyen_Full_2010,MINE_paper}.  

Variational representation formulas for divergences, often referred to as dual formulations,  convert  divergence estimation into, in principle,  an infinite-dimensional optimization problem over a function space.  A typical example of a variational representation is the LT representation of the $f$-divergence between $Q$ and $P$, given by \cite{Broniatowski,Nguyen_Full_2010}
\begin{equation}
D_f(Q\|P)
= \sup_{ g\in \mathcal{M}_b(\Omega)} \big\{E_Q[ g]-E_P[f^*( g)]\big\}\,.
\label{eq:Df_variational_bounded_intro}
\end{equation}
Such representations offer a useful mathematical tool to measure  statistical similarity between data collections as well as to build, train, and compare complex probabilistic models. The main practical advantage of variational formulas is that an explicit form of the probability distributions or  their likelihood ratio, $dQ/dP$, is not necessary. Only samples from both distributions are required since the   difference of expected values in \eqref{eq:Df_variational_bounded_intro} can be approximated by statistical averages. In practice, the infinite-dimensional function space has to be approximated or even restricted. One of the first attempts was the restriction of the function space to a reproducing kernel Hilbert space (RKHS) and the corresponding kernel-based approximation in \cite{Nguyen_Full_2010}. More recently, the optimization \eqref{eq:Df_variational_bounded_intro} has been approximated   using flexible regression models and particularly by neural networks \cite{MINE_paper}  and these techniques are widely  used in the training of generative adversarial networks (GANs) \cite{GAN,WGAN,f_GAN,wgan:gp}. Variational representations of divergences have also been used to quantify the model uncertainty in a probabilistic model (arising, e.g., from insufficient data and partial expert knowledge).  For instance, applying the $f$-divergence formula \eqref{eq:Df_variational_bounded_intro} to $cg-\nu$,  solving for $E_Q[g]$, and optimizing over $c>0$, $\nu\in\mathbb{R}$ leads to the  uncertainty quantification (UQ) bound    \cite{chowdhary_dupuis_2013,DKPP}
\begin{align}\label{eq:UQ_Df}
E_Q[g] \le \inf_{ c>0}\left\{\frac{1}{c}\Lambda_f^P[cg]+\frac{1}{c}D_f(Q\|P)\right\}\, .
\end{align}
Similarly, one can obtain a corresponding lower bound for any  quantity of interest $g \in \mathcal{M}_b(\Omega)$.
The UQ inequality \eqref{eq:UQ_Df} bounds  the uncertainty in the expectation of $g$ under an alternative model $Q$ in terms of expectations under the baseline model $P$ and the discrepancy between $Q$ and $P$ (quantified via $D_f(Q\|P)$). Further discussion of the general connection between variational characterizations of divergences and UQ can be found in \cite{Glasserman2014,AtarChowdharyDupuis,Lam2016,Breuer2016,GKRW,Dupuis2019-AAP,Dupuis:Mao:2019,birrell2020optimizing}.

Integral probability metrics are defined directly in terms of a variational formula \cite{Muller1997,sriperumbudur2009integral}, generalizing the  Kantorovich–Rubinstein variational formula for the Wasserstein metric \cite{villani2008optimal}. More specifically, they are defined by maximizing the differences of respective expected  values over a function space $\Gamma$, 
\begin{equation}
W^{\Gamma}(Q, P)
= \sup_{ g\in \Gamma} \big\{E_Q[ g]-E_P[ g]\big\}\,,
\label{eq:IPM_intro}
\end{equation}
and we refer to this object as the $\Gamma$-IPM. Despite the name, IPMs are not necessarily metrics in the mathematical sense unless further assumptions on $\Gamma$ are made. This will not be an issue for us going forward, as we are not focused on the metric property; we will be concerned with the divergence property, as defined in Section \ref{sec:notation} below. Examples of IPMs include: the total variation  metric, which is derived when the function space $\Gamma$ is the unit ball in the space of  bounded measurable functions; the Wasserstein, metric where $\Gamma$ is  the space of Lipschitz continuous functions with Lipschitz constant less than or equal to one; the Dudley metric, where the function space $\Gamma$ is the unit ball in the space of  bounded and Lipschitz continuous functions; and the maximum mean discrepancy (MMD), where $\Gamma$ is the the unit ball in a RKHS, see also \cite{Muller1997,sriperumbudur2009integral,sriperumbudur2012}. The definition of an IPM through the variational formula \eqref{eq:IPM_intro} leads to straightforward and unbiased statistical estimation algorithms \cite{sriperumbudur2012}. Furthermore, the Wasserstein metric applied to generative adversarial networks (GANs) is known to substantially improve the stability of the training process \cite{WGAN,wgan:gp}, while MMD offers one of the most reliable two-sample tests for high dimensional statistical distributions \cite{gretton2012}.

In summary, there are two fundamental mathematical ingredients involved in  variational formulas for $f$-divergences and IPMs, with both families having their own strengths and weaknesses. 
\begin{enumerate}[a)]
    \item \emph{The Objective Functional}: The objective functional in a variational representation  is the quantity being maximized, namely  $E_Q[ g]-E_P[f^*( g)]$ for the $f$-divergences and $E_Q[ g]-E_P[ g]$ for the IPMs. The former depends on $f$ and for appropriate $f$'s it is  strictly concave in $ g$, while the latter is the same for all IPMs and is linear in $ g$.  Stronger convexity/concavity properties could result in improved statistical learning, estimation, and convergence performance. 
    The ability to vary the objective functional by choosing $f$ also allows one to tailor the divergence to the data source, e.g., for heavy tailed data. 
    Finally, note that  alternative objective functionals  can yield the same divergence \cite{BenTal2007,Ruderman,MINE_paper,birrell2020optimizing}, and their careful choice  can have a substantial impact on their statistical estimation \cite{MINE_paper,Ruderman, birrell2020optimizing}. 

    \item \emph{The Function Space}: This is the space over which the objective functional is optimized.  In \eqref{eq:Df_variational_bounded_intro}, it is the same function space for all $f$-divergences, namely $\mathcal{M}_b(\Omega)$, while the choice of function space $\Gamma$ is what defines an IPM in  \eqref{eq:IPM_intro}. The choice of $\Gamma$ has a profound impact on the properties of a divergence, e.g., the ability to meaningfully compare not-absolutely continuous distributions.
\end{enumerate}

As we will show, the properties of the $(f,\Gamma)$-divergences 
can be tailored to the requirements of a particular problem through the choice of the objective functional (via $f$) and the function space  $\Gamma$. The need for such a flexible family of divergences that combines the strengths of both $f$-divergences and IPMs is motivated by problems in machine learning and UQ, where properties of the  data source or baseline model dictate the requirements on $f$ and $\Gamma$, e.g., the $f$-divergence UQ bound \eqref{eq:UQ_Df} is unable to treat structurally different  alternative models $Q$, which can easily be mutually singular with $P$, as $D_f(Q\|P)=\infty$ under a loss of absolute continuity; similar issues appear in GANs, \cite{WGAN}.

Related approaches  include the recent studies \cite{arXiv:1809.04542,NIPS2018_7771,miyato2018spectral,Bridging_fGan_WGan,NEURIPS2019_eae27d77,Dupuis:Mao:2019,2021arXiv210608929G}.  {  In \cite{miyato2018spectral} the authors studied the use of spectral normalization to impose a Lipschitz constraint on the discriminator of a GAN; this is an example of \eqref{eq:Df_Gamma_intro} with  a particular choice of function space.} {  In  \cite{Bridging_fGan_WGan}, the authors proposed   a class of objective functionals with an additional optimization layer,
aiming to bridge the gap between the variational formulas for $f$-divergences and Wasserstein metrics and applied it to adversarial training of generative models. 
However, the paper does not  provide a rigorous connection  to the Wasserstein metric, since the  function space  appearing in their main Theorem 1  cannot include a Lipschitz constraint. This is in contrast to their practical implementation in Algorithm 1, which does employ a Lipschitz constraint. Our approach bridges this gap between theory and practice, as we are able to explicitly handle Lipschitz function spaces. Finally, our approach does not require the introduction of a third neural network, no matter what the choice of $f$-divergence may be.}
On the other hand, the authors in \cite{Dupuis:Mao:2019} developed a variational formula for general function spaces in the case of the KL divergence, providing a systematic and rigorous  interpolation between KL divergence and IPMs. Definition \eqref{eq:Df_Gamma_intro} can be also viewed as a regularization of the classical $f$-divergences, and related objects  {  have also been introduced and studied in \cite{arXiv:1809.04542,NEURIPS2019_eae27d77,NIPS2018_7771,2021arXiv210608929G}. While there is some overlap with several prior works}, the aim of this paper is to provide a systematic and rigorous   development of the  $(f,\Gamma)$-divergences, focusing on a number of new properties that are potentially beneficial in learning and UQ  applications. Specifically:
\begin{enumerate}
    \item We derive conditions under which $D_f^\Gamma$ has the divergence property, and thus provide a well-defined notion of `distance' (Part 4 of Theorem \ref{thm:general_ub} and Part 4 of Theorem \ref{thm:f_div_inf_convolution}). {  One key novelty is the introduction of the object \eqref{eq:Lambda_f_def_intro} which is critical in the proof of this property.}
    \item We show that $D_f^\Gamma$ interpolates between the $f$-divergence and $\Gamma$-IPM in the sense of infimal convolutions, including existence of an optimizer (Parts 1 and 2 of Theorem \ref{thm:f_div_inf_convolution}). {  Again, \eqref{eq:Lambda_f_def_intro} plays a critical role here.}
    \item Using the infimal convolution formula, we derive a mass-redistribution/mass-transport interpretation of the $(f,\Gamma)$-divergences (Section \ref{sec:mass_transport}).
    \item  We  show that the family of  $(f,\Gamma)$-divergences includes  $f$-divergences and  $\Gamma$-IPMs in suitable asymptotic limits  (Theorem \ref{thm:limit}).
    \item The relaxation of the hard constraint $ g\in\Gamma$ in \eqref{eq:Df_Gamma_intro} to a soft-constraint penalty term is presented in Theorem \ref{thm:soft_constraint}. This is a generalization of the gradient penalty method for Wasserstein metrics  \cite{wgan:gp} to a much larger class of objective functionals and penalties and a key tool in  designing numerically efficient  implementations while still preserving the divergence property.
    \item Relaxation of the condition $\Gamma\subset \mathcal{M}_b(\Omega)$ in \eqref{eq:Df_Gamma_intro}, i.e., allowing $\Gamma$ to contain appropriate unbounded  functions, is addressed in Theorem \ref{thm:unbounded_Lip}. This is a necessary point when employing neural network estimation with unbounded activation functions.
    \item We show that the $(f,\Gamma)$-divergences inherit several properties from both  $f$-divergences and the IPMs. The primary advantage inherited from  IPMs  is the ability  to compare distributions which are not absolute continuous. The primary advantages inherited from the $f$-divergence are the strict concavity of the objective functional with respect to the test function, $g$, and the ability to compare heavy-tailed distributions (Section \ref{sec:examples}). 
\end{enumerate}
When combined, these advantages  establish a divergence with better convergence and estimation properties. We numerically demonstrate these merits in the training of GANs. In Section \ref{sec:submanifold_ex}, we show that the proposed divergence is capable of adversarial learning of lower dimensional sub-manifold distributions with heavy tails. In this example,  both $f$-GAN \cite{f_GAN} and Wasserstein GAN with  gradient penalty  (WGAN-GP) \cite{wgan:gp}  fail to converge or perform very poorly. {   Furthermore, in Section \ref{ex:C10} we present improvements over WGAN-GP and WGAN with spectral-normalization (WGAN-SN) \cite{miyato2018spectral}, as measured by  the inception score \cite{salimans2016improved}  and  FID score \cite{10.5555/3295222.3295408}  (two standard performance measures), } in real datasets and particularly in CIFAR-10 \cite{krizhevsky2009learning} image generation. 
Interestingly, the training stability is significantly enhanced when using the proposed $(f,\Gamma)$-divergence, as compared to WGAN,  which is evident from the fact that increasing the  learning rate (i.e., stochastic gradient descent step size) eventually results in the collapse of WGAN but has comparatively little impact on our newly proposed method. We conjecture that this is due  to the strict concavity of the objective functional of the $(f,\Gamma)$-divergence. We refer to  these new proposed GANs which are based on  $(f,\Gamma)$-divergences as  $(f,\Gamma)$-GANs.

The organization of the paper is as follows. The key properties of the  $(f,\Gamma)$-divergences 
are presented in Section \ref{sec:gen_f_div}. The mass-redistribution/mass-transport interpretation of the $(f,\Gamma)$-divergences is discussed  in Section \ref{sec:mass_transport}. Section \ref{sec:soft_constraint} develops a general theory of soft-constraint penalization.  Section \ref{sec:unbounded_extension} provides conditions under which the function space $\Gamma$ can be expanded to contain unbounded functions. The application of the $(f,\Gamma$)-divergences in adversarial generative modelling is presented in Section \ref{sec:examples}.  We conclude the paper and discuss plans for future work in Section \ref{sec:concl}. Finally, detailed proofs can be found in the appendices.

\section{Construction and Properties of the $(f,\Gamma)$-Divergences}
\label{sec:gen_f_div}
In this section, we will derive the divergence property for the $(f,\Gamma)$-divergences  
and show that they interpolate between $f$-divergences and IPMs as it is described in our main result (Theorem \ref{thm:f_div_inf_convolution}).  First we introduce  our notation and recall some important properties of the  $f$-divergences.

\subsection{Notation}\label{sec:notation}
For the remainder of the paper $(\Omega,\mathcal{M})$ will denote a measurable space, $\mathcal{M}(\Omega)$ will be the set of all measurable real-valued functions on $\Omega$, $\mathcal{M}_b(\Omega)$ will denote the subspace of bounded measurable functions,  $\mathcal{P}(\Omega)$ will denote the space of probability measures on $(\Omega,\mathcal{M})$, and $M(\Omega)$ will be the set of finite signed measures on $(\Omega,\mathcal{M})$. A subset $\Psi\subset \mathcal{M}_b(\Omega)$ will be called {\bf $\mathcal{P}(\Omega)$-determining} if for all $Q,P\in\mathcal{P}(\Omega)$, $\int \psi dQ=\int \psi dP$ for all $\psi\in \Psi$ implies $Q=P$. The integral (expectation) of $ g$ with respect to $P\in\mathcal{P}(\Omega)$ will also be written as $E_P[ g]$.   We say that a map $D:\mathcal{P}(\Omega)\times\mathcal{P}(\Omega)\to[0,\infty]$ has the {\bf divergence property} if $D(Q,P)=0$ if and only if $Q=P$; such maps provide a notion of `distance' between probability measures.
 \begin{remark}
 We emphasize that despite the standard (but potentially confusing)  terminology, not all $f$-divergences have the divergence property; see Section \ref{sec:f_div_background} below for further information. Going forward, we will continue to distinguish between what we call a divergence and the divergence property. 
 \end{remark}

$(S,d)$ will denote a complete separable metric space (i.e., a Polish space), $C(S)$ will denote the space of continuous real-valued functions on $S$, and $C_b(S)$ will be the subspace of bounded continuous functions.  $\Lip(S)$ will denote the space of  Lipschitz functions on $S$, $\Lip_b(S)$ the subspace of bounded Lipschitz functions,  and for $L> 0$ we let $\Lip_b^L(S)$  denote the subspace consisting of bounded $L$-Lipschitz functions (i.e., functions having Lipschitz constant $L$). $\mathcal{P}(S)$ will denote the space of Borel probability measures on $S$ equipped with the Prohorov metric, thus making $\mathcal{P}(S)$ a Polish space. Recall that the Prohorov metric topology on $\mathcal{P}(S)$ is the same as the weak topology induced by the set of functions $\pi_ g:P\mapsto E_P[ g]$, $ g\in C_b(S)$.  For $\mu\in M(S)$ (finite signed Borel measures on $S$) we define $\tau_\mu:C_b(S)\to\mathbb{R}$ by $\tau_\mu(g)=\int gd\mu$ and we let $\mathcal{T}=\{\tau_\mu:\mu\in M(S)\}$.  $\mathcal{T}$ is a separating vector space of  linear functionals on $C_b(S)$.  We equip $C_b(S)$ with the weak topology from $\mathcal{T}$ (i.e., the weakest topology on $C_b(S)$ for which every $\tau\in \mathcal{T}$ is continuous), which makes $C_b(S)$ a locally convex topological vector space with dual space $C_b(S)^*=\mathcal{T}$ \cite[Theorem 3.10]{rudin2006functional}. We will let $\overline{\mathbb{R}}\equiv\mathbb{R}\cup\{-\infty,\infty\}$ denote the extended reals. Given a  function $h:\mathbb{R}\to\overline{\mathbb{R}}$, its Legendre transform is defined by $h^*(y)\equiv \sup_{x\in\mathbb{R}}\{yx-h(x)\}$. Recall that if $h:\mathbb{R}\to(-\infty,\infty]$ is convex and lower semicontinuous (LSC) then $(h^*)^*=h$ \cite[Theorem 2.3.5]{bot2009duality}. Also recall that if $h$ is convex and finite on $(a,b)$ then the left and right derivatives, which we denote by $h^\prime_-(x)$ and $h^\prime_+(x)$ respectively, exist for all $x\in(a,b)$ \cite[Chapter 1]{roberts1974convex}. We will denote the closure of a set $A$ by $\overline{A}$ and its interior by $A^o$. {  Finally,  we include in Table \ref{table:notation} a list of important notations, some of which are defined elsewhere in the manuscript, with corresponding  references.}

\begin{table}
\centering
{ 
\caption{List of main symbols used throughout the manuscript.}
\begin{tabular}{||c c c ||} 
 \hline
 Notation & Description   & Reference  \\  \hline
 \hline
 $(\Omega,\mathcal{M})$ & Measurable space  & Section \ref{sec:notation}\\ \hline
 $(S,d)$ & Metric  space  & Section \ref{sec:notation}\\ \hline
 $M(\Omega)$ \& $M(S)$ & Spaces of finite signed measures &  Section \ref{sec:notation}  \\  \hline
 $\mathcal P(\Omega)$ \& $\mathcal P(S)$ & Spaces of probability measures & Section \ref{sec:notation} \\ \hline
 $\mathcal M(\Omega)$ \& $\mathcal M_b(\Omega)$ & Spaces of measurable real-valued  functions &  Section \ref{sec:notation}  \\  \hline 
 $C(S)$ \& $C_b(S)$ & Spaces of continuous real-valued functions & Section \ref{sec:notation}  \\  \hline
 $\Lip(S)$ \& $\Lip_b(S)$ & Spaces of Lipschitz continuous functions & Section \ref{sec:notation}  \\  \hline
 $P$, $Q$ & Probability distributions/measures & Section \ref{sec:notation} \\ \hline
 $f$ & Convex function on $\mathbb{R}$ & Definition \ref{def:F_1}\\
 \hline
 $\mathcal F_1(a,b)$ & Set of convex functions  & Definition \ref{def:F_1}\\
 \hline
 $D_f$ & $f$-Divergence &  \req{eq:Df_def} \\ 
 \hline
 $\Lambda_f^P$ & Generalization of the cumulant generating function & \req{eq:Lambda_f_def}\\
 \hline
 $\Gamma$ & Test function space & Definition \ref{def:f_Gamma_Div} \\
 \hline
$D_f^\Gamma$ & $(f,\Gamma)$-Divergence &  \req{eq:gen_f_def} \\ 
 \hline 
$W^\Gamma$ & $\Gamma$-Integral probability metric & \req{eq:gen_wasserstein} \\ 
 \hline
 $W^\rho$ & Gradient-penalty Wasserstein divergence & \req{W_rho} \\ 
 \hline
 $D^L_\alpha$ & Lipschitz $\alpha$-divergence & \req{eq:f_alpha_nu} - \eqref{eq:Df_alpha_L2}  \\ 
 \hline
\end{tabular}
\label{table:notation}
}
\end{table}

\subsection{Background on $f$-Divergences}\label{sec:f_div_background}
The $f$-divergences are constructed using functions of the following form:
\begin{definition}\label{def:F_1}
For  $a,b$ with $-\infty\leq a<1<b\leq\infty$  we define   $\mathcal{F}_1(a,b)$ to be the set of convex functions $f:(a,b)\to\mathbb{R}$ with $f(1)=0$. For $f\in\mathcal{F}_1(a,b)$, if $b$ is finite we extend the definition of $f$ by $f(b)\equiv\lim_{x\nearrow b}f(x)$. Similarly, if $a$ is finite we define  $f(a)\equiv\lim_{x\searrow a}f(x)$ (convexity implies  these limits exist in $(-\infty,\infty]$). Finally, extend $f$ to $x\not\in [a,b]$ by $f(x)=\infty$.  The resulting function $f:\mathbb{R}\to(-\infty,\infty]$ is convex and  LSC.
\end{definition}
 The  $f$-divergences are then defined as follows:
 \begin{definition}\label{def:f_div}
 For $f\in\mathcal{F}_1(a,b)$ and $Q,P\in\mathcal{P}(\Omega)$ the corresponding $f$-divergence is defined by
 \begin{align}\label{eq:Df_def}
D_f(Q\|P)\equiv \begin{cases} 
     E_P[f(dQ/dP)], & Q\ll P\\
      \infty, &Q\not\ll P\,.
   \end{cases}
\end{align}
\end{definition}
A number of important properties of $f$-divergences are collected in Appendix \ref{app:f_div}.  An $f$-divergence defines a notion of `distance' between probability measures, as is made precise by the following divergence property: $D_f(Q\|P)\geq 0$ for all $f\in\mathcal{F}_1(a,b)$ and if $f$ is furthermore strictly convex at $1$ (i.e., $f$ is not affine on any neighborhood of $1$) then $D_f(Q\|P)=0$ if and only if $Q=P$.  However, the $f$-divergences are generally not probability metrics.   Our primary examples will be the KL divergence  and the family of $\alpha$-divergences, which are constructed from the following functions:
\begin{align}\label{eq:f_alpha_def}
f_{KL}(x)\equiv x\log(x)\in\mathcal{F}_1(0,\infty)\,,\,\,\,\,\,
    f_\alpha(x)=\frac{x^\alpha-1}{\alpha(\alpha-1)}\in\mathcal{F}_1(0,\infty)\,,\text{ where $\alpha>0$, $\alpha\neq 1$}\,.
\end{align}
See \cite[Table 1]{f_GAN} for further examples.

Key to our work are a  pair of variational formulas that  relate the $f$-divergence to  the functional
\begin{align}\label{eq:Lambda_f_def}
    \Lambda_f^P[g]\equiv\inf_{\nu\in\mathbb{R}}\{\nu+E_P[f^*( g-\nu)]\}\,,\,\,\,\,\,g\in\mathcal{M}_b(\Omega)\,,
\end{align}
As we will see, $\Lambda_f^P$ takes the place of the cumulant generating function when one generalizes from the KL divergence to $f$-divergences. The first of the following formulas expresses $D_f$ as an infinite-dimensional convex conjugate of $\Lambda_f^P$ and the second is the dual variational formula:
\begin{enumerate}
\item Let   $f\in\mathcal{F}_1(a,b)$ and  $Q,P\in\mathcal{P}(\Omega)$. Then,
\begin{align}
D_f(Q\|P)=&\sup_{ g\in \mathcal{M}_b(\Omega)}\{ E_Q[ g]-E_P[f^*( g)]\}\label{eq:Df_var_formula}\\
=&\sup_{ g\in \mathcal{M}_b(\Omega)}\{ E_Q[ g]-\Lambda_f^P[g]\}\, ,\label{eq:Df_var_formula_nu}
\end{align}
{  where the second equality follows from \eqref{eq:Lambda_f_def} and \eqref{eq:Df_var_formula} due to the invariance of $\mathcal{M}_b(\Omega)$ under the shift map $ g\mapsto g-\nu$ for $\nu\in\mathbb{R}$; see also Proposition \ref{prop:Df_var_formula}.} 
\item  Let $f\in\mathcal{F}_1(a,b)$ with $a\geq 0$,  $P\in\mathcal{P}(\Omega)$, and  $ g\in \mathcal{M}_b(\Omega)$. Then
we can rewrite $\Lambda_f^P[g]$ as
\begin{align}\label{eq:Df_Gibbs_var_formula}
\Lambda_f^P[g]=\sup_{Q\in\mathcal{P}(\Omega): D_f(Q\|P)<\infty}\{E_Q[ g]-D_f(Q\|P)\}\,.
\end{align}
\end{enumerate}
\begin{remark}
$f$-divergences can alternatively be defined in terms of the densities of $Q$ and $P$ with respect to some common dominating measure  \cite{LieseVajda}.  This definition agrees with \req{eq:Df_def} when $Q\ll P$ but in some cases the definition in \cite{LieseVajda} leads to a finite value even when $Q\not\ll P$.  In this paper, we use the definition \eqref{eq:Df_def} because it satisfies the variational formula \eqref{eq:Df_var_formula}, even when $Q\not\ll P$ (see the proof of Proposition \ref{prop:Df_var_formula}), as well as the dual formula \eqref{eq:Df_Gibbs_var_formula}.
\end{remark}
When $f=f_{KL}$ 
it is straightforward to show that $\Lambda_f^P$  becomes the cumulant generating function,
\begin{align}\label{eq:Lambda_KL}
      \Lambda_{f_{KL}}^P[g]=\log E_P[e^ g]\,,
\end{align} 
and \req{eq:Df_var_formula_nu} becomes the Donsker-Varadhan variational formula \cite[Appendix C.2]{Dupuis_Ellis}. Subsequently, \req{eq:Df_Gibbs_var_formula} becomes the  Gibbs variational formula \cite[Proposition 1.4.2]{Dupuis_Ellis}. For this reason, we will call \eqref{eq:Df_Gibbs_var_formula} the Gibbs variational formula for $f$-divergences. Versions of \req{eq:Df_var_formula} were proven in \cite{Broniatowski,Nguyen_Full_2010}; we provide an elementary proof in Theorem \ref{prop:Df_var_formula} of Appendix \ref{app:f_div} for completeness.  \req{eq:Df_var_formula_nu} is implicitly found in \cite[Theorem 1]{Ruderman}; see \cite{birrell2020optimizing} for further discussion of this relationship. More specifically, \cite{Ruderman,birrell2020optimizing} show that when $a\geq 0$ the representation in \eqref{eq:Df_var_formula} arises from convex duality over the space of finite positive measures  while \eqref{eq:Df_var_formula_nu} arises from convex duality over the space of probability measures. 
On a metric space $S$, the optimizations in \eqref{eq:Df_var_formula} - \eqref{eq:Df_var_formula_nu} can be restricted to $C_b(S)$ via the application of Lusin’s Theorem (see Corollary \ref{cor:Df_LSC}). The dual formula \eqref{eq:Df_Gibbs_var_formula} was proven in \cite{BenTal2007} and is also implicitly contained in \cite[Eq. (5)]{Ruderman} (we will require a  generalization that also covers the case $a<0$; see Proposition \ref{prop:Gibbs_M1}).  Under appropriate assumptions \cite[Theorem 4.4]{Broniatowski} the   optimizer of   \eqref{eq:Df_var_formula} is given by
\begin{align}\label{eq:Df_optimizer}
    g_*=f^\prime(dQ/dP)\,.
\end{align}
The definition in \eqref{eq:Df_def} does not depend on the value of $f(x)$ for $x<0$ and it is invariant under the transformation $f\mapsto f_c$ where $f_c(x)=f(x)+c(x-1)$, $c\in\mathbb{R}$. However, the objective functionals in the variational formulas \eqref{eq:Df_var_formula} and \eqref{eq:Df_var_formula_nu} can depend on these choices due to the presence of $f^*$. They both depend on the definition of $f(x)$ for $x<0$. The identity $f_c^*(y)=f^*(y-c)+c$ implies that  the objective functional in \eqref{eq:Df_var_formula} depends on the choice of $c$ but the objective functional in \req{eq:Df_var_formula_nu} does not.  Substituting $f_c$ into \req{eq:Df_var_formula} and then taking the supremum over $c\in\mathbb{R}$ is another way to derive \req{eq:Df_var_formula_nu}, thus providing additional motivation for the introduction of $\Lambda_f^P$.

\subsection{{ Definition and General Properties of the $(f,\Gamma)$-Divergences}}\label{sec:gen_f_div_thm}
Motivated by \req{eq:Df_var_formula_nu} - \eqref{eq:Df_Gibbs_var_formula}, by working with subsets of test functions   $\Gamma\subset \mathcal{M}_b(\Omega)$ we can construct a new family of so-called $(f,\Gamma)$-divergences whose convex conjugates at $g\in\Gamma$ equal $\Lambda_f^P[g]$  and that have  variational characterizations akin to \req{eq:Df_var_formula_nu}.  This is an extension of the ideas  in \cite{Dupuis:Mao:2019}, which studied generalizations of the  KL-divergence.  \emph{The identification of $\Lambda_f^P$ as the proper replacement for the cumulant generating function is the  key new insight required to extend from the KL case to general $f$}. Specifically, we make the following definition:
\begin{definition}\label{def:f_Gamma_Div}
Let $f\in\mathcal{F}_1(a,b)$ and  $\Gamma\subset \mathcal{M}_b(\Omega)$ be nonempty. For $Q,P\in\mathcal{P}(\Omega)$  we define the $(f,\Gamma)$-divergence by 
\begin{align}\label{eq:gen_f_def}
D_f^\Gamma(Q\|P)\equiv\sup_{ g\in\Gamma}\left\{E_Q[  g]-\Lambda_f^P[g]\right\}\,,
\end{align}
where $\Lambda_f^P$ was defined in \req{eq:Lambda_f_def},
and we define the $\Gamma$-IPM by
\begin{align}\label{eq:gen_wasserstein}
W^\Gamma(Q,P)\equiv\sup_{ g\in\Gamma}\left\{E_Q[ g ]-E_P[ g]\right\}\,.
\end{align}
\end{definition}
When we want to emphasize the distinction between $D_f(Q\|P)$ and $D_f^\Gamma(Q\|P)$ we will refer to the former as a classical $f$-divergence. When $f$ corresponds to the KL-divergence (see \req{eq:f_alpha_def}) we write $R(Q\|P)$ and $R^\Gamma(Q\|P)$ in place of $D_f(Q\|P)$ and $D_f^\Gamma(Q\|P)$, respectively.

The definition \eqref{eq:gen_f_def} is an infinite-dimensional  convex conjugate, akin to \req{eq:Df_var_formula_nu}. From \eqref{eq:Df_var_formula_nu}, we see that $D_f=D_f^\Gamma$ when $\Gamma=\mathcal{M}_b(\Omega)$ or, on a metric space $S$ (and for appropriate $f$'s), when $\Gamma=C_b(S)$ (see Corollary \ref{cor:Df_LSC} and Remark \ref{remark:Df_nu_shift}). The  $W^\Gamma$'s are generalizations of the classical Wasserstein metric on a metric space, which is obtained by setting $\Gamma=\Lip_b^1(S)$. Neither $W^\Gamma$ nor $D_f^\Gamma$   necessarily have the divergence property, however,  {  our main results present  conditions which do imply the divergence property. As we will see, the use of $\Lambda_f^P$ in \eqref{eq:gen_f_def} is crucial in our proof of the divergence property (see  Theorem \ref{thm:general_ub}), as well as in our derivation of the infimal convolution formula (see Theorem  \ref{thm:f_div_inf_convolution}). } 

One can alternatively write the $(f,\Gamma)$-divergence as
\begin{align}\label{eq:Df_Gamma_def2}
    D_f^\Gamma(Q\|P)=\sup_{ g\in\Gamma,\nu\in\mathbb{R}}\left\{E_Q[  g-\nu]-E_P[f^*( g-\nu)]\right\}\,.
\end{align}
This formulation is useful when computing a numerical approximation to $D_f^\Gamma(Q\|P)$. {  It shows that  $\Lambda_f^P$ in \eqref{eq:gen_f_def} does not need to be computed separately; one can formulate the computation as a single optimization problem, incorporating one additional $1$-dimensional parameter.} In addition, if $\Gamma$ is closed under the shift transformations $ g\mapsto g-\nu$, $\nu\in\mathbb{R}$ then one can write
\begin{align}\label{eq:Df_Gamma_no_shift}
    D_f^\Gamma(Q\|P)=\sup_{ g\in\Gamma}\left\{E_Q[  g]-E_P[f^*( g)]\right\}\,,
\end{align}
thus arriving at the objects defined in \cite{arXiv:1809.04542,NEURIPS2019_eae27d77,NIPS2018_7771}. 
In the KL case, one can simplify \req{eq:gen_f_def} by using \eqref{eq:Lambda_KL},
\begin{align}\label{eq:R_Gamma}
R^\Gamma(Q\|P)=\sup_{ g\in\Gamma}\left\{E_Q[  g]-\log E_P[e^{ g}]\}\right\}\,,
\end{align}
which results in the special case studied in \cite{Dupuis:Mao:2019}.

{ 
Several of our results will require us to work on a metric space (see Section \ref{sec:polish}), but first we present several properties that hold more generally.  In the following theorem we derive a   dual   variational formula to \eqref{eq:gen_f_def}, which shows that if $ g\in\Gamma$ then \req{eq:Df_Gibbs_var_formula} holds with $D_f$ replaced by $D_f^\Gamma$. This lends further credence to the definition \eqref{eq:gen_f_def} and its use of $\Lambda_f^P$.}
\begin{theorem}
\label{thm:gen_fdiv_dual_var}
Let $f\in\mathcal{F}_1(a,b)$ where $a\geq 0$, $P\in\mathcal{P}(\Omega)$, and $\Gamma\subset \mathcal{M}_b(\Omega)$ be nonempty.  For $g\in\Gamma$ we have
\begin{align}\label{eq:Gibbs_VF_Phi_P}
(D_f^\Gamma)^*( g;P)\equiv \sup_{Q\in \mathcal{P}(\Omega)}\{E_Q[ g] -D_f^\Gamma(Q\|P)\}=\Lambda_f^P[g]\,.
\end{align}
\end{theorem}
\begin{remark}
We refer to Theorem \ref{thm:gen_fdiv_dual_var_app} in Appendix \ref{app:proofs} for the proof.  While most cases of interest like \req{eq:f_alpha_def} have $a\geq 0$, we also cover the case $a< 0$ in Theorem \ref{thm:gen_fdiv_dual_var_app}.
\end{remark}
Theorem \ref{thm:gen_fdiv_dual_var}  establishes $D_f^\Gamma$ as  a natural  generalization of $D_f$ when $\Gamma$ is used as the test-function space, generalizing the dual formula \eqref{eq:Df_Gibbs_var_formula} for $f$-divegences obtained in \cite{BenTal2007,Ruderman}. {  Next we show that the $D_f^\Gamma$ is bounded above by both $D_f$ and $W_\Gamma$. This fact allows the $(f,\Gamma)$-divergences to inherit many useful properties from both $f$-divergences and IPMs; see the examples in Section \ref{sec:examples}. We also give conditions under which $D^\Gamma_f$ has the divergence property and thus   provides a notion of `distance' between probability measures. This, along with Theorem \ref{thm:f_div_inf_convolution} below, constitute the main theoretical results of this paper. The proof of Theorem \ref{thm:general_ub} can be found in Theorem \ref{thm:general_ub_app} of Appendix \ref{app:proofs}.
\begin{theorem}\label{thm:general_ub}
Let $f\in\mathcal{F}_1(a,b)$, $\Gamma\subset\mathcal{M}_b(\Omega)$ be nonempty, and $Q,P\in\mathcal{P}(\Omega)$. 
\begin{enumerate}
    \item 
\begin{align}\label{eq:inf_conv_ineq}
    D_f^\Gamma(Q\|P)\leq \inf_{\eta\in\mathcal{P}(\Omega)}\{D_f(\eta\|P)+W^\Gamma(Q,\eta)\}\,.
\end{align}
In particular, $D_f^\Gamma(Q\|P)\leq \min\{D_f(Q\|P),W^\Gamma(Q,P)\}$.
\item The map $(Q,P)\in\mathcal{P}(S)\times\mathcal{P}(S)\mapsto D_f^\Gamma(Q\|P)$ is convex. 
\item If there exists $c_0\in \Gamma\cap\mathbb{R}$ then $D_f^\Gamma(Q\|P)\geq 0$.
\item Suppose $f$ and $\Gamma$  satisfy the following:
\begin{enumerate}
\item There exist a nonempty set $\Psi\subset\Gamma$ with the following properties:
\begin{enumerate}
\item $\Psi$ is $\mathcal{P}(\Omega)$-determining.
\item For all $\psi\in\Psi$  there exists $c_0\in\mathbb{R}$, $\epsilon_0>0$ such that $c_0+\epsilon \psi\in\Gamma$ for all $|\epsilon|<\epsilon_0$.
\end{enumerate}
\item $f$ is strictly convex on a neighborhood of $1$.
\item $f^*$ is finite and $C^1$ on a neighborhood of $\nu_0\equiv f_+^\prime(1)$.
\end{enumerate}
Then:
\begin{enumerate}[label=(\roman*)]
\item $D_f^{\Gamma}$ has the divergence property.
\item $W^{\Gamma}$ has the divergence property.
\end{enumerate}
\end{enumerate}
\end{theorem}
\begin{remark}
Under stronger assumptions one can show that \req{eq:inf_conv_ineq} is in fact an equality; see Theorem \ref{thm:f_div_inf_convolution} below.
\end{remark}
\begin{remark}
Assumptions 4(b) and 4(c) hold, for instance, if $f$ is strictly convex on $(a,b)$ and $\nu_0\in\{f^*<\infty\}^o$ (see Theorem 26.3 in \cite{rockafellar1970convex}).
\end{remark}
\req{eq:inf_conv_ineq} implies the following upper bound on $D^\Gamma_f$:
\begin{corollary}[Upper Bounds]\label{cor:upper_bound}
Let $\mathcal{U}\subset \mathcal{P}(\Omega)$.  Then
\begin{align}
    D_f^\Gamma(Q\|P)\leq \inf_{\eta\in\mathcal{U}}\{D_f(\eta\|P)+W^\Gamma(Q,\eta)\}\,.
\end{align}
For instance, $\mathcal{U}$ could be a pushforward family, i.e., the distributions of $h_\theta(X)$, $\theta\in \Theta$ where $h_\theta$ are $\Omega$-valued measurable maps and $X$ is a random quantity. Such families are used in GANs; see Section \ref{sec:examples}.
\end{corollary}
{\bf Examples of $P(\Omega)$-determining sets:}
\begin{enumerate}
    \item Exponentials, $e^{c\cdot x}$, $c\in\mathbb{R}^n$, i.e., the moment generating function; see Section 30 in \cite{billingsley2012probability}.
    \item The set of $1$-Lipschitz functions, $g$, on a metric space with $\|g\|_\infty\leq 1$. This follows from the Portmanteau Theorem (see, e.g., Theorem 2.1 in \cite{billingsley2013convergence}).
    \item The unit ball of a reproducing kernel Hilbert space (RKHS), under appropriate assumptions (see \cite{JMLR:v12:sriperumbudur11a}).
    \item The set of ReLU neural networks. This follows from the universal approximation theorem \cite{Cybenko1989} and also applies to other activation functions, e.g., sigmoid.
    \item      The set of ReLU neural networks with spectral normalization  \cite{miyato2018spectral}.
\end{enumerate}
Several of these classes of functions have been utilized in existing methods; see Table \ref{tab:related_work} below. Our  examples in Section \ref{sec:examples} will utilize Lipschitz functions and ReLU neural networks, including spectral normalization in Section \ref{sec:SN}.
\begin{remark}
Note that it is a well-known result that polynomials do not constitute a $\mathcal{P}(\Omega)$-determining set; there exist distinct measures that agree on all moments. 
\end{remark}
\begin{remark}
Depending on the domain, several of the above examples of $P(\Omega)$-determining sets consist of unbounded functions. To fit them into our framework it generally suffices to work with truncated versions of these functions; we refer to Section \ref{sec:unbounded_extension} for a detailed discussion.
\end{remark}
}

\subsection{$(f,\Gamma)$-Divergences on Polish Spaces}\label{sec:polish}
When working on a Polish space, $S$, and under further assumptions on  $f$ and $\Gamma$, we are able to show that $D_f^\Gamma$  interpolates between the classical $f$-divergence, $D_f$, and the $\Gamma$-IPM, $W^\Gamma$. At various points, we will require $f$ and $\Gamma$ to have the following properties:
\begin{definition}\label{def:admissible}
We will call  $f\in\mathcal{F}_1(a,b)$ {\bf admissible} if  
$\{f^*<\infty\}=\mathbb{R}$ and $\lim_{y\to-\infty}f^*(y)<\infty$ (note that this limit always exists by convexity).  If $f$ is also strictly convex at $1$ then we will call $f$  {\bf strictly admissible}. We will call $\Gamma\subset C_b(S)$ {\bf admissible} if $0\in\Gamma$, $\Gamma$ is convex, and $\Gamma$ is closed  in the weak topology generated by the maps $\tau_\mu$, $\mu\in M(S)$ (see Section \ref{sec:notation}).    $\Gamma$ will be called {\bf strictly admissible} if it also satisfies the following property: There exists a $\mathcal{P}(S)$-determining set $\Psi\subset C_b(S)$ such that for all $\psi\in\Psi$ there exists $c\in\mathbb{R}$, $\epsilon>0$ such that $c\pm\epsilon \psi\in\Gamma$.
\end{definition}
Our main result, Theorem \ref{thm:f_div_inf_convolution}, will require admissibility of both $f$ and $\Gamma$. The functions $f_{KL}$ and $f_\alpha$, $\alpha>1$, defined in \req{eq:f_alpha_def}, are strictly admissible but $f_\alpha$, $\alpha\in(0,1)$ is not admissible {  (however, Theorem \ref{thm:general_ub} above does apply to $f_\alpha$ for $\alpha\in(0,1)$)}.  The admissibility requirements that $\Gamma$ be convex and closed will let us express $D_f^\Gamma$  as the infinite-dimensional convex conjugate of a convex and LSC functional. This will allow us to analyze $D_f^\Gamma$ using tools from convex analysis.  Strict admissibility will be key in proving the divergence property for both $W^\Gamma$ and $D_f^\Gamma$.  

{\bf Examples of strictly admissible $\Gamma$:}
\begin{enumerate}
    \item $\Gamma=C_b(S)$, which leads to the classical $f$-divergences.
    \item $\Gamma=\Lip_b^1(S)$, i.e. all bounded 1-Lipschitz functions, which leads to  generalizations of the Wasserstein metric.
    \item $\Gamma=\{g\in C_b(S):|g|\leq 1\}$, which leads to generalizations of the total variation metric.
        \item $\Gamma=\{g\in \Lip^1_b(S):|g|\leq 1\}$, which leads to generalizations of the Dudley metric.
        \item {  $\Gamma=\{g\in X:\|g\|_X\leq 1\}$, the unit ball in a RKHS $X\subset C_b(S)$ (under appropriate assumptions given in Lemma \ref{lemma:RKHS}). This yields a generalization of MMD and is also related to the recent KL - MMD interpolation method in \cite{2021arXiv210608929G}; the latter employs a soft constraint rather than working on the RKHS unit ball and is based on the representation \eqref{eq:Df_var_formula} instead of \eqref{eq:Df_var_formula_nu}.}
\end{enumerate}
Note that the first two examples are shift invariant (hence \req{eq:Df_Gamma_no_shift} is applicable) while the latter three are not.

{ 
We are now ready to present the second key theorem in this paper, where  we derive the infimal convolution representation of $D_f^\Gamma$ and provide alternative (to Theorem \ref{thm:general_ub}) conditions that ensure $D_f^\Gamma$ possesses the divergence property. The proof can be found in Appendix \ref{app:proofs}, Theorem \ref{thm:f_div_inf_convolution_app}.}
\begin{theorem}
\label{thm:f_div_inf_convolution}
 Suppose   $f$ and $\Gamma$ are admissible. For $Q,P\in\mathcal{P}(S)$ let $D^\Gamma_f(Q\|P)$ be defined by \eqref{eq:gen_f_def}  and let $W^\Gamma(Q,P)$ be defined as in \eqref{eq:gen_wasserstein}. These have the following properties:
\begin{enumerate}
\item Infimal Convolution Formula: 
\begin{align}\label{eq:inf_conv}
D_f^\Gamma(Q\|P)=\inf_{\eta\in \mathcal{P}(S)}\{D_f(\eta\|P)+W^\Gamma(Q,\eta)\}\,.
\end{align}
In particular, $0\leq D_f^\Gamma(Q\|P)\leq \min\{D_f(Q\|P),W^\Gamma(Q,P)\}$. 
\item If $D_f^\Gamma(Q\|P)<\infty$ then there exists $\eta_*\in\mathcal{P}(S)$ such that
\begin{align}\label{eq:inf_conv_existence}
D_f^\Gamma(Q\|P)=D_f(\eta_*\|P)+W^\Gamma(Q,\eta_*)\,.
\end{align}
If $f$ is strictly convex then there is a unique such $\eta_*$.
\item Divergence Property for $W^\Gamma$: If $\Gamma$ is strictly admissible then $W^\Gamma$ has the divergence property. 
\item Divergence Property for $D^\Gamma_f$:  If $f$ and $\Gamma$ are both strictly admissible then $D_f^\Gamma$ has the divergence property.
\end{enumerate}
\end{theorem}

\begin{remark}
If $a\geq 0$ in Definition \ref{def:F_1} then $f^*$ is nondecreasing and so the  condition  $\lim_{y\to-\infty}f^*(y)<\infty$  is satisfied; see Lemma \ref{lemma:f_star_inc}. In many cases, the divergence property for $D_f^\Gamma$ still holds even if one or both  of the conditions $\lim_{y\to-\infty}f^*(y)<\infty$, $\{f^*<\infty\}=\mathbb{R}$ are violated and also  under relaxed conditions on $\Gamma$; this was  shown in Theorem \ref{thm:general_ub}.
\end{remark}
The infimal convolution formula \eqref{eq:inf_conv} - \eqref{eq:inf_conv_existence} gives one precise sense in which the $(f,\Gamma)$-divergence variationally interpolates between the $\Gamma$-IPM, $W^\Gamma$,  and the classical $f$-divergence, $D_f$. It is a generalization of the results in \cite{NIPS2018_7771,Dupuis:Mao:2019}, the former assuming compactly supported measures and the latter covering the KL case. {  We end this subsection by referring the reader to Table~\ref{tab:related_work}, which lists related works and connections to our general framework.}

\begin{table}
\centering
{ 
\caption{{  Summarizing how our main Theorems extend or relate to certain existing methods.  Our theory either applies directly to the cited methods or motivates the construction of closely related interpolation and/or regularization methods that are based on \eqref{eq:Df_var_formula_nu}.}}
\begin{tabular}{ |p{2.7cm}||p{3cm}|p{4.5cm}|p{4.5cm}|  }
 \hline
 \multicolumn{4}{|c|}{Extension of \& connections to related work} \\
 \hline
 Related Paper & Function Space $\Gamma$ & Objective Functional & Relevant Theorems\\
 \hline
 Goodfellow et al., \cite{GAN}& Neural networks & JS divergence using \eqref{eq:Df_var_formula} &Theorem~\ref{thm:general_ub}\\
 Nowozin et al., \cite{f_GAN}&   Neural networks  &  f-divergence using \eqref{eq:Df_var_formula}  &
 Theorem~\ref{thm:general_ub}\\
 Belghazi et al., \cite{MINE_paper}&  Neural networks   & KL-div. using \eqref{eq:Df_var_formula} \&
 \eqref{eq:Df_var_formula_nu}& Theorem~\ref{thm:general_ub}\\
 Miyato et al., \cite{miyato2018spectral}& Neural networks \& spectral normalization  &  IPM \eqref{eq:gen_wasserstein} or f-divergence  \eqref{eq:Df_var_formula} & Theorem~\ref{thm:general_ub}\\
 Arjovsky et al., \cite{WGAN} & $\Lip_b^1(S)$ & IPM \eqref{eq:gen_wasserstein} &  Theorem~\ref{thm:f_div_inf_convolution}\\
 Gulrajani et al., \cite{wgan:gp}    & $\Lip_b(S)$ & IPM \eqref{eq:gen_wasserstein} \& gradient penalty &  Theorem~\ref{thm:f_div_inf_convolution} \& Theorem~\ref{thm:soft_constraint}\\
Song et al., \cite{Bridging_fGan_WGan} (Algorithm 1)& $\Lip_b^1(S)$ & KL divergence using \eqref{eq:Df_var_formula}  &  Theorem~\ref{thm:f_div_inf_convolution} \\
  Nguyen et al., \cite{Nguyen_Full_2010}   & RKHS    & KL, f-divergence using \eqref{eq:Df_var_formula} &   Theorem~\ref{thm:f_div_inf_convolution}\\
 Gretton et al., \cite{gretton2012} & Unit ball in RKHS  & IPM \eqref{eq:gen_wasserstein} & Theorem~\ref{thm:f_div_inf_convolution}\\
 Glaser et al., \cite{2021arXiv210608929G} & RKHS & KL-div. using \eqref{eq:Df_var_formula} \&  RKHS norm penalty  &  Theorem~\ref{thm:f_div_inf_convolution} \& Theorem~\ref{thm:soft_constraint}\\
 Dupuis et al., \cite{Dupuis:Mao:2019} & convex \& closed $\Gamma$
& KL-divergence & Theorem~\ref{thm:f_div_inf_convolution}\\
\hline
\end{tabular}
\label{tab:related_work} 
}
\end{table}

\subsection{Additional Properties}

The following theorem  details the behavior of $D_f^\Gamma$ in a pair of limiting regimes and further illustrates the manner in which $D_f^\Gamma$ interpolates between $D_f$ and $W^\Gamma$. These results again require (strict) admissibility (see Definition \ref{def:admissible}).
\begin{theorem}\label{thm:limit}
Let $Q,P\in\mathcal{P}(S)$ and $\Gamma$, $f$ both be admissible.  Then for all $c>0$ the set $\Gamma_c\equiv\{c g: g\in\Gamma\}$ is admissible and we have the following two limiting formulas.
\begin{enumerate}
\item If $\Gamma$ is strictly admissible then the sets $\Gamma_L$ are strictly admissible for all $L>0$ and
\begin{align}
\lim_{L\to\infty} D^{\Gamma_L}_f(Q\|P)=D_f(Q\|P)\,.
\end{align}
\item If $f$ is strictly admissible then 
\begin{align}
\lim_{\delta\searrow 0}\frac{1}{\delta} D_f^{\Gamma_\delta}(Q\|P)=W^\Gamma(Q,P)\,.
\end{align}
\end{enumerate}
\end{theorem}
The proof of Theorem \ref{thm:limit} is very similar to that of  the corresponding results in the KL case \cite[Proposition 5.1 and 5.2]{Dupuis:Mao:2019}. For completeness, we include its proof in Appendix \ref{app:proofs} (Theorem \ref{thm:limit_app}).

{ 
Theorem \ref{thm:general_ub} implies the following convergence and continuity properties (see Theorem \ref{thm:conv_app} in  Appendix \ref{app:proofs} for the proof):
\begin{theorem}\label{thm:conv}
Let $f\in\mathcal{F}_1(a,b)$ and $\Gamma\subset\mathcal{M}_b(\Omega)$. Then:
\begin{enumerate}
    \item If there exists $c_0\in \Gamma\cap\mathbb{R}$ then  $W^\Gamma(Q_n,P)\to 0 \implies D_f^\Gamma(Q_n\|P) \to 0$ and $D_f(Q_n\|P)\to 0 \implies D_f^\Gamma(Q_n\|P) \to 0$, and similarly if one permutes the order of $Q_n$ and $P$.
    
    \item Suppose $f$ and $\Gamma$  satisfy the following:
\begin{enumerate}
\item There exist a nonempty set $\Psi\subset\Gamma$ with the following properties:
\begin{enumerate}
\item $\Psi$ is $\mathcal{P}(\Omega)$-determining.
\item For all $\psi\in\Psi$  there exists $c_0\in\mathbb{R}$, $\epsilon_0>0$ such that $c_0+\epsilon \psi\in\Gamma$ for all $|\epsilon|<\epsilon_0$.
\end{enumerate}
\item $f$ is strictly convex on a neighborhood of $1$.
\item $f^*$ is finite and $C^1$ on a neighborhood of $\nu_0\equiv f_+^\prime(1)$.
\end{enumerate}
Let $P,Q_n\in\mathcal{P}(\Omega)$, $n\in\mathbb{Z}_+$. If $D_f^\Gamma(Q_n\|P)\to 0$ or  $D_f^\Gamma(P\|Q_n)\to 0$ then $E_{Q_n}[\psi]\to E_P[\psi]$ for all $\psi\in\Psi$.  
\item On a metric space $S$, if $f$ is admissible then the map $(Q,P)\in\mathcal{P}(S)\times\mathcal{P}(S)\mapsto D_f^\Gamma(Q\|P)$ is  lower semicontinuous.
\end{enumerate}
\end{theorem}
\begin{corollary}\label{cor:weak_conv}
Under the assumptions of Part 2 of Theorem \ref{thm:conv} we have the following: If $\Gamma=\Lip_b^1(S)$ where $S$ is a compact metric space then one can take $\Psi=\Gamma$  and thereby conclude that $D_f^\Gamma(Q_n\|P)\to 0$ iff $D_f^\Gamma(P\|Q_n)\to 0$ iff $Q_n\to P$ in distribution iff $W^\Gamma(Q_n,P)\to 0$. 
\end{corollary}
\begin{remark}
Corollary \ref{cor:weak_conv} follows from the equivalence between weak convergence and convergence in the Wasserstein metric on compact spaces; see Theorem 2 in \cite{pmlr-v70-arjovsky17a} for this further relations between convergence in the Wasserstein metric and other notions of convergence.
\end{remark}
}

{ 
Finally, we derive a data processing inequality for $(f,\Gamma)$-divergences (see Theorem \ref{thm:data_proc_app} in Appendix \ref{app:proofs} for the proof). This result applies to general measurable spaces.  We will need the following notation: Let $(N,\mathcal{N})$ be another measurable space and $K$ be a probability kernel from $\Omega$ to $N$. Given $P\in\mathcal{P}(\Omega)$ we denote the composition of $P$ with $K$ by $P\otimes K$ (a probability measure on $\Omega\times N$) and we denote the marginal distribution on $N$ by $K[P]$. Given $g\in \mathcal{M}_b(\Omega\times N)$ we let $K[g]$ denote the bounded measurable function on $\Omega$ given by $x\to\int g(x,y)K_x(dy)$.
\begin{theorem}[Data Processing Inequality]\label{thm:data_proc}
Let $f\in\mathcal{F}_1(a,b)$, $Q,P\in\mathcal{P}(\Omega)$, and $K$ be a probability kernel from $(\Omega,\mathcal{M})$ to $(N,\mathcal{N})$. 
\begin{enumerate}
    \item  Let  $\Gamma\subset \mathcal{M}_b(N)$ be nonempty. Then  
    \begin{align}\label{eq:data_proc1}
        D_f^\Gamma\left(K[Q]\|K[P]\right)\leq D_f^{K[\Gamma]}(Q\|P)\,.
    \end{align}
    \item  Let  $\Gamma\subset \mathcal{M}_b(\Omega\times N)$ be nonempty. Then  
    \begin{align}\label{eq:data_proc2}
        D_f^\Gamma\left(Q\otimes K\|P\otimes K\right)\leq D_f^{K[\Gamma]}(Q\|P)\,.
        \end{align}
\end{enumerate}
\end{theorem}
\begin{remark}
In \req{eq:data_proc1} we use the obvious embedding of $\mathcal{M}_b(N)\subset\mathcal{M}_b(\Omega\times N)$ to define $K[\Gamma]\equiv\{K[g]:g\in\Gamma\}$.
\end{remark}

}

\section{Mass-Redistribution/Mass-Transport Interpretation of the $(f,\Gamma)$-Divergences}\label{sec:mass_transport}

The bound 
\begin{align}\label{eq:Df_Gamma_W_bound}
D_f^\Gamma(Q\|P)\leq W^\Gamma(Q,P)\,,    
\end{align}
which follows from {  Part 1 of either Theorem \ref{thm:general_ub} or Theorem 
\ref{thm:f_div_inf_convolution},}
makes it clear that $D_f^\Gamma(Q\|P)$ can be finite and informative even if $Q\not\ll P$. For instance, if $\Gamma=\Lip_b^1(S)$ then $W^\Gamma$ is the classical Wasserstein metric, and  this can be finite even for mutually singular $Q$ and $P$.  It is well-known that the Wasserstein metric can be understood in terms of mass transport \cite{villani2008optimal}. Generalizing this idea, the variational formula \eqref{eq:inf_conv_existence} allows us to interpret the $(f,\Gamma)$-divergences  in terms of a two-stage mass-redistribution/mass-transport process:
\begin{enumerate}
    \item First the  `mass' distribution, $P$, is redistributed to form an intermediate measure, $\eta_*$. This has cost $D_f(\eta_*\|P)$, which depends on the relative amount of mass   moved from or added to each point, but is insensitive to the distance that the mass is moved. 
    However, the support of $\eta_*$ cannot be enlarged or shifted outside the support of $P$ during its construction, otherwise the cost would be infinite.
    \item Next, the mass is transported from $\eta_*$ to $Q$ with a cost  $W^\Gamma(Q,\eta_*)$ that depends on the distance the mass must be moved.  In this step, the support of $\eta_*$ could be drastically different from the support of $Q$, if necessary.
\end{enumerate}
The optimizing $\eta_*$ achieves the optimal balance between the cost of redistributing mass in step 1 and the cost of transporting mass in step 2.   
\begin{remark}
When $\Gamma\neq \Lip_b^1(S)$, $D_f^\Gamma$ is still characterized by the above two-stage procedure, with the only difference being that the interpretation of $W^\Gamma$ may differ. 
\end{remark}

In this section we derive a characterization of   the solution to the infimal convolution problem \eqref{eq:inf_conv} (in the case where $f\in\mathcal{F}_1(a,b)$ with $a\geq 0$) and will use this to provide further insight into the mass-redistribution/mass-transport interpretation. A key step will be to first obtain existence and uniqueness results regarding the dual optimization problem \eqref{eq:Df_Gibbs_var_formula} for the classical $f$-divergences. The proof is found in Appendix \ref{app:proofs}, Theorem \ref{thm:Gibbs_optimizer_app}.

\begin{theorem}\label{thm:Gibbs_optimizer}
 Let $P\in\mathcal{P}(\Omega)$,   $ g\in\mathcal{M}_b(\Omega)$, and $f\in\mathcal{F}_1(a,b)$ be admissible with  $a\geq 0$. If $f$ is strictly convex on $(a,b)$ then there exists  $\nu_*\in\mathbb{R}$ such that
\begin{align}
dQ_*\equiv (f^*)^\prime( g-\nu_*) dP
\end{align}
is a probability measure and
\begin{align}
&\sup_{Q\in\mathcal{P}(\Omega)}\{E_Q[ g]-D_f(Q\|P)\}= E_{Q_*}[ g]-D_f(Q_*\|P)=\nu_*+E_P[f^*( g-\nu_*)]=\Lambda_f^P[g]\,.
\end{align}
Moreover, $Q_*$ is the unique solution to the
optimization problem
\begin{align}\label{eq:Gibb_Q_unique}
\sup_{Q\in\mathcal{P}(\Omega)}\{E_Q[ g]-D_f(Q\|P)\}\,.
\end{align}
\end{theorem}

Theorem \ref{thm:Gibbs_optimizer} (specifically, the  generalization found in Theorem \ref{thm:Gibbs_optimizer_app}) allows us to derive in Theorem \ref{thm:inf_conv_sol} a  characterization of the solution, $\eta_*$, to the infimal convolution problem \eqref{eq:inf_conv}. First we present a formal calculation; a precise statement of the result can be found in Theorem \ref{thm:inf_conv_sol} and a rigorous proof is given  in Theorem \ref{thm:inf_conv_sol_app} of Appendix \ref{app:proofs}. This result generalizes Theorem 4.12 from \cite{Dupuis:Mao:2019}, which considered the KL case: First assume $(g_*, \nu_*)$ is a maximizer of \eqref{eq:gen_f_def}, and assume $\eta_*$ solves  \eqref{eq:inf_conv}. Then
\begin{align}\label{eq:Df_Phi_formula_informal}
D_f^\Gamma(Q\|P)&=E_Q[  g_*]-(\nu_*+E_P[f^*( g_*-\nu_*)])\\
&=E_Q[  g_*]-E_{\eta_*}[  g_*]+E_{\eta_*}[  g_*]-(\nu_*+E_P[f^*( g_*-\nu_*)])\notag\\
&\le W^\Gamma(Q,\eta_*)+D_f(\eta_*\|P)=
D_f^\Gamma(Q\|P)\, .\notag
\end{align}
Therefore, as the inequalities become equalities, we have
\[
W^\Gamma(Q,\eta_*)=E_Q[  g_*]-E_{\eta_*}[g_*]
\]
and
\begin{align}\label{eq:Df_equality_formal}
    D_f(\eta_*\|P)=E_{\eta_*}[  g_*]-(\nu_*+E_P[f^*( g_*-\nu_*)])\,.
\end{align}
Note that this also implies $E_P[(f^*)^\prime_+( g_*-\nu_*)]=1$. 
\begin{align}
    d\eta_* & = (f^*)^\prime( g_*-\nu_*)dP\, , \\
    g_*& =f^\prime( d\eta_*/dP)+\nu_*\, \,\,\,\,P\text{-a.s.}
\end{align}
In particular, in the KL case \cite[Remark 4.11]{Dupuis:Mao:2019}), one has
\begin{align}
    g_*=\log(d\eta_*/dP)+c_0\, \,\,\,\,P\text{-a.s.}
\end{align}
for some $c_0\in\mathbb{R}$ and if $Q\ll P$ this leads to
\begin{align}\label{eq:KL:Gamma:classical}
    R^\Gamma(Q\|P)=E_Q[\log(d\eta_*/dP)]\,,
\end{align}
which has an obvious similarity to the formula for the classical KL divergence.

\begin{theorem}\label{thm:inf_conv_sol}
Let  $\Gamma\subset C_b(S)$ be admissible and $f\in\mathcal{F}_1(a,b)$ be admissible, where   $a\geq 0$ and $f^*$ is $C^1$.  Fix $Q,P\in\mathcal{P}(S)$ and suppose we have $ g_*\in\Gamma$ and $\nu_*\in\mathbb{R}$ that satisfy the following:
\begin{enumerate}
\item $f((f^*)^\prime( g_*-\nu_*))\in L^1(P)$,
\item  $E_P[(f^*)^\prime( g_*-\nu_*)]=1$,
\item $W^\Gamma(Q,\eta_*)=E_Q[  g_*]-E_{\eta_*}[ g_*]$, where  $d\eta_*\equiv (f^*)^\prime( g_*-\nu_*)dP$.
\end{enumerate}
Then $\eta_*\in\mathcal{P}(S)$ solves the infimal convolution problem \eqref{eq:inf_conv}   and
\begin{align}\label{eq:Df_Phi_formula}
D_f^\Gamma(Q\|P)=E_Q[  g_*]-(\nu_*+E_P[f^*( g_*-\nu_*)])\,.
\end{align}
If $f$ is strictly convex then $\eta_*$ is the unique solution to the infimal convolution problem.
\end{theorem}

\begin{remark}
In the context of MMD,  $g_*$   is called the witness function \cite{gretton2012}. In the KL case,  the existence of $g_*$ can be proven under appropriate compactness assumptions \cite[Theorem 4.8]{Dupuis:Mao:2019}. 
\end{remark}
\begin{remark}
\req{eq:inf_conv_existence} from Theorem \ref{thm:f_div_inf_convolution} makes it clear that $D_f^\Gamma(Q\|P)< D_f(Q\|P)$ in `most' cases.  An exception to this occurs when   \req{eq:Df_var_formula} has an optimizer $g_*$ with $ g_*\in\Gamma$.  In such cases we have $D_f^\Gamma(Q\|P)=D_f(Q\|P)$, the supremum \eqref{eq:gen_f_def} will also be achieved at $ g_*$ since \req{eq:Df_Phi_formula} holds with $\nu_*=0$, and the solution to the infimal convolution problem is $\eta_*=Q$.
\end{remark}

{  In general, the task of computing the intermediate measure $\eta_*$ in \eqref{eq:inf_conv_existence} is difficult, though a naive approach could proceed as follows:
\begin{enumerate}
    \item Approximate $\eta\in\mathcal{P}(S)$ by a neural network family $h_\theta(X)$, where $X$ is some random noise source (as in the generator of a GAN; see Section \ref{sec:examples}); in this step we are using Corollary \ref{cor:upper_bound} to construct an upper bound.
    \item Approximate $D_f(\eta\|P)$ and $W^\Gamma(Q,\eta)$ via their variational formulas \eqref{eq:Df_var_formula} or \eqref{eq:Df_var_formula_nu} and \eqref{eq:gen_wasserstein} respectively, with the function spaces being approximated via neural network families (as in the discriminator of a GAN; again, see Section \ref{sec:examples}).
    \item Solve the resulting min-max problem \eqref{eq:inf_conv} via a stochastic-gradient-descent method to approximate $\eta_*$ (and also $g_*$).
\end{enumerate}
We did not explore the effectiveness of this naive method here, as it is tangential to the goals of this paper; we leave the computation of $\eta_*$ for a future work.  Nevertheless, the following subsection presents a simple example that provides useful intuition.}

\subsection{Example: Dirac Masses}

Here we consider a  simple  example involving Dirac masses where the $(f,\Gamma)$-divergence can be explicitly computed  using Theorem \ref{thm:inf_conv_sol}. This example  further illustrates the two-stage mass-redistribution/mass-transport interpretation of the infimal convolution formula \eqref{eq:inf_conv_existence} and demonstrates how the location and distribution of probability mass impacts the result; see Figure \ref{fig:mass_transport}.  Further explicit examples in the KL case can be found in \cite{2020arXiv201108441M}.  

Let $0=x_1<x_2<x_3$ and define the uniform distributions
\begin{align}\label{eq:dirac}
P=\frac{1}{2}\delta_{x_1}+\frac{1}{2}\delta_{x_2}\,,\,\,\,\,\,Q=\frac{1}{3}\delta_{x_1}+\frac{1}{3}\delta_{x_2}+\frac{1}{3}\delta_{x_3}\,.  
\end{align}
Note that $Q\not\ll P$ and so $D_f(Q\|P)=\infty$; we will see that the $(f,\Gamma)$-divergences can be finite.  Specifically, we will compute the $(f_\alpha,\Lip_b^1(\mathbb{R}))$-divergence for $\alpha>1$ via Theorem \ref{thm:inf_conv_sol}. To do this we must find $ g_*\in\Lip_b^1(\mathbb{R})$ and $\nu_*\in\mathbb{R}$ such that
\begin{align}
    &\frac{1}{2}(f_\alpha^*)^\prime( g_*(x_1)-\nu_*)+\frac{1}{2}(f_\alpha^*)^\prime( g_*(x_2)-\nu_*)=1\,,\label{eq:ex1}\\
    & g_*\in \argmax_{ g\in\Lip_b^1(\mathbb{R})}\left\{\frac{1}{3}( g(x_1)+ g(x_2)+ g(x_3))-\frac{1}{2}\left( g(x_1)(f^*_\alpha)^\prime( g_*(x_1)-\nu_*)+ g(x_2)(f^*_\alpha)^\prime( g_*(x_2)-\nu_*)\right)\right\}\,,\label{eq:ex2}
\end{align}
where 
\begin{align}
    (f_\alpha^*)^\prime(y)=(\alpha-1)^{1/(\alpha-1)}y^{1/(\alpha-1)}1_{y> 0}
\end{align}
 \begin{figure}
\begin{minipage}[b]{0.45\linewidth}
  \centering
\includegraphics[scale=.50]{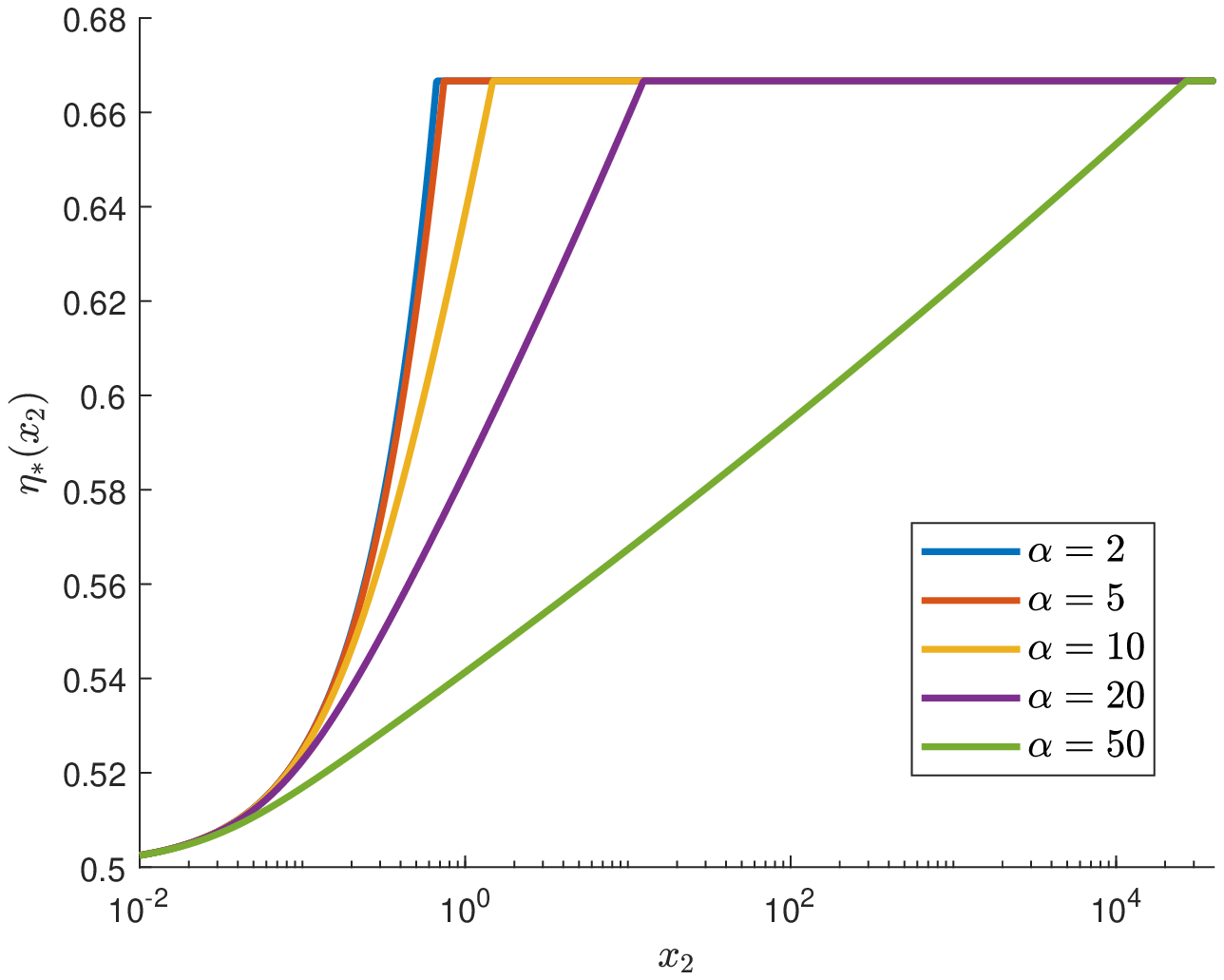} \end{minipage}
\hspace{0.5cm}
\begin{minipage}[b]{0.45\linewidth}
\centering
\includegraphics[scale=.50]{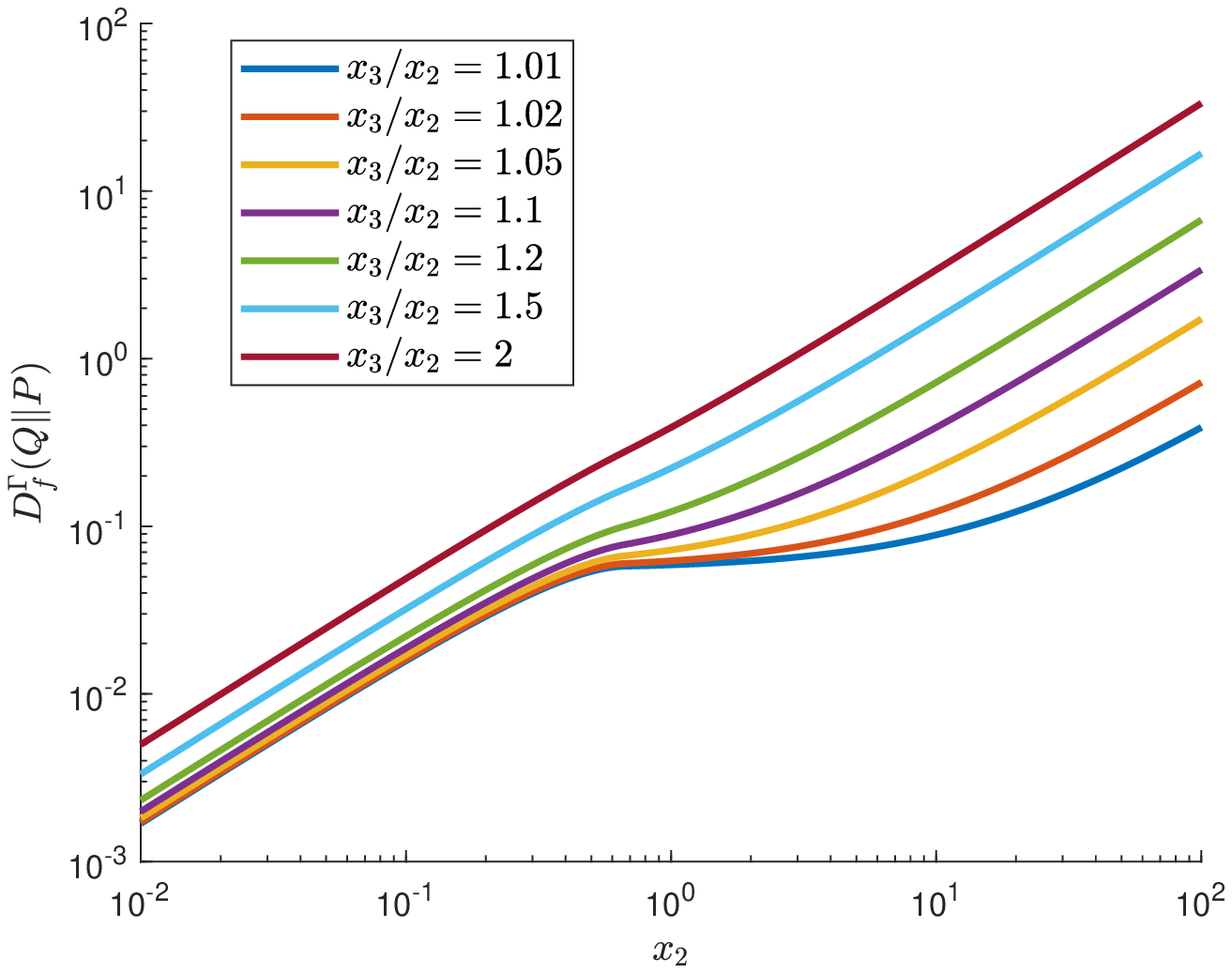} \end{minipage}

   \caption{ Solution of the infimal convolution problem \eqref{eq:inf_conv_existence} for $D_{f_\alpha}^\Gamma(Q\|P)$, where $\Gamma=\Lip_b^1(\mathbb{R})$ and $Q$ and $P$ are given by \req{eq:dirac}. The left panel shows the mass $\eta_*(x_2)$ as a function of $x_2$.  For each value of $\alpha$ there is a transition point where all of the mass required by $Q$ at $x_3$ is first redistributed to $x_2$ when forming $\eta_*$, resulting in $\eta_*(x_2)=2/3$. Note that the amount of mass moved to $x_2$ in the redistribution step does not depend on the distance of $x_3$ from $x_2$, only on the distance of $x_2$ from $x_1=0$. The right panel shows $D_{f_2}^\Gamma(Q\|P)$ as a function of $x_2$ and for several different values of the ratio $x_3/x_2$. 
 }\label{fig:mass_transport}
\end{figure}
(see \req{eq:f_alpha_star}); \req{eq:ex1} is a simplification of Assumption 2 from  Theorem \ref{thm:inf_conv_sol} and \eqref{eq:ex2} corresponds to Assumption 3. The solution to the infimal convolution problem then has the form
\begin{align}\label{eq:gamma_star_ex}
    d\eta_*= \frac{1}{2}(f_\alpha^*)^\prime( g_*(x_1)-\nu_*)\delta_{x_1}+
    \frac{1}{2}(f_\alpha^*)^\prime( g_*(x_2)-\nu_*)\delta_{x_2}\,.
\end{align}
We will now outline how one solves for $\nu_*$ and $g_*$. Without loss of generality we can assume $ g_*(x_1)=0$ (the objective functional for $W^\Gamma$ is invariant under constant shifts and at the same time, shifting $g_*$ in $\eta_*$ can be achieved by redefining $\nu_*$).   The only dependence on $ g(x_3)$ in \req{eq:ex2} is in the $ g(x_3)/3$ term, hence the optimal solution has $ g(x_3)=x_3-x_2+ g(x_2)$. Therefore we need to solve
\begin{align}\label{eq:dirac_example_eqs}
    &\frac{1}{2}(f_\alpha^*)^\prime(-\nu_*)+\frac{1}{2}(f_\alpha^*)^\prime( g_*(x_2)-\nu_*)=1\,,\\
    & g_*(x_2)\in \argmax_{ g(x_2)\in[-x_2,x_2]}\left\{\frac{1}{3}(x_3-x_2)+\left(\frac{2}{3}-\frac{1}{2}(f^*_\alpha)^\prime( g_*(x_2)-\nu_*)\right) g(x_2)\right\}\notag
\end{align}
for $\nu_*$ and $g_*(x_2)$. The solution to this is obtained as follows:
\begin{enumerate}
    \item Let $\nu_*( g_2)$ be the unique solution to $\frac{1}{2}(f_\alpha^*)^\prime(-\nu_*)+\frac{1}{2}(f_\alpha^*)^\prime( g_2-\nu_*)=1$; the two terms on the left hand side will be used to obtain the redistributed weights in $\eta_*$.
    \item Take $ g_{*,2}$ such that $\frac{1}{2}(f_\alpha^*)^\prime( g_{*,2}-\nu_*( g_{*,2}))=2/3$; this is inspired by the second line in \req{eq:dirac_example_eqs}.
    \item If $0<x_2< g_{*,2}$ then the solution to \req{eq:dirac_example_eqs} is obtained at $\nu_*=\nu_*(x_2)$ and
    \begin{align}
 g_*(x)=\begin{cases} 
     0, & x< 0\\
      x, &x\in[0,x_3)\\
      x_3,& x\geq x_3\,.
   \end{cases}
\end{align}
In this case, the optimal solution has $1/3<\eta_*(x_2)<2/3$, i.e., some amount of mass is redistributed from $x_1=0$ to $x_2$ when forming $\eta_*$ and then mass is transported  from both $x_1$ and $x_2$ to $x_3$ to form $Q$.
    \item If $x_2\geq  g_{*,2}$ then the solution to \req{eq:dirac_example_eqs} is obtained at $\nu_*=\nu_*( g_{*,2})$ and 
    \begin{align}
 g_*(x)=\begin{cases} 
     0, & x< 0\\
      \frac{ g_{*,2}}{x_2} x, &x\in[0,x_2)\\
      x-x_2+ g_{*,2},& x\in[x_2,x_3)\\
      x_3-x_2+ g_{*,2},& x\geq x_3\,.
   \end{cases}
\end{align}
In this case, $x_2$ is sufficiently far away from $x_1=0$ that  the optimal solution, $\eta_*$, is obtained by first redistributing  mass from $x_1=0$ to $x_2$ so that $\eta_*(x_1)=1/3$, $\eta_*(x_2)=2/3$. In the second step,  mass is transported solely from $x_2$ to $x_3$ in order to form $Q$.
\end{enumerate}
This completes the construction of $\eta_*$ from \req{eq:gamma_star_ex}. The value of the $(f_\alpha,\Lip_b^1(\mathbb{R}))$-divergence can then be computed via \req{eq:Df_Phi_formula}. The computation of $g_{*,2}$ and $\nu_*(g_{*,2})$ from steps 1 and 2 must be done numerically and so we illustrate the solution graphically in Figure \ref{fig:mass_transport} by plotting $\eta_*(x_2)$ as a function of $x_2$ for a number of $\alpha$'s. This shows how the mass must be redistributed when forming $\eta_*$ from $P$.     The above calculations reveal an interesting transition; when $x_2$ is not close\footnote{With closeness being defined not only relative to the distance between the two points but also depending on $f$.} to $x_1$ then the mass is transferred solely from $x_2$ after it has been redistributed from $x_1$. However, when $x_1$ and $x_2$ are close enough then redistributing all the necessary mass from $x_1$ to $x_2$ is not optimal and it is cheaper to transport probability mass from both $x_1$ and $x_2$ to $x_3$.  The transition between these cases corresponds to the point where $x_2$ crosses above $g_{*,2}$ (which depends on $\alpha$) and hence  $\eta_*(x_2)$ saturates at the value $2/3$.

\section{Soft Constraints and the Divergence Property}\label{sec:soft_constraint}
For  computational purposes, it is  often advantageous to replace the hard (i.e., strict) constraint $ g\in\Gamma$ with a soft constraint in the form of a penalty term, $V$, subtracted from the objective functional; by a penalty term, we mean $V$ `activates' (i.e., is nonzero) when the constraint $ g\in\Gamma$ is violated.  In this way we can construct a new divergence $D_f^{V}$ with $D_f^\Gamma\leq D_f^{V}\leq D_f$ (we let the context distinguish between cases where the superscript denotes a constraint space and cases where it denotes  a penalty term); see Theorem \ref{thm:soft_constraint} for the main result of this section.

Of particular interest is the case  $\Gamma=\Lip^1_b(\mathbb{R}^n)$ (we equip $\mathbb{R}^n$ with the Euclidean metric), where the $1$-Lipschtiz constraint can be relaxed to  a one-sided gradient penalty term, thus defining objects such as
\begin{align}\label{eq:Lip_soft_constraint}
D_f^{\rho}(Q\|P)=\sup_{ g\in\Lip_b(\mathbb{R}^n)}\left\{E_Q[  g]-\Lambda_f^P[g]-\lambda\int \max\{0,\|\nabla  g\|^2-1\}d\rho_{Q,P}\right\}\,,
\end{align}
where $\lambda>0$ is the strength of the penalty term and $\rho_{Q,P}$ is a positive measure (often depending on $Q$ and $P$).  Here we are relying on  Rademacher's theorem (see Theorem 5.8.6 in \cite{evans2010partial}): $L$-Lipschitz functions on $\mathbb{R}^n$  are  differentiable Lebesgue-a.e. and the norm of the gradient is bounded by $L$.  The penalty term in \req{eq:Lip_soft_constraint} will therefore be activated only when $ g$ is not $1$-Lipschitz.

Divergences with soft Lipschitz constraints  were first applied to Wasserstein GAN   \cite{wgan:gp} with great success, but the theoretical properties of such objects have not been explored; specifically, it has not been shown that they satisfy the divergence property.  Here we show in great generality that the relaxation of a hard constraint to a soft constraint preserves the divergence property, and therefore objects such as \eqref{eq:Lip_soft_constraint} still provide a well-defined notion of `distance' between probability measures. The basic requirement is that the penalty term, which we denote by $V$, vanishes on the constraint space $\Gamma$.
\begin{lemma}\label{lemma:soft_constraint}
Let $(\Omega,\mathcal{M})$ be a measurable space, $\Gamma\subset\widetilde{\Gamma}\subset\mathcal{M}(\Omega)$, $H:\widetilde{\Gamma}\times\mathcal{P}(\Omega)\times\mathcal{P}(\Omega)\to\overline{\mathbb{R}}$, and $V:\widetilde{\Gamma}\times\mathcal{P}(\Omega)\times\mathcal{P}(\Omega)\to[0,\infty]$ with $V|_{\Gamma\times\mathcal{P}(\Omega)\times\mathcal{P}(\Omega)}=0$.  Define
\begin{align}
&D^\Gamma(Q\|P)=\sup_{ g\in \Gamma} H[ g;Q,P]\,,\,\,\,\,\,   D^{\widetilde{\Gamma}}(Q\|P)=\sup_{ g\in {\widetilde{\Gamma}}} H[ g;Q,P]\,,\\
&D^{V}(Q\|P)=\sup_{ g\in{\widetilde{\Gamma}}}\{H[ g;Q,P]-V[ g;Q,P]\}\,,\notag
\end{align}
where $\infty-\infty\equiv -\infty$. If $D^\Gamma$ and $D^{\widetilde{\Gamma}}$ both have the divergence property then so does $D^{V}$.

\end{lemma}
\begin{remark}
The convention $\infty-\infty\equiv-\infty$ is simply a convenient  rigorous shorthand for restricting the supremum to those $ g$'s for which this generally undefined operation does not occur. 
\end{remark}
\begin{remark}\label{remark:generalized_constraint}
More generally, if the supremum $\sup_{ g\in \Gamma} H[ g;Q,P]$ is achieved at $ g_*\in\Gamma$ (depending on $Q,P$) then the requirement  $V|_{\Gamma\times\mathcal{P}(\Omega)\times\mathcal{P}(\Omega)}=0$ can be relaxed to $V[ g_*;Q,P]=0$ for all $Q,P$.
\end{remark}
\begin{proof}
Using $\Gamma\subset{\widetilde{\Gamma}}$, $V\geq 0$, and $V|_\Gamma=0$ we have $D^\Gamma\leq  D^{V}\leq D^{\widetilde{\Gamma}}$. $D^\Gamma$ satisfies the divergence property, hence is non-negative.  Therefore $D^{V}\geq 0$. $D^{\widetilde{\Gamma}}$ has the divergence property, hence if $Q=P$ then $0=D^{\widetilde{\Gamma}}(Q\|P)\geq D^{V}(Q\|P)\geq 0$.  Therefore $D^{V}(Q\|P)=0$.  Finally, if $D^{V}(Q\|P)=0$ then  $D^\Gamma(Q\|P)=0$ and hence the divergence property for $D^\Gamma$ implies $Q=P$.
\end{proof}

Using Theorem \ref{thm:f_div_inf_convolution} and   Corollary \ref{cor:D_f_var_unbounded}, we can apply Lemma \ref{lemma:soft_constraint} to the $(f,\Gamma)$-divergences and thereby conclude the following:
\begin{theorem}
\label{thm:soft_constraint}
Let $f$ and $\Gamma\subset C_b(S)$ be strictly admissible. Let $\Gamma\subset{\widetilde{\Gamma}}\subset\mathcal{M}(S)$   and $V:{\widetilde{\Gamma}}\times\mathcal{P}(S)\times\mathcal{P}(S)\to[0,\infty]$   with $V|_{\Gamma\times\mathcal{P}(S)\times\mathcal{P}(S)}=0$. For $Q,P\in\mathcal{P}(S)$ define
\begin{align}
D^{V}_f(Q\|P)\equiv\sup_{ g\in{\widetilde{\Gamma}}}\left\{\left( E_Q[ g]-\Lambda_f^P[g]\right)-V[ g;Q,P]\right\}\,,
\end{align}
where $\infty-\infty\equiv-\infty$, $-\infty+\infty\equiv-\infty$.  Then $D^{V}_f$ has the divergence property and $D_f^\Gamma\leq D_f^{V}\leq D_f$.
\end{theorem}
\begin{proof}
Combine Lemma \ref{lemma:soft_constraint} with Part 4 of Theorem \ref{thm:f_div_inf_convolution} and Theorem \ref{thm:Df_var_unbounded} below; the latter shows that the variational formula for $D_f$ also  holds when using the test-function space $\mathcal{M}(\Omega)$.
\end{proof}

\subsection{Soft-Lipschitz Constraints on $\mathbb{R}^n$: One-Sided Versus Two-Sided Penalties}
The gradient penalty term in \req{eq:Lip_soft_constraint} is one-sided, meaning that it penalizes $\|\nabla g\|>1$ but not $\|\nabla g\|\leq 1$.  This is consistent with the hard constraint that the Lipschitz constant be less than or equal to $1$. The first use of soft Lipschitz penalties in \cite{wgan:gp}, which considered the Wasserstein metric, also used a two-sided gradient penalty,
\begin{align}\label{W_rho}
    W^\rho(Q,P)=\sup_{ g\in\Lip_b(\mathbb{R}^n)}\left\{E_Q[ g]-E_P[ g]-\lambda\int (\|\nabla g\|-1)^2d\rho_{Q,P}\right\}\,,
\end{align}
which penalizes $\|\nabla g\|\neq 1$. An intuitively reasonable requirement to impose on any soft constraint is that it vanish on the exact optimizer (if one exists) of the original strictly-constrained optimization problem. The justification for a two-sided gradient penalty in the  Wasserstein case  rests on   Proposition 1 in \cite{wgan:gp}, which shows that the exact optimizer of the Kantorovich-Rubinstein variational formula for the classical  Wasserstein metric has gradient with norm $1$ a.e. As the two-sided gradient penalty vanishes on such functions, the object \eqref{W_rho} will still possess the divergence property (see Remark \ref{remark:generalized_constraint}). However, two-sided gradient penalties are not  appropriate constraint-relaxations of the $(f,\Gamma)$-divergences, as the gradient of the exact optimizer generally does not have norm $1$ a.e. We demonstrate this via the following simple counterexample:   Let  $\Gamma=\Lip^1_b(\mathbb{R}^n)$,  $P\in\mathcal{P}(\mathbb{R}^n)$,  and define $Q$ by $dQ/dP=Z^{-1}e^{-\min\{\|x\|,1\}/2}$.  The optimizer of the variational formula defined in \eqref{eq:Df_optimizer} is given for the classical KL divergence  by
\begin{align}
     g_*=\log(dQ/dP)+1=-\min\{\|x\|,1\}/2+1+\log(Z^{-1})\,,
\end{align} 
which is bounded and $1/2$-Lipschitz, and so $ g_*\in\Gamma$.  Therefore it is straightforward to see that $g_*$ is also the optimizer for $R^{\Gamma}(Q\|P)$ and it satisfies $\|\nabla g_*\|\leq 1/2$ a.e. This proves that the 2-sided penalty does not vanish on $ g_*$. Similar counterexamples can be constructed using \req{eq:Df_optimizer} for other choices of $f$.

\section{Extension of the $(f,\Gamma)$-Divergence Variational Formula to Unbounded Functions}
\label{sec:unbounded_extension}
The assumption that all of the test functions $ g\in\Gamma$ are bounded can be very restrictive  in practice. In this section we provide general conditions under which the test-function space can be expanded to include (possibly) unbounded functions without changing the value of  $D_f^\Gamma$. This fact  will be used in the numerical examples in Section \ref{sec:examples} below. The main result in this Section in Theorem \ref{thm:unbounded_Lip}.

The key step in the extension to unbounded $ g$'s is the following lower bound.
\begin{lemma}\label{lemma:phi_unbounded}
Let $f$, $\Gamma$ be admissible and, in addition, suppose $f^*$ is bounded below. Fix $Q,P\in\mathcal{P}(S)$. If $ g\in L^1(Q)$  and there exists $ g_n\in\Gamma$, a measurable set $A$, and $C\in\mathbb{R}$ with $ g_n\to g$ pointwise, $| g_n|\leq| g|$ for all $n$, and   $ g_n\leq  g 1_A+C1_{A^c}$ for all $n$, then
\begin{align}
D_f^\Gamma(Q\|P)\geq E_Q[  g]-\Lambda_f^P[g]\,.
\end{align}
\end{lemma}
\begin{remark}
The additional assumption that $f^*$ is bounded below is satisfied in many cases of interest, e.g., the KL divergence and $\alpha$-divergences for $\alpha>1$.
\end{remark}
\begin{proof}
We need to show that
\begin{align}
D_f^\Gamma(Q\|P)\geq E_Q[ g] -(\nu+E_P[f^*( g-\nu)])
\end{align}
for all $\nu\in\mathbb{R}$.  Note that we have assumed $f^*$ is bounded below by some $D\in\mathbb{R}$, hence $E_P[f^*( g-\nu)]$ exists in $(-\infty,\infty]$.  If $E_P[f^*( g-\nu)]=\infty$ then the claim is trivial, so for the remainder of this proof we suppose $f^*( g-\nu)\in L^1(P)$.

The assumptions on $g$ allow us to use the dominated convergence theorem to conclude $E_Q[ g_n]\to E_Q[ g]$. Continuity of $f^*$  implies $f^*( g_n-\nu)\to f^*( g-\nu)$.  The admissibility assumption implies $\lim_{y\to-\infty}f^*(y)<\infty$. Using this together with  Lemma \ref{lemma:f_star_nondec} we see that  $f^*$ is nondecreasing, hence
\begin{align}
D\leq f^*( g_n-\nu)\leq f^*( g-\nu)1_A+f^*(C-\nu)1_{A^c}\in L^1(P)\,.
\end{align}
Therefore the dominated convergence theorem implies $E_P[f^*( g_n-\nu)]\to E_P[f^*( g-\nu)]$. We have  $ g_n\in \Gamma$, hence \req{eq:gen_f_def} implies
\begin{align}
D_f^\Gamma(Q\|P)\geq& \lim_{n\to\infty}\left(E_Q[ g_n]-(\nu+E_P[f^*( g_n-\nu)])\right)\\
=&E_Q[ g]-(\nu+E_P[f^*( g-\nu)])\,.\notag
\end{align}
This completes the proof.
\end{proof}

Using Lemma \ref{lemma:phi_unbounded}, one can augment $\Gamma$ by including any functions that satisfy the stated assumptions; this will not change the value of the supremum in \eqref{eq:gen_f_def}.  Rather than formulating a general result of this type, we consider one of the most useful special cases, the set of Lipschitz functions. Other cases can be treated similarly. 
\begin{lemma}\label{lemma:Lip_Phi}
Let $c:S\times S\to[0,\infty]$, $L\in(0,\infty)$, and define 
\begin{align}\label{eq:Lip_b_def}
\Lip_b^L(S,c)=\{ g\in C_b(S):| g(x)- g(y)|\leq L c(x,y)\text{ for all }x,y\in S\}\,.
\end{align}
If $c=d$ (the metric on $S$) then we use our earlier notation, $\Lip_b^L(S)$, in place  of $\Lip_b^L(S,d)$.

 The set $\Lip_b^L(S,c)$ is admissible and if $d\leq K c$ for some $K\in(0,\infty)$   then $\Lip_b^L(S,c)$ is strictly admissible. 
\end{lemma}
\begin{proof}
 Convexity is trivial. Weak convergence in $C_b(S)$ implies pointwise convergence (take $\mu=\delta_x$, $x\in S$), hence $\Lip_b^L(S,c)$ is closed. Finally, if $d\leq K c$ then strict admissibility follows from the fact that $\Lip_b^{L/K}(S)$ is $\mathcal{P}(S)$-determining and $\Lip_b^{L/K}(S)\subset \Lip_b^L(S,c)$.
\end{proof}
\begin{remark}
For $L>0$ we have $\Lip^L_b(S,c)=\{L g: g\in\Lip^1_b(S,c)\}$ and so (under appropriate assumptions) Theorem \ref{thm:limit} implies the following limiting formulas:
\begin{align}
&\lim_{L\to\infty}D_f^{\Lip^L_b(S,c)}(Q\|P)=D_f(Q\|P)\,,\\
&\lim_{L\searrow 0}L^{-1}D_f^{\Lip^L_b(S,c)}(Q\|P)=W^{\Lip^1_b(S,c)}(Q,P)\,.\notag
\end{align}
\end{remark}
Using Lemma \ref{lemma:phi_unbounded} we can show that the boundedness constraint can be dropped in the formula for $D_f^{\Gamma}$ when $\Gamma=\Lip_b^L(S,c)$; we exploit this fact in the numerical examples in Section \ref{sec:examples} below.
\begin{theorem}\label{thm:unbounded_Lip}
Let $c:S\times S\to[0,\infty]$, $L\in(0,\infty)$, and define 
\begin{align}
\Lip^L(S,c)=\{ g\in C(S):| g(x)- g(y)|\leq L c(x,y)\text{ for all }x,y\in S\}\,.
\end{align}
 Let  $f$ be admissible such that $f^*$ is bounded below.  Then  for $Q,P\in\mathcal{P}(S)$  we have
\begin{align} \label{eq:unbounded_var_formula}
D_f^{\Lip_b^L(S,c)}(Q\|P)=&\sup_{ g\in\Lip^L(S,c)\cap L^1(Q)}\left\{ E_Q[  g]-\Lambda_f^P[g]\right\}\,.
\end{align}
\end{theorem}
\begin{proof}
Lemma \ref{lemma:Lip_Phi} shows that $\Gamma\equiv \Lip_b^L(S,c)$ satisfies the conditions of Theorem \ref{thm:f_div_inf_convolution}. Fix  $ g\in \Lip^L(S,c)\cap L^1(Q)$. For $n\in\mathbb{Z}^+$ define $ g_n=n1_{ g>n}+g1_{-n\leq  g\leq n}-n1_{ g<-n}$. It is easy to see that $ g_n\in \Lip_b^L(S,c)$,  $ g_n\to  g$, $| g_n|\leq| g|$, $ g_n\leq g1_{ g\geq 0}$.  Therefore  Lemma \ref{lemma:phi_unbounded} implies
\begin{align}\label{eq:unbounded_var_formula_lb}
D_f^\Gamma(Q\|P)\geq E_Q[ g ]-\inf_{\nu\in\mathbb{R}}\{\nu+E_P[f^*( g-\nu)]\}\,.
\end{align}
One inequality in \req{eq:unbounded_var_formula} follows from taking the supremum over all $ g\in \Lip^L(S,c)\cap L^1(Q)$ in \req{eq:unbounded_var_formula_lb} and the reverse follows from the fact that $\Lip_b^L(S,c)\subset\Lip^L(S,c)\cap L^1(Q)$.
\end{proof}

\section{$(f,\Gamma)$-GANs}\label{sec:examples}
Generative adversarial networks  constitute a class of methods for `learning' a probability distribution, $Q$, via a two-player game between a discriminator and a generator (both neural networks) \cite{GAN,f_GAN,WGAN,wgan:gp,CumulantGAN:Pantazisetal}. Mathematically, most GANs can be formulated as  divergence minimization problems for a divergence, $D$, that has a variational characterization $D(Q\|P)=\sup_{ g\in\Gamma} H[ g;Q,P]$. The goal is then to solve the following optimization problem:
\begin{align}\label{eq:gen_GAN}
 \inf_{\theta\in\Theta} D(Q\|P_\theta)=\inf_{\theta\in\Theta} \sup_{ g\in\Gamma} H[ g;Q,P_\theta]&\,.
\end{align}
Here, $g$ is called the discriminator and $P_\theta$ is the distribution of $h_\theta(X)$, where $X$ is a random noise source and $h_\theta$, $\theta\in\Theta$ is a neural network family (the generator). The minimax problem \eqref{eq:gen_GAN} can  be interpreted as two-player zero-sum game. GANs based on the Wasserstein-metric have been very successful \cite{WGAN,wgan:gp} and GANs based on the classical $f$-divergences have also been explored \cite{f_GAN}.  Here we show that GANs based on the $D_f^\Gamma$ divergences, which generalize and interpolate between the above two extremes, inherit desirable properties from both IPM-GANs (e.g., Wasserstein GAN) and $f$-GANs. Specifically, we focus on the following:
\begin{enumerate}
    \item $(f,\Gamma)$-GANs can perform well when applied to heavy-tailed distributions. This property is inherited from the classical $f$-divergences.
    \item $(f,\Gamma)$-GANs can perform well even when there is a lack of absolute continuity. This property is inherited from the $\Gamma$-IPMs.
\end{enumerate}
We will specifically focus on the cases where $f=f_\alpha$, $\alpha>1$, (see \req{eq:f_alpha_def}) and $\Gamma=\Lip_b^L(\mathbb{R}^n)$ (see Lemma \ref{lemma:Lip_Phi}). We call the corresponding $(f,\Gamma)$-divergences the Lipschitz $\alpha$-divergences and will denote them by $D_{\alpha}^L$.  As $\Gamma$ is closed under shifts, we can express these divergences in one of two ways (see \req{eq:Df_Gamma_no_shift}):
\begin{align}
    D^L_{\alpha}(Q\|P)=&\sup_{ g\in \Lip_b^L(\mathbb{R}^n)}\{E_Q[ g]-\Lambda_{f_\alpha}^P[g]\}\label{eq:f_alpha_nu}\\
    =&\sup_{ g\in \Lip_b^L(\mathbb{R}^n)}\{E_Q[ g]-E_P[f_\alpha^*( g)]\}\,.\label{eq:Df_alpha_L2}
\end{align}
The formula for $f_\alpha^*$ can be found in \req{eq:f_alpha_star} below.
\begin{remark}
Formally taking the $\alpha\to\infty$ limit of \eqref{eq:Df_alpha_L2} we arrive at what we call the Lipschitz $\infty$-divergence:
\begin{align}
    D^L_{\infty}(Q\|P)=&\sup_{ g\in \Lip_b^L(\mathbb{R}^n)}\{E_Q[ g]-E_P[\max\{ g,0\}]\}\,.
\end{align}
It is  straightforward to show that $D^L_{\infty}(Q\|P)=LW(Q,P)$, where $W$  is classical Wasserstein metric
\begin{align}
    W(Q,P)=\sup_{ g\in \Lip_b^1(\mathbb{R}^n)}\{E_Q[ g]-E_P[ g]\}\,,
\end{align}
 though they are expressed in terms of  different objective functionals (hence their performance can differ in practice).
\end{remark}
In numerical computations it can be  inconvenient to restrict one's attention to  bounded discriminators only. Fortunately, as shown in Theorem \ref{thm:unbounded_Lip} above,  the equality \eqref{eq:gen_f_def} remains true  when $\Gamma$ is expanded to include many unbounded $ g$'s. This fact justifies our use of unbounded discriminators (i.e., unbounded activation functions) in the following computations. 

As our baseline method we take the two-sided gradient-penalized Wasserstein GAN (WGAN-GP) from \cite{wgan:gp}:

{\bf WGAN-GP:}
\begin{align}\label{eq:Wgan}
 \inf_\theta\sup_{ g\in\Lip(\mathbb{R}^n)}\left\{E_{Q}[ g]-E_{P_\theta}[ g]-\lambda\int(\|\nabla g\|/L-1)^2d\rho_\theta\right\}\, ,
\end{align}
where $\lambda>0$ is the strength of the penalty regularization.
 Here, and below, we have  relaxed the  Lipschitz constraint to a gradient penalty (two-sided for WGAN-GP and one-sided otherwise; see Section \ref{sec:soft_constraint} for further discussion). We approximate the supreumum over $ g$ by the supremum over a  neural network family (the discriminator network). Again, the family of measures $P_\theta$ are the distributions of $X_\theta=h_\theta(X)$ where  $h_\theta$ is the generator neural network, parameterized by $\theta\in\Theta$, and we let $X$ be a Gaussian noise source. Finally, we let $\rho_\theta\sim TX_\theta+(1-T)Y$ where $X_\theta$, $Y\sim Q$, $T\sim Unif([0,1])$ are all independent (this choice of $\rho_\theta$ was used in \cite{wgan:gp}).  
 
 We compare WGAN-GP to the  Lipschitz $\alpha$-GANs and Lipschitz KL-GAN, defined based on \eqref{eq:Df_alpha_L2} and \eqref{eq:f_alpha_nu} respectively:
 
{\bf Lipschitz $\alpha$-GAN}
\begin{align}\label{eq:Lip_alpha_GAN}
 \inf_{\theta\in\Theta}\sup_{ g\in\Lip(\mathbb{R}^n)}\left\{E_{Q}[ g]- E_{P_\theta}[f_\alpha^*( g)]-\lambda\int\max\{0,\|\nabla g\|^2/L^2-1\}d\rho_\theta\right\}\,.
\end{align}
When we want to make the values of $\alpha$ and/or $L$ explicit we will  refer to these as the $D_\alpha^L$-GANs. By swapping $Q$ and $P_\theta$ one obtains another family of GANs, which we call the reverse Lipschitz $\alpha$-GANs (when clarity is needed, \eqref{eq:Lip_alpha_GAN} will be called a forward GAN). We note that forward and reverse GANs can have very different properties \cite{2017arXiv170100160G}. 

In the case of the KL-divergence one can  evaluate the optimization over $\nu$ in \eqref{eq:f_alpha_nu} (see \req{eq:R_Gamma}), leading to the following:

{\bf Lipschitz KL-GAN}
\begin{align}\label{eq:Lip_KL_GAN}
 \inf_{\theta\in\Theta}\sup_{ g\in\Lip(\mathbb{R}^n)}\left\{E_{Q}[ g]-\log E_{P_\theta}[e^{ g}]-\lambda\int\max\{0,\|\nabla g\|^2/L^2-1\}d\rho_\theta\right\}\,.
\end{align}
{ 
\begin{remark}
For numerical purposes the GAN \eqref{eq:Lip_KL_GAN}, obtained using the  representation \eqref{eq:Df_var_formula_nu}, performs significantly better than the GAN obtained from \eqref{eq:Df_var_formula}. This is due to the numerical issues inherent in computing $E_P[f^*_{KL}(g)]=E_P[\exp(g-1)]$, as compared to computing the cumulant generating function $\log E_P[\exp(g)]$; see also \cite{MINE_paper}. We also refer to \cite{birrell2020optimizing} for a more general perspective on finding tighter  variational representations of divergences.
\end{remark}
}

{ 
\subsection{Statistical Estimation of $(f,\Gamma)$-Divergences}

In numerical computations, we approximate the $(f,\Gamma)$-divergence by replacing expectations under $Q$ and $P$  in \eqref{eq:gen_f_def} or \eqref{eq:Df_Gamma_def2} with their $m$-sample empirical means using  i.i.d. samples from $Q$ and $P$ respectively, i.e., we estimate
\begin{align}\label{eq:f_Gamma_est}
    E[D_f^\Gamma(Q_m\|P_m)]=E\left[\sup_{g\in\Gamma,\nu\in\mathbb{R}}\{E_{Q_m}[g-\nu]-E_{P_m}[f^*(g-\nu)]\}\right]\,.
\end{align}
Note that, at fixed $g$ and $\nu$, the objective functional on the right-hand-side of \eqref{eq:f_Gamma_est} is an unbiased estimator of the $(f,\Gamma)$-divergence objective functional. Including the optimization over $g$ and $\nu$ we obtain a biased estimator  which is an upper bound on $D_f^\Gamma$, as shown in the  lemma below. 
\begin{lemma}\label{lemma:bias_bound}
Let $f\in\mathcal{F}_1(a,b)$, $\Gamma\subset\mathcal{M}_b(\Omega)$ be nonempty,  $Q,P\in\mathcal{P}(\Omega)$, and $Q_m$, $P_m$  be  empirical distributions constructed from $m$ i.i.d. samples from $Q$ and $P$ respectively.  Then
\begin{align}
    E[D_f^\Gamma(Q_m\|P_m)]\geq D_f^\Gamma(Q\|P)\,.
\end{align}
\end{lemma}
\begin{proof}
Using \eqref{eq:Df_Gamma_def2} we can compute
\begin{align}
E[D_f^\Gamma(Q_m\|P_m)]=&E\left[\sup_{g\in\Gamma,\nu\in\mathbb{R}}\{E_{Q_m}[g-\nu]-E_{P_m}[f^*(g-\nu)]\}\right]\\
\geq &\sup_{g\in\Gamma,\nu\in\mathbb{R}}E\left[E_{Q_m}[g-\nu]-E_{P_m}[f^*(g-\nu)]\right]\notag\\
=&\sup_{g\in\Gamma,\nu\in\mathbb{R}}\{E_{Q}[g-\nu]-E_{P}[f^*(g-\nu)]\}=D_f^\Gamma(Q\|P)\,.\notag
\end{align}
\end{proof}
\begin{remark}
As noted above, in the KL case one can evaluate the optimization over $\nu$ in \eqref{eq:f_Gamma_est}. This results in a biased objective functional due to the presence of the logarithm outside the expectation in \eqref{eq:R_Gamma}. This same issue was addressed earlier in \cite{MINE_paper}, e.g., by using sufficiently large minibatch sizes or an exponential moving average. This concern is not present in the objective functional for \eqref{eq:f_Gamma_est} (or \eqref{eq:Lip_alpha_GAN}).
\end{remark}
}

In the following  GAN examples  we work with Lipschitz functions and approximate the optimization over $\Lip(\mathbb{R}^n)$ by the optimization over some neural network family $g_\phi$, $\phi\in\Phi$, and  estimate the expectations using the $m$-sample empirical measures $Q_m$, $P_{m,\theta}$, $\rho_{m,\theta}$, e.g., we approximate the Lipschitz $\alpha$-GAN \eqref{eq:Lip_alpha_GAN} by
\begin{align}\label{eq:Lip_alpha_GAN_emp}
 \inf_{\theta\in\Theta}\sup_{ \phi\in\Phi}\left\{E_{Q_m}[ g_\phi]- E_{P_{m,\theta}}[f_\alpha^*( g_\phi)]-\lambda\int\max\{0,\|\nabla g_\phi\|^2/L^2-1\}d\rho_{m,\theta}\right\}\,.
\end{align}
Various neural network architectures are known to be universal approximators \cite{HORNIK1989359,Cybenko1989,pinkus_1999,10.5555/3295222.3295371,pmlr-v125-kidger20a}. Approximating the supremum over $ g\in \Lip(\mathbb{R}^n)$ by the supremum over a finite-dimensional neural network family essentially results in a lower bound on the original, intended divergence.  In the case of KL and R{\'e}nyi divergences, such an approximation scheme is known to lead to consistent estimators as the sample size and network complexity grows (see \cite{MINE_paper} and \cite{2020arXiv200703814B} respectively). Investigating the analogous consistency result for the $(f,\Gamma)$-divergence estimator is one avenue for future work.

\subsection{$(f,\Gamma)$-GANs  for Non-Absolutely-Continuous Heavy-Tailed   Distributions}\label{sec:submanifold_ex}

We mentioned above that the $f$-divergences are  better suited to heavy-tailed distributions, as compared to the Wasserstein metric. Before demonstrating this in the context of GANs we provide a   simple explicit example. Let $dQ=x^{-2}1_{x\geq 1}dx$ and $dP=(1+\delta)x^{-(2+\delta)}1_{x\geq 1}dx$ for $\delta>0$, i.e., the tail of $P$ decays faster than that of $Q$. For $\alpha>1$ we can use \req{eq:Df_def} to compute
\begin{align}
    D_{f_\alpha}(Q\|P)=\frac{1}{\alpha(\alpha-1)(1+\delta)^{\alpha-1}}\int_1^\infty x^{\delta \alpha-(2+\delta)}dx -\frac{1}{\alpha(\alpha-1)}\,,
\end{align}
and so $D_{f_\alpha}(Q\|P)<\infty$ for all $\delta\in(0,1/(\alpha-1))$. On the other hand, we can use the formula for the Wasserstein metric on $\mathcal{P}(\mathbb{R})$ from \cite{doi:10.1137/1118101} to compute
\begin{align}\label{eq:W_infty_example}
 W(Q,P)=&\int_{-\infty}^\infty|F_Q(t)-F_P(t)|dt=\int_1^\infty t^{-1}-t^{-(1+\delta)}dt=\infty
\end{align}
for all $\delta>0$ ($F_P$ and $F_Q$ denote the cumulative distribution functions). 

This calculation suggests that Lipschtiz $\alpha$-GANs may succeed for  heavy-tailed distributions, even when WGAN-GP fails  to converge.  On the other hand the Wasserstein metric can be finite and informative even when $Q$ and $P$ are non-absolutely continuous, unlike the classical $f$-divergences \eqref{eq:Df_def}. The $(f,\Gamma)$-divergences inherit both of these strengths from the Wasserstein and $f$-divergences (see Part 1 of Theorem \ref{thm:f_div_inf_convolution}), thus allowing for the training of GANs with heavy-tailed data and in the absence of absolute continuity. We demonstrate this via the following example, where  both the WGAN-GP  and classical $f$-GAN (i.e., without gradient penalty) fail to converge but the $(f,\Gamma)$-GANs succeed.

Here the data source, $Q$, is a mixture of four 2-dimensional t-distributions with $0.5$ degrees of freedom, embedded in a plane in 12-dimensional space; note that this is a heavy-tailed distribution, as the mean does not exist; this suggests that WGAN will have difficulty learning this distribution. The generator uses a 10-dimensional noise source and so the generator and data source are generally not absolutely continuous with respect to one another (the former has support equal to the full 12-dimensional space while the latter is supported on a 2-dimensional plane). This suggests one cannot use the classical  $f$-GAN \cite{f_GAN}, i.e., without gradient penalty (we confirmed  that they perform very poorly on this problem). The $(f,\Gamma)$-GANs allow us to address both of the above difficulties; heavy tails can be accommodated by an appropriate choice of $f$ and the lack of absolute continuity is addressed by using a $1$-Lipschitz constraint (as in the Wasserstein metric).
\begin{figure}
  \centering
\begin{subfigure}[b]{1\textwidth}
   \includegraphics[width=1\linewidth]{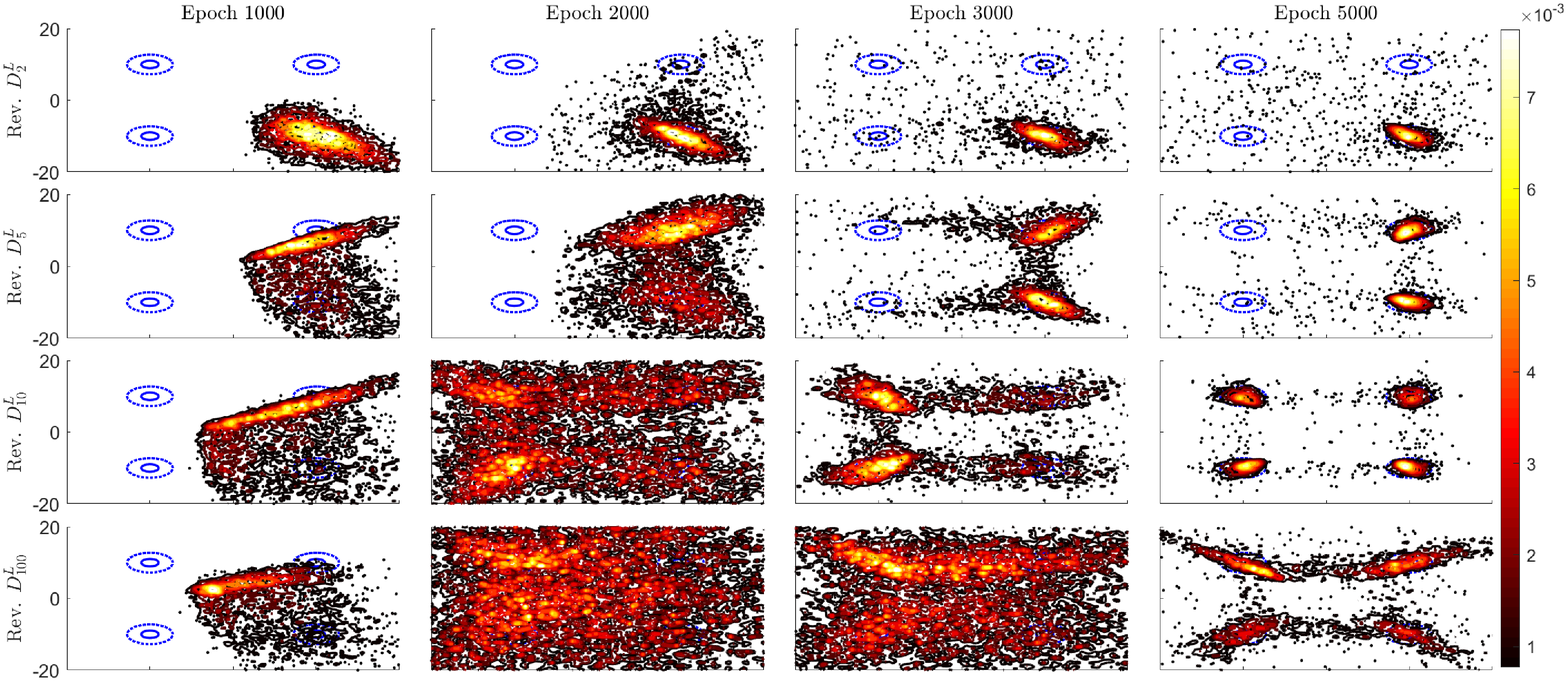}
   \caption{}
\end{subfigure}

\begin{subfigure}[b]{1\textwidth}
   \includegraphics[width=1\linewidth]{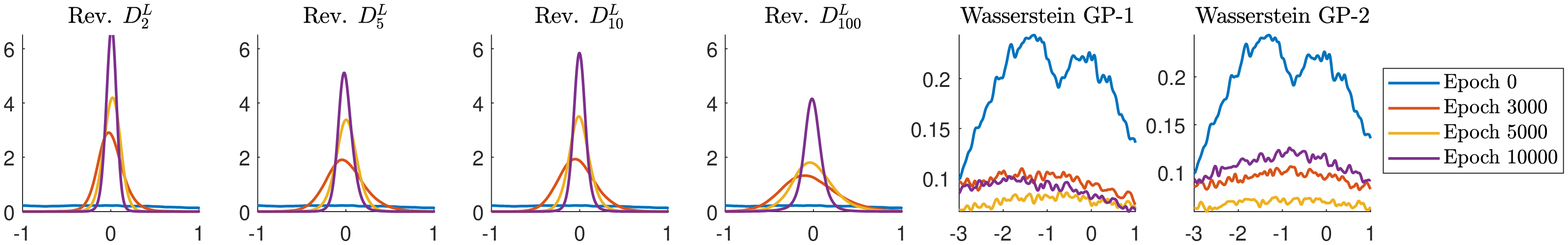}
   \caption{}
\end{subfigure}

\begin{subfigure}[b]{.99\textwidth}
 \begin{minipage}[b]{0.49\linewidth}
  \centering
\includegraphics[scale=.49]{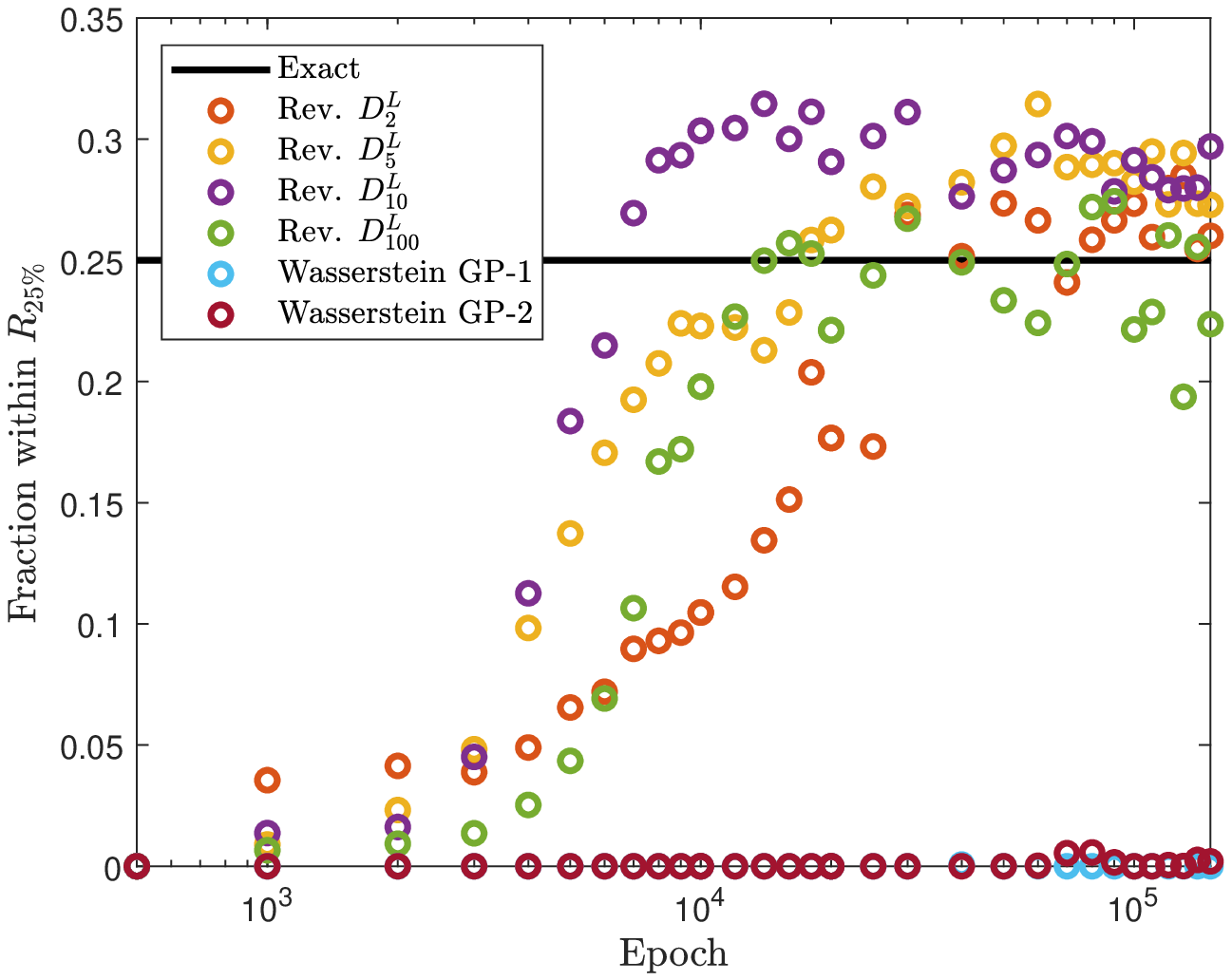}\end{minipage}
\begin{minipage}[b]{0.49\linewidth}
\centering
\includegraphics[scale=.49]{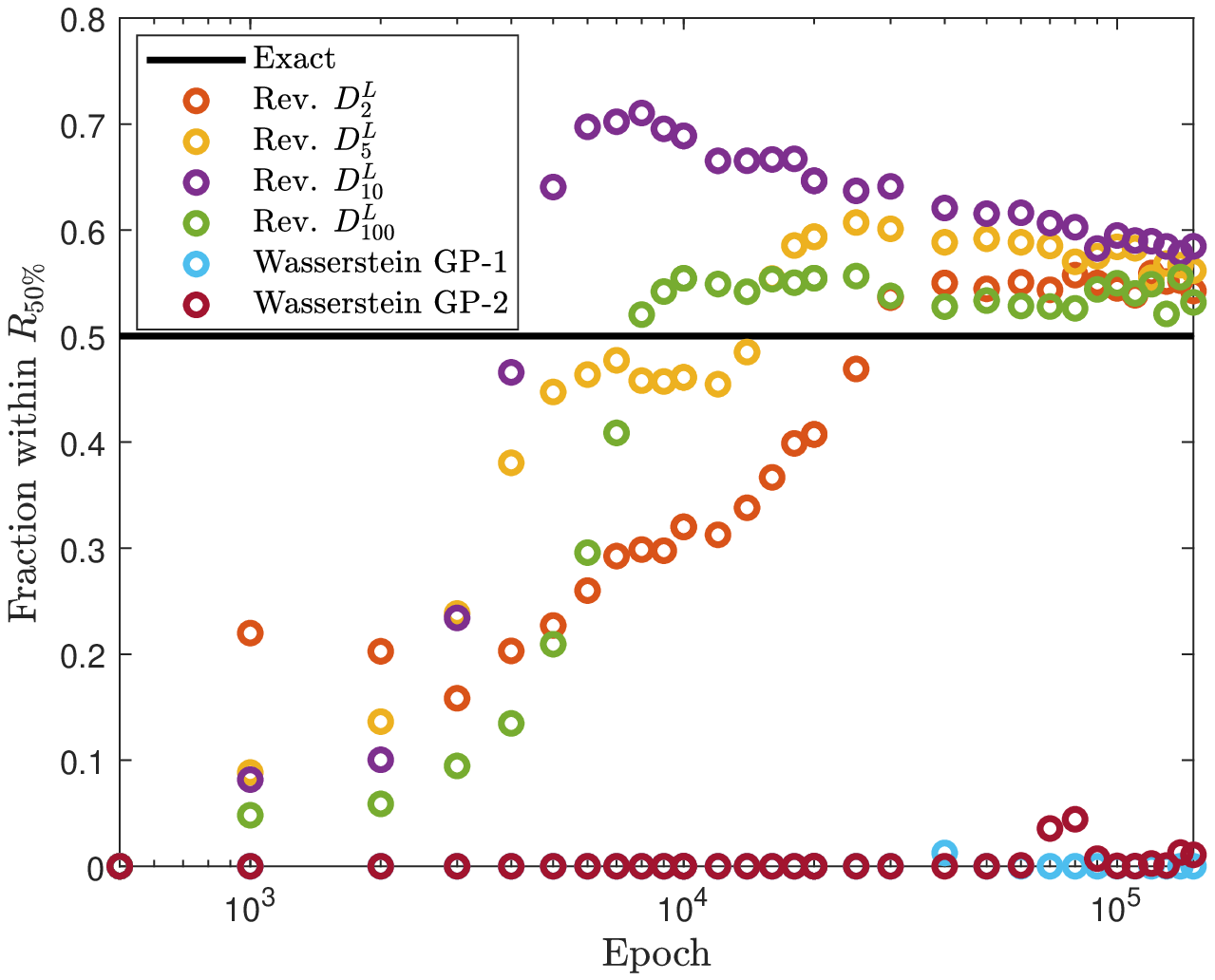} \end{minipage}
   \caption{}
\end{subfigure}
   \caption{We present generator samples and their statistical behavior from Wasserstein and reverse Lipschitz $\alpha$-GAN methods. The dataset used in training consists of 5000 samples from a mixture of four 2-dimensional t-distributions with $0.5$ degrees of freedom that are  embedded  in a plane in 12-dimensional space. Panel (a) shows the projection onto the 2-dimensional support plane (each column shows the result after a given number of training epochs; the solid and dashed blue ovals mark the 25\% and 50\% probability regions, respectively, of the data source while the heat-map shows the generator samples. Panel (b) shows the generator distribution, projected onto components orthogonal to the support plane. Values concentrated around zero indicate convergence to the sub-manifold.   In Panel (c) we show the fraction of generator samples, projected onto the 2D support plane of the measure, that are within the $25\%$ and $50\%$ probability regions.  In this example we used gradient-penalty parameter values   $\lambda=10$, $L=1$;  {  Wasserstein GAN was run with both 1-sided and 2-sided gradient penalties (GP-1 and GP-2 respectively).} In all cases the generator and discriminator have three fully connected hidden layers of 64, 32, and 16 nodes respectively, and with ReLU activation functions. The generator uses a 10-dimensional Gaussian noise source.  Each SGD iteration was performed with a minibatch size of 1000 {  and  5 discriminator iterations were performed for every one generator iteration}.   Computations were done in TensorFlow and we used the RMSProp SGD algorithm with a learning rate of $2\times 10^{-4}$.
     }\label{fig:gen_alpha_GAN}
\end{figure}

In Figure \ref{fig:gen_alpha_GAN} below we show generator samples for Wasserstein GAN,   as in \req{eq:Wgan} and \cite{wgan:gp}, and for various reverse Lipschitz $\alpha$-GANs \eqref{eq:Lip_alpha_GAN}. Specifically,  panel (a)  shows the projection onto the 2-dimensional support plane of $Q$ (the heat-map shows samples from the generator and the data source, $Q$, is illustrated by the blue ovals) and panel  (b) shows the generator distribution, projected onto components orthogonal to the support plane.  Panel (a) does not show  WGAN-GP samples, as WGAN-GP failed to converge in this example; this is demonstrated   in panel (b) wherein we see that the Lipschitz $\alpha$-GAN samples concentrate near the support plane (at $0$) while the WGAN-GP samples spread out away from the support plane. The  classical $f$-GAN without gradient penalty \cite{f_GAN}, which we don't show here, similarly failed to converge; this is unsurprising  due to the lack of absolute continuity.  Again, we can see that WGAN-GP fails to converge, while the Lipschitz $\alpha$-GANs perform well. Some $\alpha$'s perform significantly better than others, making it an important hyperparameter to tune in this case. {  Results from a second set of runs, using a larger sample set, are shown in Figure \ref{fig:gen_alpha_GAN_2} in Appendix \ref{app:extra_figs}; the conclusions are similar.}  Forward Lipschitz GANs and forward Lipschitz KL-GANs all experienced blow-up and so they are not shown here.  This behavior is reasonable when one considers the fact that $Q$ is heavy tailed, while $P_\theta$ is not (it is generated by pushing forward Gaussian noise by Lipschitz functions), and so $D_f(Q\|P_\theta)=\infty$, while $D_f(P_\theta\|Q)<\infty$ (see \req{eq:Df_def}). As we have already demonstrated the inability of the Wasserstein metric to compare heavy-tailed distributions (see \req{eq:W_infty_example}), it is reasonable to conjecture that the finiteness of $D_{f_\alpha}$ is key in determining the success of the $D_\alpha^L$-GAN. {  Interestingly, the Lipschitz constraint also appears to be key to the convergence of the method, something one would not anticipate solely based on finiteness of the corresponding divergences.  We illustrate this with Figure \ref{fig:2Dstudent_GAN} in Appendix \ref{app:extra_figs}, where we apply the same method to the mixture of four 2-dimensional t-distributions, but without the high-dimensional embedding.  In this case, the classical $f$-divergence is finite, however we find that the classical $f$-GAN fails to converge --WGAN also fails-- but the $(f,\Gamma)$-GANs succeed. The theoretical understanding of this behavior is an interesting question, 
but we will not pursue it further here.}

\subsection{ Strict Convexity and Enhanced Stability of $(f,\Gamma)$-GANs}\label{ex:C10}
Even in the absence of heavy tails, we find that the Lipschitz $\alpha$-GANs can outperform  WGAN-GP, as measured both by accuracy on quantities of interest as well as improved stability. The improved stability can be motivated by a simple (formal) calculation of the Hessian of the objective functional in \req{eq:Df_Gamma_no_shift},
\begin{align}\label{eq:loss:concavity:general}
    H_f[g;Q,P]\equiv E_Q[  g]-E_P[f^*( g)]
\end{align}
(see Appendix \ref{app:Taylor} for an analysis of the  objective functional in the non shift-invariant case \eqref{eq:gen_f_def}). Let  $g_0\in \Gamma$ and    perturb in some direction $\psi$, i.e.,  take a line segment $ g_\epsilon= g_0+\epsilon\psi\in\Gamma$. Then
\begin{align}\label{eq:H_f_inv_hessian}
    \frac{d^2}{d\epsilon^2}|_{\epsilon=0} H_f[g_\epsilon;Q,P]= -E_P[(f^*)^{\prime\prime}( g_0)\psi^2]\,.
\end{align}
Convexity of $f^*$ implies $(f^*)^{\prime\prime}\geq 0$. If we have $(f^*)^{\prime\prime}>0$ then \eqref{eq:H_f_inv_hessian} implies the objective functional   is strictly concave at $g_0$ in all directions, $\psi$, are nonzero on a set of positive $P$-probability.  {  This strict concavity implies that  the maximization problem \eqref{eq:gen_f_def} is a strictly convex optimization problem and suggests that numerical computation of $D_f^\Gamma(Q\|P)$ via   \eqref{eq:gen_f_def} may generally be more stable than  computation of the $\Gamma$-IPM \eqref{eq:gen_wasserstein}, as the latter uses a linear objective functional.
Indeed, in  \cite{daskalakis2018training} the authors demonstrated  that gradient descent/ascent  dynamics (used for training of GANs) oscillate without converging to the optimum for the Wasserstein-GAN loss function \eqref{eq:gen_wasserstein} in the special case where $\Gamma$ consists of a parametric family of linear functions. In this case,  more sophisticated algorithms  such as training with optimism \cite{daskalakis2018training} or two-step extra-gradient approaches \cite{abs-1901-08511} were required to guarantee convergence. 
Here our $(f, \Gamma)$ interpolation replaces   the optimization of a linear objective functional  
in the case of the $\Gamma$-IPM \eqref{eq:gen_wasserstein}
with  the   strictly concave problem \eqref{eq:gen_f_def}. In the case of a linear discriminator space  $\Gamma$, we obtain a complete theoretical justification based on  the concavity calculation \eqref{eq:H_f_inv_hessian}. In particular,  
we consider \eqref{eq:loss:concavity:general} where
\begin{align}\label{eq:loss:concavity:linear}
    H_f[g_\phi;Q,P] = E_Q[g_\phi]-E_P[f^*(g_\phi)]\, ,
\end{align}
and where we assume that  $\Gamma=\{g=g_\phi(x): \phi=(\phi_1, \phi_2,...,\phi_k)\in D\}$
is a  parametric linear family ($D$ is a closed, convex subset of $\mathbb{R}^k$), i.e., for any constants $a_0, a_1$ and any parameter values $\phi_0$, $\phi_1$ we have
\begin{equation}\label{eq:linear:Gamma}
    g_{a_0\phi_0+a_1\phi_1}(x)=a_0g_{\phi_0}(x)+a_1g_{\phi_1}(x)\, .
\end{equation}
Using \eqref{eq:linear:Gamma} and considering  \eqref{eq:H_f_inv_hessian} for any $g=g_{\phi_0}$ and  $\psi(x)=g_{\phi_1}$, $g_\epsilon=g_{\phi_0}+\epsilon g_{\phi_1}$,
 we readily have 
\begin{align}\label{eq:H_f_inv_hessian_linear}
     \phi_1^\intercal\nabla_\phi^2H_f[g_{\phi_0};Q,P]\phi_1 = \frac{d^2}{d\epsilon^2}|_{\epsilon=0} H_f[g_\epsilon;Q,P]
    = -E_P[(f^*)^{\prime\prime}( g_{\phi_0}) g_{\phi_1}^2]\, , 
\end{align}
provided all expected values are finite. 
As in \eqref{eq:H_f_inv_hessian}, this analysis implies the  strict concavity of \eqref{eq:loss:concavity:linear} with respect to the linear parametrization $\phi$. Thus our analysis covers linear spaces $\Gamma$ such as linear combinations of splines or reproducing kernel Hilbert spaces (RKHS). However, when  $\Gamma$ is a family of neural networks then the $g_\phi$'s are not linear in $\phi$ and  the above analysis does not apply. We will not pursue the theoretical analysis of this important case here but instead we will carry out an empirical study that explores the improved stability that (local) strict concavity would imply.
 }

\begin{figure}[ht]

\begin{minipage}[b]{0.45\linewidth}
  \centering
\includegraphics[scale=.50]{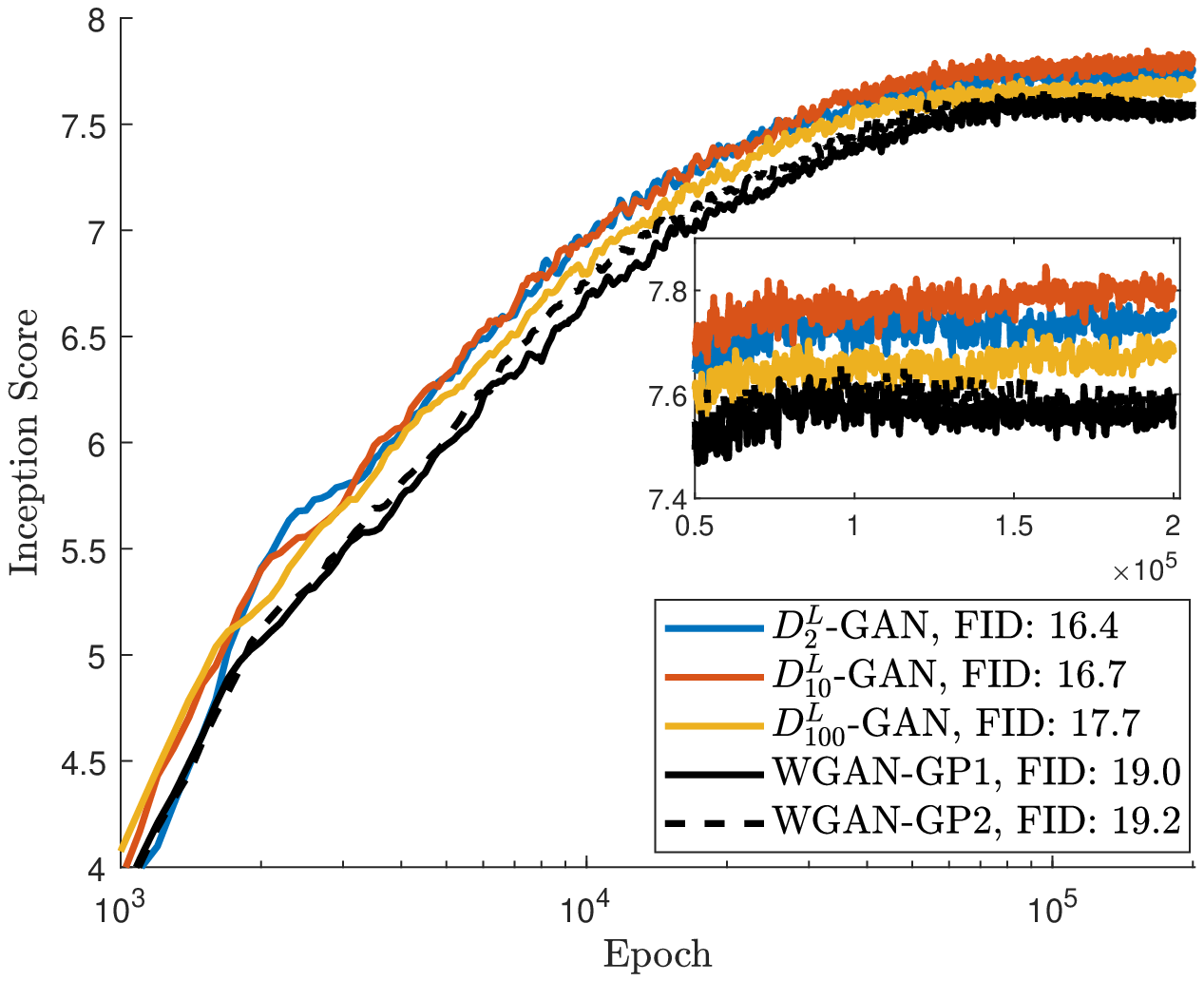} \end{minipage}
\hspace{0.5cm}
\begin{minipage}[b]{0.45\linewidth}
\centering
\includegraphics[scale=.50]{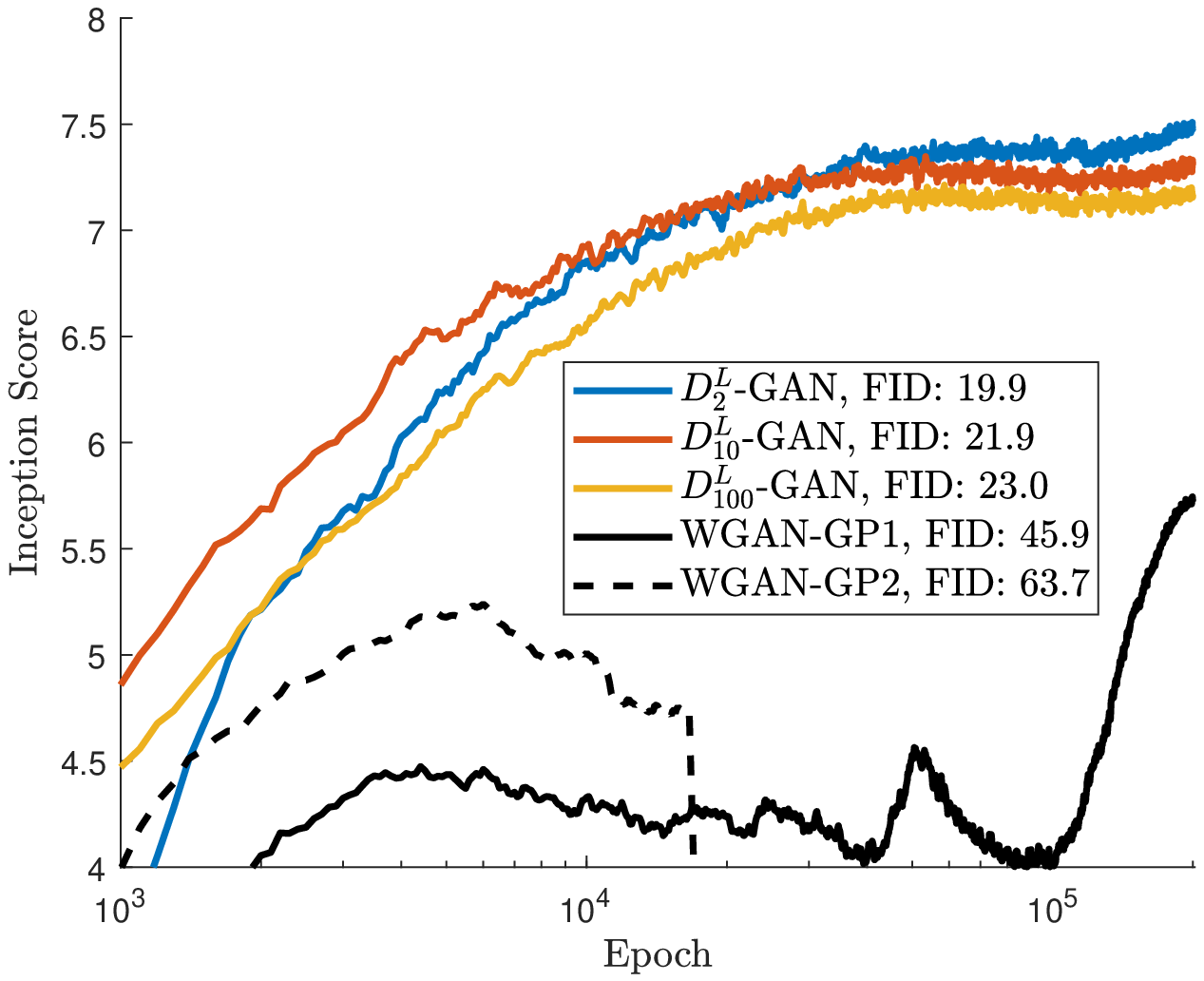} \end{minipage}

\caption{ {  Comparison between Lipschitz $\alpha$-GAN and WGAN-GP (both 1 and 2-sided) on the CIFAR-10 dataset.  Here we plot the inception score as a function of the number of training epochs (moving average over the last 5 data points, with results  averaged over 5 runs). We also show the averaged final FID score in the legend, computed using 50000 samples  from both $Q$ and $P$. The neural network architecture is as in Appendix F of \cite{wgan:gp}; in particular, it employs residual blocks. The left panel used an initial learning rate of $0.0002$ (the same as in \cite{wgan:gp}) while in the right panel we used an initial learning rate of $0.001$.  Here, and in other similar tests, we find the Lipschitz $\alpha$-GANs to be significantly more stable and require less tuning of the learning rate; in particular, none of the WGAN-GP2 runs shown in the right panel were able to complete successfully.  } }\label{fig:C10_GAN}
\end{figure}

In Figure \ref{fig:C10_GAN} we demonstrate both improved performance and improved stability of the Lipschitz $\alpha$-GANs, as compared to WGAN-GP, on the CIFAR-10 dataset \cite{krizhevsky2009learning}, which consists of 32x32 RGB images from 10 classes.  We use the same ResNet neural network architecture  as in \cite[Appendix F]{wgan:gp} and focus on  evaluating the benefits of simply modifying the objective functional. We  employ the adaptive learning rate Adam Optimizer method \cite{kingma2014adam} {  using the hyperparameter values shown in Algorithm 1 of \cite{wgan:gp} (note that in \cite{wgan:gp}, $\alpha$ denotes the learning rate parameter and should not be confused with our use for $\alpha$-divergences)}. We show the inception score as a function of the number of training epochs; the inception score \cite{salimans2016improved} is a commonly used performance measure for  evaluating the diversity of images produced by a GAN. It uses a pre-trained classifier to estimate the number of distinct classes produced by the generator and so, when applied to CIFAR-10, values closer to 10 are considered better. {   In the legends we also show the final FID score  achieved by each method. FID score is a performance measure that computes a distance between feature vectors of a classification model when applied to the original data, as compared to the generated samples  \cite{10.5555/3295222.3295408}; a lower FID score is better.} In the left panel of Figure \ref{fig:C10_GAN} we show the results using an initial learning rate of $0.0002$; {  we find a small improvement in inception score and substantial improvement in FID score when using the Lipschitz $\alpha$-GANs, as compared to WGAN-GP (either $1$ or $2$-sided).} 
{  In this example we find the performance to be relatively insensitive to the value of $\alpha$.  }

In addition to the performance improvement, we find the Lipschitz $\alpha$-GANs to be far less sensitive to the choice of learning rate. In the right panel of Figure \ref{fig:C10_GAN} we show  results using an initial learning rate of $0.001$; here we observe significant degradation of the performance of WGAN-GP, but only a slight impact on the Lipschitz $\alpha$-GANs. We conjecture that this increased stability is due to the strong concavity of the $(f,\Gamma)$-divergence objective functionals. { 
Regarding increased stability, these numerical findings,  the analysis for  a general (non-parameterized) function space $\Gamma$ in \eqref{eq:H_f_inv_hessian}, as well as for  the  linear parametric case \eqref{eq:H_f_inv_hessian_linear}
provide only preliminary indications for the conjecture;   a dedicated analysis for general parameterized $\Gamma$'s that will include nonconvex parametric families such as neural networks is clearly necessary but we will not pursue it further here.

\subsubsection{Enhanced Stability and Spectral Normalization}\label{sec:SN}
In \cite{miyato2018spectral} the authors showed that spectral normalization, which directly controls the Lipschitz constant of each layer of a neural network by setting the largest singular value of its weight matrix to $1$, provides enhanced stability as compared to WGAN-GP and at a lower computational cost (see Figures 1 and 2 in \cite{miyato2018spectral}). Their method, which uses  the Jensen-Shannon divergence, is equivalent to  \req{eq:Df_Gamma_no_shift} (i.e., they do not include an optimization over shifts as in \eqref{eq:Df_Gamma_def2}) with a change of variables $g=\log(D)$ and using a function space $\Gamma$ that consists of a neural network family with spectral normalization. In this example we use a spectral normalization function space in our method \eqref{eq:Df_Gamma_def2}; this falls under the purview of Theorem \ref{thm:general_ub}  (see Table \ref{tab:related_work}). We provide empirical evidence that the improved stability they observed is at least partially due to the strict concavity of the objective functional. Specifically, we find that WGAN with spectral normalization fails to inherit this improved stability and even fails to outperform WGAN-GP.  Our results demonstrate that combining spectral normalization with other (strictly convex) objective functionals can   enhance stability, similar both to what was observed in \cite{miyato2018spectral} and also to what we found in Figure \ref{fig:C10_GAN}. Here we again study the case $f=f_\alpha$, denoting these methods by $D_\alpha^{SN}$; results are shown in Figure \ref{fig:C10_sn_GAN}. }
\begin{figure}
  \centering
\includegraphics[scale=.50]{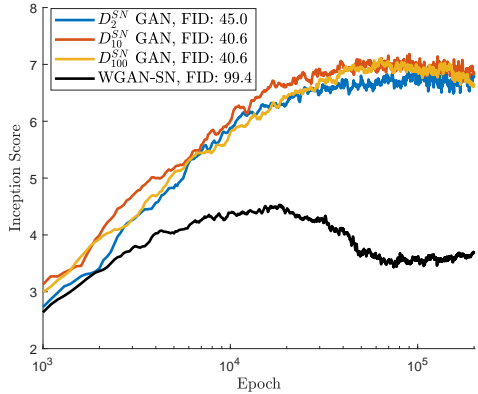}
\caption{ {  A comparison between Lipschitz $\alpha$-GAN and WGAN, both using spectral normalization (SN) to enforce Lipschitz constraints.   We used an initial learning rate of $0.0001$ and otherwise employed same ResNet architecture and hyperparameters as in Figure \ref{fig:C10_GAN}. In particular, we did not attempt to further optimize the architecture when changing from a gradient penalty to SN. None of the methods performed as well as their gradient penalty counterparts from Figure \ref{fig:C10_GAN}, but note the especially poor performance of WGAN-SN.  This suggests a further robustness of our methods to the use of sub-optimal architecture and hyperparameters.   }}\label{fig:C10_sn_GAN}
\end{figure}

\section{Conclusion}\label{sec:concl}
We have provided a systematic and rigorous exploration of the properties of the $(f,\Gamma)$-divergences, as defined in \req{eq:Df_Gamma_intro}. This work was motivated by the need for a flexible collection of novel divergences that combine key properties from $f$-divergences and Wasserstein metrics, such as the ability to work with heavy tails and with not-absolutely continuous distributions. {  A large list of proposed GANs fall under the presented mathematical framework (see Table \ref{tab:related_work}), unifying to a considerable extent the loss formulation of GANs.} We have illustrated the utility of the $(f,\Gamma)$-divergences in the training of GANs, showing both an increased domain of applicability and improved convergence stability. {  The theoretical results  allow for a wide range of choices on $f$ and $\Gamma$. We have shown that there are families of distributions that are better suited for $(f,\Gamma)$-divergence over either $f$-divergence or $\Gamma$-IPM. A more systematic exploration on the selection of proper $f$ and $\Gamma$ will add practical value from a practitioner's perspective, but further and more elaborate experimentation is required, along with a need for new theoretical insights. In the future we  intend to further study the stability, the related statistical estimation theory, and  explore these new divergences in additional  challenging settings } such as high-dimensional time-series generation, extreme events prediction, mutual information estimation, and uncertainty quantification for heavy-tailed distributions and in the absence of absolute continuity.

\appendix

\section{Properties of the Legendre Transform of $f\in\mathcal{F}_1(a,b)$}
Here we collect a number of important properties regarding the LT of function in $ \mathcal{F}_1(a,b)$. Recall that we use the same notation for the convex LSC extension $f:\mathbb{R}\to(-\infty,\infty]$ (see Definition \ref{def:F_1}). First we state an important continuity result \cite[Theorem 10.1]{rockafellar1970convex}.
\begin{lemma}\label{lemma:f_star_cont}
Let $f\in\mathcal{F}_1(a,b)$. Then $f^*(y)=\sup_{x\in(a,b)}\{yx-f(x)\}$ and $f^*$ is continuous on $\overline{\{f^*<\infty\}}$.
\end{lemma}

\begin{lemma}\label{eq:f_star_lb}
Let $f\in\mathcal{F}_1(a,b)$. Then $f^*(y)\geq y$ for all $y\in\mathbb{R}$.
\end{lemma}
\begin{proof}
$ f^*(y)=\sup_{x\in\mathbb{R}}\{yx-f(x)\}\geq 1\cdot y-f(1)=y$.
\end{proof}
{ 
\begin{lemma}\label{lemma:dom=R}
If $f\in\mathcal{F}_1(a,b)$ is superlinear, i.e.,  $\lim_{x\to\pm\infty}f(x)/|x|=\infty$, then $\{f^*<\infty\}=\mathbb{R}$.
\end{lemma}
\begin{proof}
Suppose $y\in\mathbb{R}$ with $f^*(y)=\infty$, i.e., $\sup_{x\in (a,b)}\{yx-f(x)\}=\infty$.  Then there exists $x_n\in(a,b)$ with $yx_n-f(x_n)\to\infty$.  We can take a subsequence $x_{n_j}$ with $x_{n_j}\to x\in[a,b]$. If $x$ is finite then continuity of $f$ on $\overline{(a,b)}$ allows us to compute $f(x)=-\infty$, a contradiction.  If $x$ is infinite then we write
\begin{align}
    |x_{n_j}|(\sgn(x_{n_j})y-f(x_{n_j})/|x_{n_j}|)\to \infty\,,
\end{align}
which contradicts $f(x_{n_j})/|x_{n_j}|\to\infty$. This completes the proof.
\end{proof}
}
\begin{lemma}\label{lemma:f_star_nondec}
Let $f\in\mathcal{F}_1(a,b)$ and suppose there exists $x_n\in\mathbb{R}$ with $x_n\to-\infty$ and $f^*(x_n)$ uniformly bounded above.  Then $f^*$ is nondecreasing.
\end{lemma}
\begin{proof}
Suppose not.  Then there exists $y_1<y_2$ with $f^*(y_1)>f^*(y_2)$ (in particular, $f^*(y_2)<\infty$). Take $N$ such that for $n\geq N$ we have $x_n<y_1$.  For $n\geq N$ let $t_n=(y_1-x_n)/(y_2-x_n)$.  Then $t_n\in(0,1)$, $1-t_n=(y_2-y_1)/(y_2-x_n)\to 0$ as $n\to\infty$ and $t_n y_2+(1-t_n)x_n=y_1$.  Hence
\begin{align}
f^*(y_1)\leq t_n f^*(y_2)+(1-t_n)f^*(x_n)\,.
\end{align}
Note that this implies $f^*(y_1)<\infty$. We therefore have
\begin{align}
f^*(x_n)\geq \frac{f^*(y_1)-f^*(y_2)}{1-t_n}+f^*(y_2)\to\infty
\end{align}
as $n\to\infty$.  This is a contradiction and so we are done.
\end{proof}

\begin{lemma}\label{lemma:f_star_bounded_or_inc}
Let   $f\in\mathcal{F}_1(a,b)$.  Then one of the following holds:
\begin{enumerate}
\item $f^*$ is bounded below.
\item The set $I=\{y:f^*(y)<\infty\}$ is of the form $I=(-\infty,d)$ or $I=(-\infty,d]$ for some $d\in(-\infty,\infty]$ and $f^*$ is nondecreasing.

\end{enumerate}
In addition, if $f^*$ is not bounded below then there exists $b\leq 0$ such that   $f^*|_{(-\infty,b]}\leq 0$ and $f^*|_{(b,\infty)}\geq 0$.
\end{lemma}
\begin{proof}
Suppose $f^*$ is not bounded below.  Take $y_n\in I$ with $f^*(y_n)\to -\infty$.  We know $f^*(y)\geq y$, hence $y_n\to-\infty$.  Let $d=\sup I$. The set $I$ is convex, therefore   $(y_n,d)\subset  I$ for all $n$. Hence $(-\infty,d)\subset I\subset (-\infty,d]$.

Now suppose we have $x_1<x_2$ with  $f^*(x_1)>f^*(x_2)$.  With $y_n$ as above, take  $n$ such that $y_n<x_1$ and $f^*(y_n)<f^*(x_2)$ and let $t=(x_2-x_1)/(x_2-y_n)\in(0,1)$.   $f^*$ is convex, hence
\begin{align}
f^*(x_1)=&f^*(ty_n+(1-t)x_2)\leq tf^*(y_n)+(1-t)f^*(x_2)\\
\leq&  tf^*(x_2)+(1-t)f^*(x_2)=f^*(x_2)<f^*(x_1)\,.\notag
\end{align}
This is a  contradiction, therefore $f^*$ is nondecreasing.   

If $f^*$ is not bounded below then let $b=\sup\{x:x\leq 0,f^*(x)\leq 0\}$.  The  properties $f^*|_{(-\infty,b]}\leq 0$ and $f^*|_{(b,\infty)}\geq 0$ follow from the fact that $f^*(0)\geq 0$ and $f^*$ is non-decreasing and is continuous on $\overline{I}$ (see Lemma \ref{lemma:f_star_cont}).
\end{proof}

\begin{lemma}\label{lemma:f_star_inc}
Let $f\in\mathcal{F}_1(a,b)$ with $a\geq 0$.  Then $f^*$ is nondecreasing and $\{f^*<\infty\}=(-\infty,d)$ for $d\in(-\infty,\infty]$ or $\{f^*<\infty\}=(-\infty,d]$ for $d\in\mathbb{R}$.
\end{lemma}
\begin{proof}
 Let $y_1<y_2$.  Then $xy_1-f(x)\leq xy_2-f(x)$ for all $x\geq 0$.  $f|_{(-\infty,0)}=\infty$, hence
    \begin{align}
        f^*(y_1)=\sup_{x\in\mathbb{R}}\{y_1x-f(x)\}=\sup_{x\geq 0}\{y_1x-f(x)\}\leq\sup_{x\geq 0}\{y_2x-f(x)\}=f^*(y_2)\,.
    \end{align}
\end{proof}

Next we give several results pertaining to the derivative of a convex function and its LT. A key tool will be the following decomposition of a convex function into an affine part and a remainder which can be found in \cite{LieseVajda}:
\begin{lemma}\label{lemma:convex_taylor}
Let $f\in\mathcal{F}_1(a,b)$. For $x,y\in(a,b)$ we have
\begin{align}\label{eq:convex_taylor}
f(y)=f(x)+f_+^\prime(x)(y-x)+R_{f}(x,y)\,,
\end{align}
where $R_f\geq 0$, $R_f(x,x)=0$, and if $z$ is between $x$ and $y$ then $R_f(x,z)\leq R_f(x,y)$.
\end{lemma}

Using this we can derive an explicit formula for $f^*$ and prove several useful identities.
\begin{lemma}\label{lemma:f_star_formula}
Let $f\in\mathcal{F}_1(a,b)$ and $y\in\{f^*<\infty\}^o$.  Then 
\begin{align}
f^*(y)=y (f^*)^\prime_+(y)-f((f^*)^\prime_+(y))\,.
\end{align}
\end{lemma}
\begin{proof}
By assumption, $I\equiv \{f^*<\infty\}$ has nonempty interior and so
\begin{align}
f(x)=\sup_{z\in I^o}\{zx-f^*(z)\}
\end{align}
for all $x$. Applying Lemma \ref{lemma:convex_taylor} to the convex function $f^*$ on the interval $I^o$ gives 
\begin{align}
f^*(z)=f^*(y)+(f^*)^\prime_+(y)(z-y)+R_{f^*}(y,z)
\end{align}
for all  $z\in I^o$.  The assumption $y\in I$ implies $(f^*)^\prime_+(y)$ exists and is finite. Hence
\begin{align}
f((f^*)^\prime_+(y))=&\sup_{z\in I^o}\{z(f^*)^\prime_+(y)-f^*(z)\}\\
=&\sup_{z\in I^o}\{z(f^*)^\prime_+(y)-f^*(y)-(f^*)^\prime_+(y)(z-y)-R_{f^*}(y,z)\}\notag\\
=&(f^*)^\prime_+(y)y-f^*(y)-\inf_{z\in I^o}R_{f^*}(y,z)=(f^*)^\prime_+(y)y-f^*(y)\,.\notag
\end{align}
\end{proof}

\begin{lemma}\label{lemma:nu0}
Let $f\in\mathcal{F}_1(a,b)$ and define $\nu_0=f_+^\prime(1)$.  Then:
\begin{enumerate}
    \item $f^*(\nu_0)=\nu_0$.
    \item If $\nu_0\in\{f^*<\infty\}^o$ and $f$ is strictly convex on a neighborhood of $1$ then $(f^*)^\prime_+(\nu_0)=1$.
\end{enumerate}
\end{lemma}
\begin{proof}
\begin{enumerate}
    \item Using Lemma \ref{lemma:convex_taylor} we can compute
\begin{align}
f^*(\nu_0)=&\sup_{x\in(a,b)}\{\nu_0 x-f(x)\}=\sup_{x\in(a,b)}\{\nu_0 x-(\nu_0(x-1)+R_f(1,x))\}\\
=&\nu_0-\inf_{x\in(a,b)}R_f(1,x)=\nu_0\,.\notag
\end{align}
\item Using Lemma \ref{lemma:f_star_formula} along with Part 1 of this lemma we can write
\begin{align}\label{eq:f_star_prime_nu0_1}
f((f^*)^\prime_+(\nu_0))=\nu_0 (f^*)^\prime_+(\nu_0)-f^*(\nu_0)=\nu_0 ((f^*)^\prime_+(\nu_0)-1)\,.
\end{align}
In particular, we see that $(f^*)^\prime_+(\nu_0)\in\{f<\infty\}\subset\overline{(a,b)}$. Using Lemma \ref{lemma:convex_taylor}  we can write 
\begin{align}\label{eq:f_taylor_1}
f(x)=f_+^\prime(1)(x-1)+R_f(1,x)
\end{align}
for $x\in (a,b)$.  Therefore
\begin{align}
f((f^*)^\prime_+(\nu_0))=\nu_0((f^*)^\prime_+(\nu_0)-1)+R_f(1,(f^*)^\prime_+(\nu_0))
\end{align}
(this holds even if $(f^*)^\prime_+(\nu_0)$ equals either $a$ or $b$, as can be seen by taking limits in \eqref{eq:f_taylor_1} and using the continuous extension of $R_f(1,\cdot)$). Combining this with \req{eq:f_star_prime_nu0_1} we find 
\begin{align}
R_f(1,(f^*)^\prime_+(\nu_0))=0\,.
\end{align}
Using Lemma \ref{lemma:convex_taylor} (and taking limits if necessary) we obtain $0\leq R_f(1,z)\leq R_f(1,(f^*)^\prime_+(\nu_0))$ for all $z$ between $1$ and $(f^*)^\prime_+(\nu_0)$, and hence $R_f(1,z)=0$ for all such $z$.   If $(f^*)^\prime_+(\nu_0)\neq 1$ then this, combined with \req{eq:convex_taylor}, implies that $f$ is affine on the (non-trivial) line segment from $1$ to $(f^*)^\prime_+(\nu_0)$, which would contradict the assumption that $f$ is strictly convex on a neighborhood of $1$.  Therefore we conclude that $(f^*)^\prime_+(\nu_0)= 1$.
\end{enumerate}
\end{proof}
Finally, we will need  formulas for $f_{KL}^*$ and $f_\alpha^*$ from \req{eq:f_alpha_def}:
\begin{align}
&f_{KL}^*(y)=\exp(y-1),\label{eq:f_KL_star}\\
&f_\alpha^*(y)=\begin{cases} 
\alpha^{-1}(\alpha-1)^{\alpha/(\alpha-1)}y^{\alpha/(\alpha-1)}1_{y>0}+\frac{1}{\alpha(\alpha-1)},&\alpha>1\\
\infty 1_{y\geq 0}+\left(\alpha^{-1}(1-\alpha)^{-\alpha/(1-\alpha)}|y|^{-\alpha/(1-\alpha)}-\frac{1}{\alpha(1-\alpha)}\right)1_{y<0},&\alpha\in(0,1)\,.
   \end{cases}\label{eq:f_alpha_star}
\end{align}
Note that $f_{KL}$ and  $f_\alpha$ for $\alpha>1$ are all strictly admissible but $f_\alpha$ is not admissible  if $\alpha\in(0,1)$ (see Definition \ref{def:admissible}).  Theorem \ref{thm:f_div_inf_convolution} applies to $f_{KL}$ and to $f_\alpha$, $\alpha>1$ while  Theorem \ref{thm:general_ub} applies to the case $\alpha\in(0,1)$.

\section{Properties of the Classical $f$-Divergences}\label{app:f_div}
Here we collect a number of important properties of the classical $f$-divergences; see Definition \ref{def:f_div}. Perhaps most important is the following variational characterization; versions of this were proven in  \cite{Broniatowski,Nguyen_Full_2010}.
\begin{proposition}\label{prop:Df_var_formula}
Let   $f\in\mathcal{F}_1(a,b)$ and  $P$, $Q$ be probability measures on $(\Omega,\mathcal{M})$.  Then
\begin{align}\label{eq:Df_var_app}
D_f(Q\|P)=&\sup_{ g\in \mathcal{M}_b(\Omega)}\{ E_Q[ g]-E_P[f^*( g)]\}\\
=&\sup_{ g\in \mathcal{M}_b(\Omega)}\{ E_Q[ g]-\Lambda_f^P[g]\} \notag.
\end{align}
\end{proposition}

\begin{proof}
Lemma \ref{eq:f_star_lb} implies that $f^*(y)\geq y$ and so for $ g\in\mathcal{M}_b(\Omega)$ we  have $E_P[f^*( g)]\geq E_P[ g]>-\infty$. Therefore the objective functional in \req{eq:Df_var_app} is well defined. Fix $y_0\in\mathbb{R}$ with $f^*(y_0)\in\mathbb{R}$ and first consider the case $Q\not\ll P$.   Then there exists a measurable set $A$ with $P(A)=0$ and $Q(A)>0$. If we define $ g=R1_A+y_01_{A^c}$ then $g$ is bounded, measurable, and 
\begin{align}
&E_Q[ g]-E_{P}[f^*( g)]=RQ(A)+y_0Q(A^c)-f^*(y_0)P(A^c)
\end{align}
for all $R$.  Hence
\begin{align}
&\sup_{ g\in\mathcal{M}_b(\Omega)}\{E_Q[ g]-E_{P}[f^*( g)]\}\geq \lim_{R\to\infty}(RQ(A)+y_0Q(A^c)-f^*(y_0)P(A^c))=\infty\,.
\end{align}
Therefore $ \sup_{ g\in\mathcal{M}_b(\Omega)}\{E_Q[ g]-E_{P}[f^*( g)]\}=\infty=D_f(Q\|P)$.

Now suppose $Q\ll P$: $f$ is convex and LSC on $\mathbb{R}$, hence we can use convex duality to compute
\begin{align}
E_Q[ g]-E_P[f^*( g)]=&E_P[ g dQ/dP-f^*( g)]\leq E_P[\sup_{y\in\mathbb{R}}\{ydQ/dP-f^*(y)\}]\\
=&E_P[(f^*)^*(dQ/dP)]=E_P[f(dQ/dP)]=D_f(Q\|P)\notag
\end{align}
for all $ g\in\mathcal{M}_b(\Omega)$.
Therefore it suffices to show $\sup_{ g\in\mathcal{M}_{b}(\Omega)}\{E_Q[ g]-E_P[f^*( g)]\}\geq D_f(Q\|P)$.  Let $I\equiv\{f^*<\infty\}$. This is a nonempty interval in $\mathbb{R}$, hence we can find compact intervals $I_n\subset I_{n+1}\subset I$ with $\cup_n I_n=I$. $f^*$ is convex and LSC on $\mathbb{R}$, hence it is continuous on ${\overline{I}}$. In particular, $y\to yx-f^*(y)$ is continuous on the compact set $I_n$.  Therefore there exists measurable $ g_{n}:\Omega\to I_n$ such that $| g_n dQ/dP-f^*( g_{n})-\sup_{y\in I_n}\{ydQ/dP-f^*(y)\}|<1/n$. The functions  $ g_{n}$ are also bounded since $\range(g_n)\subset I_n$, a compact subset of $\mathbb{R}$, hence
\begin{align}
&\sup_{ g\in\mathcal{M}_{b}(\Omega)}\{E_Q[ g]-E_P[f^*( g)]\}\geq E_Q[ g_{n}]-E_P[f^*( g_{n})]=E_P[ dQ/dP  g_n-f^*( g_n)]\\
\geq &E_P[\sup_{y\in I_n}\{y dQ/dP-f^*(y)\}]-1/n\notag
\end{align}
  for all $n$. Therefore
\begin{align}
&\sup_{ g\in\mathcal{M}_{b}(\Omega)}\{E_Q[ g]-E_P[f^*( g)]\}\geq  \liminf_{n\to\infty}E_P[ \sup_{y\in I_n}\{y dQ/dP-f^*(y)\}]\,.
\end{align}
Fix a large enough $N$ such that $y_0\in I_N$. Then for $n\geq N$ we have $\sup_{y\in I_n}\{y dQ/dP-f^*(y)\}\geq y_0 dQ/dP-f^*(y_0)\in L^1(P)$  (recall that  $f^*(y_0)$ is finite). Therefore we can use Fatou's Lemma to compute
\begin{align}
\liminf_{n\to\infty} E_P[ \sup_{y\in I_n}\{y dQ/dP-f^*(y)\} ]\geq& E_P[\liminf_{n\to\infty}  \sup_{y\in I_n}\{y dQ/dP-f^*(y)\} ]=E_P[\sup_{y\in I}\{ydQ/dP-f^*(y)\}]\\
=&E_P[(f^*)^*(dQ/dP)]=D_f(Q\|P)\,.\notag
\end{align}
This completes the proof of the first equality.  To prove the second we compute
\begin{align}
  \sup_{ g\in \mathcal{M}_b(\Omega)}\{ E_Q[ g]-\Lambda_f^P[g]\}
  =&\sup_{ g\in \mathcal{M}_b(\Omega)}\{ E_Q[ g]-\inf_{\nu\in\mathbb{R}}\{\nu+E_P[f^*(g-\nu)]\}\}\\
    =&\sup_{ g\in \mathcal{M}_b(\Omega),\nu\in\mathbb{R}}\{E_Q[ g-\nu]-E_P[f^*(g-\nu)]\}\notag\\
        =&\sup_{ g\in \mathcal{M}_b(\Omega)}\{E_Q[ g]-E_P[f^*(g)]\}\,,
\end{align}
where in the last line we used   the fact that the map $\mathbb{R}\times\mathcal{M}_b(\Omega)\to\mathcal{M}_b(\Omega)$, $(\nu,g)\mapsto g-\nu$ is surjective.
\end{proof}
On a metric space, and assuming $f^*$ is everywhere finite, one can restrict the optimization in \eqref{eq:Df_var_app} to the set of bounded continuous functions.
\begin{corollary}\label{cor:Df_LSC}
Let   $f\in\mathcal{F}_1(a,b)$, $S$ be a metric space, and $Q,P\in\mathcal{P}(S)$. If $\{f^*<\infty\}=\mathbb{R}$  then
\begin{align}\label{eq:Df_var_Cb}
D_f(Q\|P)=\sup_{ g\in C_b(S)}\{ E_Q[ g]-E_P[f^*( g)]\}\,.
\end{align} 
In particular, $(Q,P)\mapsto D_f(Q\|P)$ is  lower semicontinuous.
\end{corollary}
\begin{proof}
To prove \req{eq:Df_var_Cb} we start with \req{eq:Df_var_app} and use the extension of Lusin's theorem found in Appendix D of \cite{dudley2014uniform} (which  applies to an arbitrary metric space) to approximate  bounded measurable functions with bounded continuous functions. The assumption  $\{f^*<\infty\}$ implies $f^*( g)\in C_b(S)$ for all $ g\in C_b(S)$ and so $(Q,P)\mapsto E_Q[ g]-E_P[f^*( g)]$ is continuous.  The supremum over $ g$ is therefore lower semicontinuous.
\end{proof}
One can further restrict the optimization to Lipschitz functions.
\begin{corollary}
Let   $f\in\mathcal{F}_1(a,b)$, $S$ be a metric space, $Q,P\in\mathcal{P}(S)$. If $\{f^*<\infty\}=\mathbb{R}$   then
\begin{align}\label{eq:Df_Lip_b}
D_f(Q\|P)=\sup_{ g\in \Lip_b(S)}\{ E_Q[ g]-E_P[f^*( g)]\}\,,
\end{align} 
where $\Lip_b(S)$ denotes the set of real-valued bounded Lipschitz functions on $S$.

\end{corollary}
\begin{proof}
The result follows from Corollary \ref{cor:Df_LSC}, together with the fact that every $ g\in C_b(S)$ is the pointwise limit of Lipschitz functions, $ g_n$, with $\| g_n\|_\infty\leq \| g\|_\infty$ (see Box 1.5 on page 6 of \cite{santambrogio2015optimal}).
\end{proof}
\begin{remark}\label{remark:Df_nu_shift}
Due to the invariance of the spaces $\mathcal{M}_b(S)$, $C_b(S)$, and $\Lip_b(S)$ under shifting by a constant, one can replace $E_P[f^*( g)]$ in any of \req{eq:Df_var_app}, \req{eq:Df_var_Cb}, or \req{eq:Df_Lip_b} by $\inf_{\nu\in\mathbb{R}}\{\nu+E_P[f^*( g-\nu)]\}$ without changing the left-hand-side.
\end{remark}

\begin{lemma}\label{lemma:Df_compact_sublevel}
Let $f\in\mathcal{F}_1(a,b)$, $S$ be a Polish space, and $P\in \mathcal{P}(S)$.  If $\{f^*<\infty\}=\mathbb{R}$ then  the map $\mathcal{P}(S)\to[0,\infty]$, $Q\mapsto D_f(Q\|P)$ has compact sublevel sets.
\end{lemma}
\begin{proof}
Let $c \in\mathbb{R}$ and consider the sublevel set $L_c =\{Q:D_f(Q\|P)\leq c \}$. If $c <0$ then $L_c =\emptyset$   and the claim is trivial, hence let $c \geq 0$. Corollary \ref{cor:Df_LSC} implies that  the map $\mathcal{P}(S)\times\mathcal{P}(S)\to [0,\infty]$, $(Q,P)\mapsto D_f(Q\|P)$ is LSC, therefore $L_c $ is closed. By Prohorov's theorem, if $L_c $ is tight then $L_c $ is precompact which will complete the proof: $S$ is Polish, hence $P$ is tight. Therefore for every $\delta>0$ there exists a compact set $K$ such that $P(K^c)\leq \delta$. Given $\epsilon>0$ choose $d>0$ large enough that $(c +f^*(0))/d\leq \epsilon/2$ and choose  $\delta>0$ such that $\frac{1}{d}(f^*(d)-f^*(0))\delta\leq\epsilon/2$.  Then for any $Q\in L_c $, letting $ g=d1_{K^c}$ (a bounded measurable function) and using the variational formula \eqref{eq:Df_var_app} we find
\begin{align}
c \geq D_f(Q\|P)\geq E_Q[ g]-E_P[f^*( g)]= d Q(K^c)-f^*(0) -(f^*(d)  -f^*(0))P(K^c)\,.
\end{align}
Hence
\begin{align}
Q(K^c)\leq (f^*(0) +c )/d+d^{-1}(f^*(d)  -f^*(0))P(K^c)\leq \epsilon\,.
\end{align}
Therefore we conclude that $L_c $ is tight. This completes the proof.
\end{proof}

In light of Corollary \ref{cor:Df_LSC} and Lemma \ref{lemma:Df_compact_sublevel}, it is useful to have  simple conditions that   ensure $\{f^*<\infty\}=\mathbb{R}$; see Lemma \ref{lemma:dom=R} above for such a result.

Next we show that  $D_f(Q\|P)$ is strictly convex in $Q$ when $f$ is strictly convex.
\begin{lemma}\label{lemma:strictly_convex}
Let $f\in\mathcal{F}_1(a,b)$ be strictly convex and $P\in\mathcal{P}(\Omega)$. Then $Q\mapsto D_f(Q\|P)$ is strictly convex on $\{Q:D_f(Q\|P)<\infty\}$.
\end{lemma}
\begin{proof}
First note that strict convexity of $f$ on $(a,b)$ implies strict convexity of the convex LSC extension (also denoted by $f$) on  $\{f<\infty\}$.  Fix distinct $Q_0,Q_1\in \{D_f(Q\|P)<\infty\}$ and $t\in(0,1)$, and define  $Q_t=tQ_1+(1-t)Q_0$.   Convexity of     $Q\mapsto D_f(Q\|P)$ (which follows from \req{eq:Df_var_app}) implies $D_f(Q_t\|P)<\infty$.  

Define $F=\{f(dQ_1/dP)<\infty\text{ and }f(dQ_0/dP)<\infty\}$ and $G=\{dQ_1/dP\neq dQ_0/dP\}$. Finiteness of $D_f(Q_i\|P)$ implies $P(F)=1$ and   $Q_0\neq Q_1$ implies $P(F\cap G)>0$.    We can write
\begin{align}
&tD_f(Q_1\|P)+(1-t)D_f(Q_0\|P)-D_f(Q_t\|P)\\
=&E_P[1_Ftf(dQ_1/dP)+1_F(1-t)f(dQ_0/dP)-1_F f(dQ_t/dP)]\,,\notag
\end{align}
where convexity of $f$ implies the integrand is non-negative and strict convexity of $f$ implies the integrand is positive on $F\cap G$. Therefore we can bound it below by integrating only over $F\cap G$, a set of positive measure.  Hence the expectation is positive and we have proven the claim.
\end{proof}

The following lemma is the key step in the proof of the Gibbs variational principle for $f$-divergences in Proposition \ref{prop:Gibbs_M1} below.
\begin{lemma}\label{lemma:Gibbs}
Let $f\in\mathcal{F}_1(a,b)$, $P$ be a probability measure on $(\Omega,\mathcal{M})$, and $ g\in \mathcal{M}_b(\Omega)$. Then
\begin{align}\label{eq:E_f_star_var}
E_P[f^*( g)]=&\sup_{h\in \mathcal{M}_b(\Omega):E_P[f(h)]<\infty}\{E_P[ g h]-E_P[f(h)]\}\\
=&\sup_{h\in \mathcal{M}_{cr}(\Omega,(a,b))}\{E_P[ g h]-E_P[f(h)]\}\,,\notag
\end{align}
where $\mathcal{M}_{cr}(\Omega,(a,b))$ denotes the set of measurable functions on $\Omega$ whose range is contained in a compact subset of $(a,b)$.

\end{lemma}
\begin{proof}
 $f$ is convex and $f(1)=0$, hence $f(x)\geq  f_+^\prime(1)(x-1)$. This implies $E_{P}[f(h)]$ exists in $(-\infty,\infty]$. Therefore the right-hand-sides of \eqref{eq:E_f_star_var} are well defined and the terms inside the suprema are finite for all $h$'s satisfying the indicated conditions. The left-hand-side is well defined in $(-\infty,\infty]$ since $f^*( g)\geq  g$ (see Lemma \ref{eq:f_star_lb}). For $h\in \mathcal{M}_b(\Omega)$ with $E_P[f(h)]<\infty$ we have
\begin{align}
E_P[ g h]-E_P[f(h)]=E_P[ g h-f(h)]\leq E_P[\sup_{x\in\mathbb{R}}\{ g x-f(x)\}]= E_P[f^*( g)]
\end{align}
and hence
\begin{align}\label{eq:EP_fstar_lb}
\sup_{h\in \mathcal{M}_b(\Omega):E_P[f(h)]<\infty}\{E_P[ g h]-E_P[f(h)]\}\leq E_P[f^*( g)]\,.
\end{align}

For the other direction, let $a<a_n<1<b_n<b$ with $a_n\searrow a$, $b_n\nearrow b$. Then
\begin{align}
f^*( g)=\sup_n\sup_{x\in[a_n,b_n]}\{ g x-f(x)\}=\lim_n \sup_{x\in[a_n,b_n]}\{ g x-f(x)\}\,.
\end{align}
 By letting $x=1$ we see that $\sup_{x\in[a_n,b_n]}\{ g x-f(x)\}\geq  g\in L^1(P)$, and therefore Fatou's Lemma implies 
\begin{align}
E_P[f^*( g)]=E_P[\lim_n \sup_{x\in[a_n,b_n]}\{ g x-f(x)\}]\leq\liminf_n E_P[\sup_{x\in[a_n,b_n]}\{ g x-f(x)\}]\,.
\end{align}
Using continuity of $ g x-f(x)$ on $x\in(a,b)$ and finiteness of $\sup_{x\in[a_n,b_n]}\{ g x-f(x)\}$  we see that there exists measurable $h_n:\Omega\to [a_n,b_n]$ such that
\begin{align}
\sup_{x\in[a_n,b_n]}\{ g x-f(x)\}\leq \frac{1}{n}+ g h_n-f(h_n)\,.
\end{align}
 $h_n\in\mathcal{M}_{cr}(\Omega,(a,b))$, hence
\begin{align}
E_P[\sup_{x\in[a_n,b_n]}\{ g x-f(x)\}]\leq \frac{1}{n}+\sup_{h\in\mathcal{M}_{cr}(\Omega,(a,b))}\{E_P[ g h]-E_P[f(h)]\}\,.
\end{align}
Therefore
\begin{align}
E_P[f^*( g)]\leq\liminf_n E_P[\sup_{x\in[a_n,b_n]}\{ g x-f(x)\}]\leq \sup_{h\in\mathcal{M}_{cr}(\Omega,(a,b))}\{E_P[ g h]-E_P[f(h)]\}\,.
\end{align}
We have
$\mathcal{M}_{cr}(\Omega,(a,b))\subset\{h\in\mathcal{M}_b(\Omega):E_P[f(h)]<\infty\}$ and so
\begin{align}
E_P[f^*( g)]\leq \sup_{h\in\mathcal{M}_{cr}(\Omega,(a,b))}\{E_P[ g h]-E_P[f(h)]\}\leq \sup_{h\in\mathcal{M}_b(\Omega):E_P[f(h)]<\infty}\{E_P[ g h]-E_P[f(h)]\}\,.
\end{align}
When combined with \req{eq:EP_fstar_lb}, this completes the proof.
\end{proof}

We can now prove the Gibbs variational formula for $f$-divergences in full generality; note that Corollary \ref{corr:Gibbs_a_pos}, which covers the case where $a\geq 0$, as  proven in \cite{BenTal2007}, but to the best of our knowledge the case \eqref{eq:Df_Gibbs_M1}, which covers $a< 0$, is new:
\begin{proposition}\label{prop:Gibbs_M1}
Let $f\in\mathcal{F}_1(a,b)$, $P\in\mathcal{P}(\Omega)$, and  $ g\in \mathcal{M}_b(\Omega)$. Then
\begin{align}\label{eq:Df_Gibbs_M1}
\sup_{h\in \mathcal{M}_b(\Omega):E_P[h]=1, E_P[f(h)]<\infty}\{E_P[ g h]-E_P[f(h)]\}=\inf_{\nu\in\mathbb{R}}\{\nu+E_P[f^*( g-\nu)]\}\,.
\end{align}
\end{proposition}
\begin{corollary}\label{corr:Gibbs_a_pos}
If $a\geq 0$ then \eqref{eq:Df_Gibbs_M1} can be written as
\begin{align}\label{eq:Df_Gibbs}
\sup_{Q\in\mathcal{P}(\Omega): D_f(Q\|P)<\infty}\{E_Q[ g]-D_f(Q\|P)\}=\inf_{\nu\in\mathbb{R}}\{\nu+E_P[f^*( g-\nu)]\}\,.
\end{align}
\end{corollary}
\begin{remark}
\req{eq:Df_Gibbs_M1} is  an optimization over  signed measures, $d\mu=hdP$, of net `charge' $1$.
\end{remark}
\begin{remark}
The most commonly used case where $a<0$ is the $\chi^2$-divergence, which corresponds to the choice $f(x)=x^2-1$, $a=-\infty$, $b=\infty$.
\end{remark}
\begin{proof}
Define the convex set $X=\{h\in \mathcal{M}_b(\Omega):E_P[f(h)]<\infty\}$ (note that convexity of $f$ implies $E_P[f(h)]>-\infty$ for all $h\in\mathcal{M}_b(\Omega)$), define the convex function $F:X\to \mathbb{R}$ by $F[h]=E_P[f(h)]-E_P[ g h]$, and define the affine function $H:\mathcal{M}_b(\Omega)\to\mathbb{R}$ by $H[h]=1-E_P[h]$.  These satisfy the Slater conditions (see Theorem 8.3.1 and Problem 8.7  in \cite{luenberger1997optimization})   and so we have strong duality:
\begin{align}\label{eq:Gibb_general_f}
&\sup_{h\in \mathcal{M}_b(\Omega):E_P[h]=1, E_P[f(h)]<\infty}\{E_P[ g h]-E_P[f(h)]\}\\
=&\inf_{\nu\in\mathbb{R}}\{\nu+\sup_{h\in\mathcal{M}_b(\Omega):E_P[f(h)]<\infty}\{E_P[( g-\nu)h]-E_P[f(h)]\}\notag\\
=&\inf_{\nu\in\mathbb{R}}\{\nu+E_P[f^*( g-\nu)]\}\,,\notag
\end{align}
where we used Lemma \ref{lemma:Gibbs} to obtain the last line.  If $a\geq 0$ then $E_P[f(h)]=\infty$ unless $h\geq 0$ $P$-a.s. and so the supremum in \eqref{eq:Gibb_general_f} can be restricted to non-negative $h$.  Defining $dQ=hdP$ we can then rewrite \eqref{eq:Gibb_general_f} as the supremum over $Q\in\mathcal{P}(\Omega)$ with $D_f(Q\|P)<\infty$.
\end{proof}

\subsection{Variational Characterization  over Unbounded $ g$'s}
In many cases it is useful to extend the variational formula \eqref{eq:Df_var_app} to unbounded $ g$'s.  In this section we give several such results. First we prove a pair of lemmas that ensure certain expectations are finite.  The first of these can also be found in Lemma 2 of \cite{birrell2020optimizing}.
\begin{lemma}\label{lemma:EP_finite}
Let   $f\in\mathcal{F}_1(a,b)$ and  $Q,P\in\mathcal{P}(\Omega)$ with $D_f(Q\|P)<\infty$.  If $ g\in L^1(Q)$ then $E_P[f^*( g)^-]<\infty$.
\end{lemma}
\begin{remark}
Recall that we use $g^\pm$ to denote the positive and negative parts of a function $g$ (so that $g^\pm\geq0$ and $g=g^+-g^-$).
\end{remark}

\begin{proof}
Fix $ g\in L^1(Q)$.  The result it trivial if $f^*$ is bounded below, so suppose not.  Lemma \ref{lemma:f_star_bounded_or_inc} therefore  implies that  $I=\{y:f^*(y)<\infty\}$ is of the form $I=(-\infty,d)$ or $I=(-\infty,d]$ for some $d\in(-\infty,\infty]$,  $f^*$ is nondecreasing, and there exists $b\in\mathbb{R}$ such that  $f^*\leq 0$ on $(-\infty,b]$ and $f^*\geq 0$ on $(b,\infty)$.   Define $ g_b= g1_{ g\leq b}+ b1_{ g> b}$, so that $ g_b\leq b$  and $ g_b\in L^1(Q)$. We have
\begin{align}
&E_P[f^*( g)^-]=E_P[1_{ g\leq b}f^*( g)^-]=E_P[1_{ g\leq b}f^*( g_b)^-]\\
\leq& E_P[f^*( g_b)^-]=E_P[-f^*( g_b)]\,,\notag
\end{align}
hence
\begin{align}
& E_P[f^*( g_b)]\leq -E_P[f^*( g)^-]\,.
\end{align}

Let $ g_{b,n}=-n1_{ g_b<-n}+ g_b1_{ g_b\geq -n}$. The $ g_{b,n}$ are bounded, therefore we can use \eqref{eq:Df_var_app} to obtain
\begin{align}
E_Q[ g_{b,n}]-E_P[f^*( g_{b,n})]\leq D_f(Q\|P)\,,
\end{align}
where $E_P[f^*( g_{b,n})]\in (-\infty,\infty]$.  This implies
\begin{align}
E_Q[ g_{b,n}]\leq D_f(Q\|P)+E_P[f^*( g_{b,n})]\,.
\end{align}

$ g_{b,n}\to g_b$ pointwise, $| g_{b,n}|\leq | g_b|$, and $ g_b\in L^1(Q)$, so we use the dominated convergence theorem to compute
\begin{align}\label{eq:EQ_phi_b_ub}
E_Q[ g_{b}]\leq \liminf_n( D_f(Q\|P)+E_P[f^*( g_{b,n})])=D_f(Q\|P)+\liminf_n E_P[f^*( g_{b,n})]\,.
\end{align}
(The assumption that $D_f(Q\|P)<\infty$ implies the right-hand-side of \eqref{eq:EQ_phi_b_ub} is well-defined). We have $ g_{b,{n+1}}\leq   g_{b,n}$, hence $f^*(  g_{b,{n+1}})\leq f^*(  g_{b,n})$.  The function $f^*$ is continuous on $(-\infty,b]$ and for $N$ large enough we have $  g_{b,n}\leq b$ for all $n\geq N$, hence
\begin{align}
 0\leq-f^*(  g_{b,n})\nearrow -  f^*(  g_{b})
\end{align}
as $n\to\infty$. Therefore the monotone convergence theorem gives $\lim_n E_P[f^*( g_{b,n})]= E_P[f^*( g_b)]$ and we find
\begin{align}
-\infty<E_Q[ g_{b}]\leq D_f(Q\|P)+E_P[f^*( g_b)]\leq D_f(Q\|P) -E_P[f^*( g)^-]\,.
\end{align}
This proves $E_P[f^*( g)^-]< \infty$, as claimed.
\end{proof}

\begin{lemma}\label{lemma:EQ_phi_plus}
Let   $f\in\mathcal{F}_1(a,b)$,  $P\in\mathcal{P}(\Omega)$, and $ g\in\mathcal{M}(\Omega)$. Suppose    $E_P[f^*( c g-\nu)^+]<\infty$ for some $\nu\in\mathbb{R}$ and $c>0$. Then for all $Q\in\mathcal{P}(\Omega)$ with  $D_f(Q\|P)<\infty$ we have $E_Q[ g^+]<\infty$.
\end{lemma}
\begin{proof}
Fix $d$ for which $f^*(d)$ is finite and define   $ g_n= g1_{ g\in[0,n)}+(d+\nu)/c 1_{ g\not\in[0,n)}\in\mathcal{M}_b(\Omega)$.  Hence $c g_n-\nu \in L^1(Q)$ and the variational formula \eqref{eq:Df_var_app} gives
\begin{align}
D_f(Q\|P)\geq E_Q[c g_n-\nu]-E_P[f^*(c g_n-\nu)]\,,
\end{align}
where  $E_P[f^*(c g_n-\nu)]$ is defined in $(-\infty,\infty]$.  Hence
\begin{align}
 E_Q[c g_n]-D_f(Q\|P)\leq \nu+E_P[f^*(c g_n-\nu)]\,.
\end{align}
We can bound
\begin{align}
f^*(c g_n-\nu)=f^*(c g-\nu)1_{ g\in[0,n)}+f^*(d)1_{ g\not\in[0,n)}\leq f^*(c g-\nu)^++|f^*(d)|\,,
\end{align}
and so
\begin{align}
E_P[f^*(c g_n-\nu)]\leq E_P[f^*(c g-\nu)^+]+|f^*(d)|\,.
\end{align}
Therefore
\begin{align}
 E_Q[c g_n]-D_f(Q\|P)\leq \nu+ E_P[f^*(c g-\nu)^+]+|f^*(d)|<\infty
\end{align}
for all $n$. Taking $n\to\infty$ we obtain
\begin{align}
 \liminf_nE_Q[ g_n]\leq c^{-1}(\nu+ E_P[f^*(c g-\nu)^+]+|f^*(d)|+D_f(Q\|P))<\infty\,.
\end{align}
The functions $ g_n$ are uniformly bounded below, therefore Fatou's Lemma implies
\begin{align}
E_Q[\liminf_n  g_n]\leq \liminf_n E_Q[ g_n]\leq c^{-1}(\nu+ E_P[f^*(c g-\nu)^+]+|f^*(d)|+D_f(Q\|P))<\infty\,.
\end{align}
We have $ g_n\to  g 1_{ g\geq 0}+c^{-1}(d+\nu)1_{ g<0}$ pointwise, hence
\begin{align}
E_Q[ g^+]+c^{-1}(d+\nu)Q( g<0)=E_Q[ g^++c^{-1}(d+\nu)1_{ g<0}]<\infty\,.
\end{align}
This implies $E_Q[ g^+]<\infty$, as claimed.
\end{proof}

We can now prove variational characterizations of $D_f$ that allow for $ g$ to be unbounded. The following is found in Theorem 2 of \cite{birrell2020optimizing}.
\begin{proposition}\label{prop:Df_var_L1}
Let $f\in\mathcal{F}_1(a,b)$ and $P,Q$ be probability measures on $(\Omega,\mathcal{M})$. If $f^*$ is bounded below or $D_f(Q\|P)<\infty$ then
\begin{align}\label{eq:Df_var_L1}
D_f(Q\|P)=&\sup_{ g\in L^1(Q)}\{E_Q[ g]-E_P[f^*( g)]\}\,,
\end{align}
where $E_P[f^*( g)]$ exists in $(-\infty,\infty]$.
\end{proposition}

\begin{proof}
 If $f^*$ is bounded below  then $E_P[f^*( g)^-]<\infty$ for all $ g\in L^1(Q)$.  If $D_f(Q\|P)<\infty$ then Lemma \ref{lemma:EP_finite} also  implies $E_P[f^*( g)^-]<\infty$.  So in either case we find that $E_P[f^*( g)]$ is defined in $(-\infty,\infty]$. In particular, the objective functional in \eqref{eq:Df_var_L1} is well defined. Due to the  variational characterization \eqref{eq:Df_var_app}, to prove  \eqref{eq:Df_var_L1} it suffices to prove $E_Q[ g]-E_P[f^*( g)]\leq D_f(Q\|P)$ for all $ g\in L^1(Q)$.

Fix $ g\in L^1(Q)$.  If $D_f(Q\|P)=\infty$ or $E_P[f^*( g)]=\infty$ then the required bound is trivial, so suppose not. In this case we have  $f^*( g)<\infty$ $P$-a.s., i.e., $ g$ maps into $I\equiv\{f^*<\infty\}$ $P$-a.s.  We are in the case where $D_f(Q\|P)<\infty$ so $Q\ll P$, hence $ g$ maps into $I$ $Q$-a.s. as well. Therefore, by redefining $ g$ on a measure zero set (under both $Q$ and $P$) we can assume that $\range( g)\subset I$. In summary, we have now reduced the problem to showing that $E_Q[ g]-E_P[f^*( g)]\leq D_f(Q\|P)$ in the case where $ g\in L^1(Q)$,  $f^*( g)\in L^1(P)$, $D_f(Q\|P)<\infty$, $\range( g)\subset I$: Fix $y_0\in I$ and define  $ g_n=y_01_{ g<-n}+ g1_{-n\leq g\leq n}+y_01_{ g>n}\in\mathcal{M}_b(\Omega)$.   \req{eq:Df_var_app} implies
\begin{align}
D_f(Q\|P)\geq  E_Q[ g_n]-E_P[f^*( g_n)]\,.
\end{align}

$ g_n\to g$ pointwise, and $| g_n|\leq | g|+|y_0|\in L^1(Q)$, hence the dominated convergence theorem gives $E_Q[ g_n]\to E_Q[ g]$.  $\range( g_n),\range( g)\subset I$ and $f^*$ is continuous on $I$, hence $f^*( g_n)\to f^*( g)$ pointwise.  We have
\begin{align}
|f^*( g_n)|=&|f^*( g_n)|1_{ g<-n}+|f^*( g_n)|1_{-n\leq g\leq n}+|f^*( g_n)|1_{ g>n}\\
\leq& |f^*(y_0)|+|f^*( g)|\in L^1(P)\,.\notag
\end{align}
Therefore the dominated convergence theorem implies $E_P[f^*( g_n)]\to E_P[f^*( g)]$.  Hence
\begin{align}
D_f(Q\|P)\geq  \lim_{n\to\infty}(E_Q[ g_n]-E_P[f^*( g_n)])=E_Q[ g]- E_P[f^*( g)]\,.
\end{align}
This completes the proof.
\end{proof}

In some cases it is convenient to define conventions regarding infinities that result in a variational formula involving the supremum over all measurable $ g$:
\begin{theorem}\label{thm:Df_var_unbounded}
Let $f\in\mathcal{F}_1(a,b)$, $Q,P\in\mathcal{P}(\Omega)$, and $\mathcal{M}_b(\Omega)\subset\Gamma\subset \mathcal{M}(\Omega)$. Then
\begin{align}\label{eq:Df_unbounded}
D_f(Q\|P)=&\sup_{ g\in \Gamma}\{E_Q[ g]-E_P[f^*( g)]\}\,,
\end{align}
where we define $\infty-\infty\equiv-\infty$, $-\infty+\infty\equiv-\infty$.
\end{theorem}
\begin{remark}
Our convention regarding infinities ensures that $\int  g d\eta\equiv\int g^+d\eta-\int g^-d\eta$ is defined in $\overline{\mathbb{R}}$ for all $\eta\in\mathcal{P}(\Omega)$, $ g\in\mathcal{M}(\Omega)$.
\end{remark}
\begin{remark}
The above  conventions regarding infinities can be viewed as a convenient but rigorous shorthand for restricting the optimization in \eqref{eq:Df_unbounded} to those $ g\in\Gamma$ for which such infinitities do not occur.  However, there is more content to Theorem \ref{thm:Df_var_unbounded}  than this simple convention. For instance, the equality \req{eq:Df_unbounded} implies that if $D_f(Q\|P)<\infty$ and $ g\in\Gamma$ with $E_Q[ g]=\infty$ then it must also be the case that $E_P[f^*( g)]=\infty$.
\end{remark}
\begin{proof}
From \eqref{eq:Df_var_app} we have
\begin{align}\label{eq:Df_M_ub}
D_f(Q\|P)=&\sup_{ g\in\mathcal{M}_b(\Omega)}\{E_Q[ g]-E_P[f^*( g)]\}\\
\leq&\sup_{ g\in\Gamma}\{E_Q[ g]-E_P[f^*( g)]\}\notag\\
\leq& \sup_{ g\in\mathcal{M}(\Omega)}\{E_Q[ g]-E_P[f^*( g)]\}\,.\notag
\end{align}
If $D_f(Q\|P)=\infty$ then the above inequalities are all equalities and we are done. In the case where $D_f(Q\|P)<\infty$,   Proposition \ref{prop:Df_var_L1} implies
\begin{align}\label{eq:Df_lb_unbounded}
    D_f(Q\|P)\geq E_Q[ g]-E_P[f^*( g)]
\end{align}
for all $ g\in L^1(Q)$. In light of \req{eq:Df_M_ub}, if we can show \eqref{eq:Df_lb_unbounded} holds for  all $ g\in\mathcal{M}(\Omega)$ then we are done. If $ g^-\not\in L^1(Q)$  then \eqref{eq:Df_lb_unbounded} is a trivial consequence of our conventions regarding infinities. This leaves only the case where  $ g\in\mathcal{M}(\Omega)$ with $ g^+\not\in L^1(Q)$ and $ g^-\in L^1(Q)$.

First we show that $E_P[f^*( g)^-]<\infty$ in this case:  If $f^*$ is bounded below this is trivial so suppose not. Therefore Lemma \ref{lemma:f_star_bounded_or_inc} implies $f^*$ is nondecreasing and there exists $b\leq 0$ such that  $f^*|_{(-\infty,b]}\leq 0$ and $f^*|_{(b,\infty)}\geq 0$.  Define $ g_n= g1_{ g\leq n}+b1_{ g>n}$ for $n\in\mathbb{Z}^+$. $ g_n\in L^1(Q)$ and  $f^*( g)^-1_{ g\geq n}=0$, hence $E_P[f^*( g_n)]\in(-\infty,\infty]$ and
\begin{align}
\infty>E_P[f^*( g_n)^-]=E_P[f^*( g)^-1_{ g\leq n}+f^*(b)^-1_{ g> n}]=E_P[f^*( g)^-]+f^*(b)^-E_P[1_{ g> n}]\,.
\end{align}
Therefore $E_P[f^*( g)^-]<\infty$ as claimed.   Lemma \ref{lemma:EQ_phi_plus} together with $D_f(Q\|P)<\infty$ and $ g^+\not\in L^1(Q)$ implies $E_P[f^*( g)^+]=\infty$. Therefore we conclude that $E_Q[ g]-E_P[f^*( g)]=\infty-\infty=-\infty<D_f(Q\|P)$. This completes the proof.
\end{proof}
The following alternative form is also  useful:
\begin{corollary}\label{cor:D_f_var_unbounded}
Let $f\in\mathcal{F}_1(a,b)$ and $Q,P\in\mathcal{P}(\Omega)$.  Then
\begin{align}
D_f(Q\|P)=&\sup_{ g\in \mathcal{M}(\Omega)}\{E_Q[ g]-\inf_{\nu\in\mathbb{R}}\{\nu+E_P[f^*( g-\nu)]\}\}\,,
\end{align}
where we define $\infty-\infty\equiv-\infty$, $-\infty+\infty\equiv-\infty$.
\end{corollary}
\begin{proof}
The bound 
\begin{align}
\sup_{ g\in \mathcal{M}(\Omega)}\{E_Q[ g]-\inf_{\nu\in\mathbb{R}}\{\nu+E_P[f^*( g-\nu)]\}\}\geq D_f(Q\|P)
\end{align}
is an obvious consequence of Theorem \ref{thm:Df_var_unbounded}.  To conclude the reverse inequality, we need to show that 
\begin{align}\label{eq:inf_nu_bound}
E_Q[ g]-\inf_{\nu\in\mathbb{R}}\{\nu+E_P[f^*( g-\nu)]\}\leq D_f(Q\|P)
\end{align}
for all $ g\in\mathcal{M}(\Omega)$. If $E_Q[ g]=\infty$ and $\inf_{\nu\in\mathbb{R}}\{\nu+E_P[f^*( g-\nu)]\}\neq \infty$ then $E_Q[ g^-]<\infty$ and there exists $\nu_0$ with $E_P[f^*( g-\nu_0)]<\infty$.  Hence Theorem \ref{thm:Df_var_unbounded} implies
\begin{align}
    D_f(Q\|P)\geq E_Q[ g-\nu_0]-E_P[f^*( g-\nu_0)]=\infty\,.
\end{align}
Therefore the claim holds in this case.  Otherwise, we are in the case where $\inf_{\nu\in\mathbb{R}}\{\nu+E_P[f^*( g-\nu)]\}=\infty$ or $E_Q[ g]=-\infty$ or  $E_Q[ g]\in\mathbb{R}$. The first two of these immediately imply \eqref{eq:inf_nu_bound}, due to our conventions regarding infinities.  Finally, if  $E_Q[ g]\in\mathbb{R}$ then
\begin{align}
E_Q[ g]-\inf_{\nu\in\mathbb{R}}\{\nu+E_P[f^*( g-\nu)]\}=&\sup_{\nu\in\mathbb{R}}\{E_Q[ g-\nu]-E_P[f^*( g-\nu)]\}\\
\leq&\sup_{g\in\mathcal{M}(\Omega)}\{E_Q[g]-E_P[f^*(g)]\}=D_f(Q\|P)\,.\notag
\end{align}
This completes the proof.
\end{proof}

\section{Proofs of Properties of  $(f,\Gamma)$-Divergences}\label{app:proofs}
In this appendix we prove the $(f,\Gamma)$-divergence properties from Section \ref{sec:gen_f_div}; some results will be proven in greater generality than were stated earlier. Our method of proof will require us to work with finite signed measures (at least during the intermediate steps), and not just with probability measures.  For that reason we provide a more general definition of the $(f,\Gamma)$-divergences here:
\begin{definition}\label{def:Df_Gamma_app}
Let $f\in\mathcal{F}_1(a,b)$ and  $\Gamma\subset \mathcal{M}_b(\Omega)$ be nonempty. For $P\in\mathcal{P}(\Omega)$ and $\mu\in M(\Omega)$ we define the $(f,\Gamma)$-divergence by
\begin{align}\label{eq:gen_f_def_app}
D_f^\Gamma(\mu\|P)=\sup_{ g\in\Gamma}\left\{\int  g d\mu-\inf_{\nu\in\mathbb{R}}\{\nu+E_P[f^*( g-\nu)]\}\right\}
\end{align}
and for $\mu,\kappa\in M(\Omega)$ we define the $\Gamma$-IPM by
\begin{align}\label{eq:gen_wasserstein_app}
W^\Gamma(\mu,\kappa)=\sup_{ g\in\Gamma}\left\{\int g d\mu-\int g d\kappa\right\}\,.
\end{align}
\end{definition}
We start with the proof of the dual variational formula from Theorem \ref{thm:gen_fdiv_dual_var}, which we restate below. In addition, we treat the case $a< 0$, which requires the more general definition \eqref{eq:gen_f_def_app}.
\begin{theorem}\label{thm:gen_fdiv_dual_var_app}
Let $f\in\mathcal{F}_1(a,b)$, $P\in\mathcal{P}(\Omega)$, and $\Gamma\subset \mathcal{M}_b(\Omega)$ be nonempty.  For $ g\in\Gamma$ we have
\begin{align}\label{eq:Gibbs_VF_Phi_M1_app}
(D_f^\Gamma)^*( g;P)\equiv\sup_{\mu\in M_1(\Omega)}\left\{\int g d\mu-D_f^\Gamma(\mu\|P)\right\}=\inf_{\nu\in\mathbb{R}}\{\nu+E_P[f^*( g-\nu)]\}\,,
\end{align}
where $M_1(\Omega)\equiv\{\mu\in M(\Omega):\mu(\Omega)=1\}$.  If $a\geq 0$ then 
\begin{align}\label{eq:Gibbs_VF_Phi_P_app}
(D_f^\Gamma)^*( g;P)\equiv \sup_{Q\in \mathcal{P}(\Omega)}\{E_Q[ g] -D_f^\Gamma(Q\|P)\}=\inf_{\nu\in\mathbb{R}}\{\nu+E_P[f^*( g-\nu)]\}\,.
\end{align}
\end{theorem}
\begin{proof}
 Let $ g\in\Gamma\subset\mathcal{M}_b(\Omega)$.  Using the definition of $D_f^\Gamma$ along with Proposition \ref{prop:Gibbs_M1} we have 
\begin{align}
&\sup_{\mu\in M_1(\Omega)}\left\{\int  g d\mu-D_f^\Gamma(\mu\|P)\right\}\\
=&\sup_{\mu\in M_1(\Omega)}\left\{\int g d\mu-\sup_{\widetilde g\in\Gamma}\left\{\int \widetilde g d\mu-\inf_{\nu\in\mathbb{R}}\left\{\nu+E_P[f^*(\widetilde g-\nu)]\right\}\right\}\right\}\notag\\
\leq&\sup_{\mu\in M_1(\Omega)}\left\{\int g d\mu-\left(\int  g d\mu-\inf_{\nu\in\mathbb{R}}\left\{\nu+E_P[f^*( g-\nu)]\right\}\right)\right\}\notag\\
=&\inf_{\nu\in\mathbb{R}}\{\nu+E_P[f^*( g-\nu)]\}\notag\\
=&\sup_{h\in \mathcal{M}_b(\Omega):E_P[h]=1,E_P[f(h)]<\infty}\left\{\int  g dP_h-E_P[f(h)]\right\}\,,\notag
\end{align}
where $dP_h=hdP$.  Noting that $P_h\in M_1(\Omega)$ and using $(f^*)^*=f$ we obtain the bound
\begin{align}
D_f^\Gamma(P_h\|P)=&\sup_{ g\in\Gamma}\sup_{\nu\in\mathbb{R}}E_P[( g-\nu)h-f^*( g-\nu)]\leq E_P[f(h)]\,.
\end{align}
Hence
\begin{align}
&\sup_{h\in \mathcal{M}_b(\Omega):E_P[h]=1,E_P[f(h)]<\infty}\left\{\int  g dP_h-E_P[f(h)]\right\}\\
\leq&\sup_{h\in \mathcal{M}_b(\Omega):E_P[h]=1,E_P[f(h)]<\infty}\left\{\int  g dP_h-D_f^\Gamma(P_h\|P)\right\}\notag\\
\leq &\sup_{\mu\in M_1(\Omega)}\left\{\int  g d\mu- D_f^\Gamma(\mu\|P)\right\}\,.\notag
\end{align}
Combining these we arrive at \req{eq:Gibbs_VF_Phi_M1_app}. If $a\geq 0$ then we can use \req{eq:Gibbs_VF_Phi_M1_app}, the bound $D_f^\Gamma(Q\|P)\leq D_f(Q\|P)$, and then \req{eq:Df_Gibbs} to obtain
\begin{align}
\inf_{\nu\in\mathbb{R}}\{\nu+E_P[f^*( g-\nu)]\}\geq&\sup_{Q\in\mathcal{P}(\Omega)}\{E_Q[ g]-D_f^\Gamma(Q\|P)\}\\
\geq & \sup_{Q\in\mathcal{P}(\Omega)}\{E_Q[ g]-D_f(Q\|P)\}\notag\\
=&\inf_{\nu\in\mathbb{R}}\{\nu+E_P[f^*( g-\nu)]\}\,,\notag
\end{align}
which proves \eqref{eq:Gibbs_VF_Phi_P_app}.
\end{proof}

{ 
Next we prove that the $(f,\Gamma)$-divergences are bounded above by the classical $f$-divergence and $\Gamma$-IPM and also derive the convexity and divergence   properties from Theorem \ref{thm:general_ub}.
\begin{theorem}\label{thm:general_ub_app}
Let $f\in\mathcal{F}_1(a,b)$, $\Gamma\subset\mathcal{M}_b(\Omega)$ be nonempty, and $Q,P\in\mathcal{P}(\Omega)$. 
\begin{enumerate}
    \item 
\begin{align}\label{eq:inf_conv_ineq_app}
    D_f^\Gamma(Q\|P)\leq \inf_{\eta\in\mathcal{P}(\Omega)}\{D_f(\eta\|P)+W^\Gamma(Q,\eta)\}\,.
\end{align}
In particular, $D_f^\Gamma(Q\|P)\leq \min\{D_f(Q\|P),W^\Gamma(Q,P)\}$.
\item The map $(Q,P)\in\mathcal{P}(S)\times\mathcal{P}(S)\mapsto D_f^\Gamma(Q\|P)$ is convex. 
\item If there exists $c_0\in \Gamma\cap\mathbb{R}$ then $D_f^\Gamma(Q\|P)\geq 0$.
\item Suppose $f$ and $\Gamma$  satisfy the following:
\begin{enumerate}
\item There exist a nonempty set $\Psi\subset\Gamma$ with the following properties:
\begin{enumerate}
\item $\Psi$ is $\mathcal{P}(\Omega)$-determining.
\item For all $\psi\in\Psi$  there exists $c_0\in\mathbb{R}$, $\epsilon_0>0$ such that $c_0+\epsilon \psi\in\Gamma$ for all $|\epsilon|<\epsilon_0$.
\end{enumerate}
\item $f$ is strictly convex on a neighborhood of $1$.
\item $f^*$ is finite and $C^1$ on a neighborhood of $\nu_0\equiv f_+^\prime(1)$.
\end{enumerate}
Then:
\begin{enumerate}[label=(\roman*)]
\item $D_f^{\Gamma}$ has the divergence property.
\item $W^{\Gamma}$ has the divergence property.
\end{enumerate}
\end{enumerate}
\end{theorem}
}
\begin{proof}
\begin{enumerate}
    \item { 
Let $f_0$ be the restriction of $f$ to the interval $(\max\{a,0\},b)$, so that $f_0\in \mathcal{F}_{1}(\max\{a,0\},b)$.  Note that $f$ and $f_0$ agree on $[0,\infty)$ and so $D_f=D_{f_0}$ (see \eqref{eq:Df_def}). The definition of the Legendre transform implies  $f_0^*\leq f^*$ and so $\Lambda_{f_0}^P\leq \Lambda_f^P$ (see \eqref{eq:Lambda_f_def}). We can now compute
\begin{align}
    \inf_{\eta\in\mathcal{P}(\Omega)}\{D_f(\eta\|P)+W^\Gamma(Q,\eta)\}=&
    \inf_{\eta\in\mathcal{P}(\Omega)}\{D_{f_0}(\eta\|P)+W^\Gamma(Q,\eta)\}\\        =&\inf_{\eta\in\mathcal{P}(\Omega)}\{\sup_{g\in\Gamma}\{D_{f_0}(\eta\|P)+E_Q[g]-E_\eta[g]\}\}\notag\\
        \geq &\sup_{g\in\Gamma}\{\inf_{\eta\in\mathcal{P}(\Omega)}\{D_{f_0}(\eta\|P)+E_Q[g]-E_\eta[g]\}\}\notag\\
             =&\sup_{g\in\Gamma}\{E_Q[g]-\Lambda_{f_0}^P[g]\}\notag\\
          \geq&\sup_{g\in\Gamma}\{E_Q[g]-\Lambda_{f}^P[g]\}=D_f^\Gamma(Q\|P)\notag\,,
\end{align}
where we used \req{eq:Df_Gibbs_var_formula} to  obtain  the second-to-last  line. 
\begin{remark}
We emphasize that $\Lambda_f^P$ naturally appears when working with the infimal convolution of an $f$-divergence and a $\Gamma$-IPM, due to the identity \eqref{eq:Df_Gibbs_var_formula}.
\end{remark}}
\item {  Convexity of $D_f^\Gamma$ on $\mathcal{P}(S)\times \mathcal{P}(S)$ follows from \eqref{eq:Df_Gamma_def2}, which shows that $D_f^\Gamma(Q\|P)$ is the supremum of functions that are  affine in $(Q,P)$.
}
\item  { 
Lemma \ref{lemma:nu0} implies $f^*(\nu_0)=\nu_0$, where $\nu_0\equiv f_+^\prime(1)$.  By assumption, $c_0\in \Gamma\cap\mathbb{R}$, hence bounding \eqref{eq:Df_Gamma_def2} below by its value at $g=c_0$, $\nu=c_0-\nu_0$ we find
\begin{align}
    D_f^\Gamma(Q\|P)\geq& E_Q[c_0-(c_0-\nu_0)]-E_P[f^*(c_0-(c_0-\nu_0))]\\
    =&\nu_0-f^*(\nu_0)=0\,.\notag
\end{align}
\begin{remark}
 Note the importance of the infimum over $\nu$ in $\Lambda_f^P$, which allowed us to obtain a lower bound at an appropriate value of $\nu$. This same technique will be used several times below and highlights the importance of employing $\Lambda_f^P$ in the definition \eqref{eq:gen_f_def} of $D_f^\Gamma$. See also Remark \ref{remark:Lambda_f}.
\end{remark}
}
\item Now suppose $f$ and $\Gamma$ satisfy 2.a - 2.c Lemma \ref{lemma:nu0} implies that
\begin{align}\label{eq:nu0_identities}
f^*(\nu_0)=\nu_0\,,\,\,\,\, (f^*)^\prime(\nu_0)=1\,.
\end{align}
Assumption 2.a.ii implies there exists $c_0\in\Gamma\cap\mathbb{R}$ hence Part 3 of this theorem implies $D_f^\Gamma(Q\|P)\geq 0$. From \req{eq:inf_conv_ineq_app} we have $D_f^{\Gamma}(Q\|P)\leq D_f(Q\|P)$.  Combining this with the non-negativity of $D_f^{\Gamma}$ and  the fact that $D_f(P\|P)=0$   we see that if $Q=P$ then $D_f^{\Gamma}(Q\|P)=0$. 
 
Next assume  $D_f^{\Gamma}(Q\|P)=0$: From assumption 2.a.ii, given $\psi\in\Psi$ there exists $c_0\in\mathbb{R}$, $\epsilon_0>0$ such that $ g_\epsilon\equiv c_0+\epsilon\psi\in\Gamma\cap\mathcal{M}_b(\Omega)$ for all $|\epsilon|<\epsilon_0$. Therefore
\begin{align}
0=&D_f^{\Gamma}(Q\|P)\geq  E_Q[ g_\epsilon]-\inf_{\nu\in\mathbb{R}}\{\nu+E_P[f^*( g_\epsilon-\nu)]\}\\
\geq&E_Q[ g_\epsilon]-(c_0-\nu_0+E_P[f^*(\nu_0+\epsilon\psi)])\notag\\
=&(\nu_0+\epsilon E_Q[\psi])-E_P[f^*(\nu_0+\epsilon\psi)])\equiv h(\epsilon)\,.\notag
\end{align}
 By assumption 2.c, there exists $\delta>0$ such that $f^*$ is finite and $C^1$ on the  ball of radius $\delta$ centered at $\nu_0$, denoted by $B_\delta(\nu_0)$.   $\psi$ is bounded, therefore we can find $C>0$ with $|\psi|\leq C$.  For $|\epsilon|<\min\{\epsilon_0,\delta/(2C)\}$ we have $\range(\nu_0+\epsilon\psi)\subset B_{\delta/2}(\nu_0)$. On $B_{\delta/2}(\nu_0)$,  $f^*$ is $C^1$ and $f^*$, $(f^*)^\prime$  are both bounded. Hence the dominated convergence theorem  implies  $h$ is $C^1$ on $|\epsilon|<\min\{\epsilon_0,\delta/C\}$ and $h^\prime(\epsilon)=E_Q[\psi]-E_P[(f^*)^\prime(\nu_0+\epsilon\psi)\psi]$.  Evaluating this at $\epsilon=0$ and using \req{eq:nu0_identities} we find
\begin{align}
h^\prime(0)=E_Q[\psi]-E_P[(f^*)^\prime(\nu_0)\psi]=E_Q[\psi]-E_P[\psi]\,.
\end{align}
Again using \eqref{eq:nu0_identities} we can also compute  $h(0)= \nu_0-E_P[f^*(\nu_0)]=0$. Combining these facts with the bound $h(\epsilon)\leq 0$ we can conclude that $h^\prime(0)=0$ and hence $E_Q[\psi] =E_P[\psi]$ for all $\psi\in\Psi$. By assumption 2.a.i, $\Psi$ is $\mathcal{P}(\Omega)$-determining and so $Q=P$. This completes the proof of the divergence property for $D_f^\Gamma$. 

{ 
The divergence property for $W^\Gamma$ then follows from the divergence property for $D^\Gamma_f$ together with the bound $D^\Gamma_f(Q\|P)\leq W^\Gamma(Q,P)$ and the definition \eqref{eq:gen_wasserstein}.
}
\end{enumerate}
\end{proof}

Next we prove the infimal convolution formula, \eqref{eq:inf_conv}, as well as the other properties from  Theorem \ref{thm:f_div_inf_convolution}, again in somewhat greater generality.
\begin{theorem}\label{thm:f_div_inf_convolution_app}
 Suppose   $f$ and $\Gamma$ are admissible. For $P\in\mathcal{P}(S)$, $\mu\in M(S)$ let $D^\Gamma_f(\mu\|P)$ be defined by \eqref{eq:gen_f_def_app} 
and for $\mu,\kappa\in M(S)$ let $W^\Gamma(\mu,\kappa)$ be defined as in \eqref{eq:gen_wasserstein_app}. These have the following properties:
\begin{enumerate}
\item Infimal Convolution Formula: 
\begin{align}\label{eq:inf_conv_app}
D_f^\Gamma(\mu\|P)=\inf_{\eta\in \mathcal{P}(S)}\{D_f(\eta\|P)+W^\Gamma(\mu,\eta)\}\,.
\end{align}
In particular, $D_f^\Gamma(\mu\|P)\leq W^\Gamma(\mu,P)$ and if $Q\in\mathcal{P}(S)$ then $D_f^\Gamma(Q\|P)\leq D_f(Q\|P)$.
\item If $D_f^\Gamma(\mu\|P)<\infty$ then there exists $\eta_*\in\mathcal{P}(S)$ such that
\begin{align}\label{eq:inf_conv_existence_app}
D_f^\Gamma(\mu\|P)=D_f(\eta_*\|P)+W^\Gamma(\mu,\eta_*)\,.
\end{align}
If $f$ is strictly convex then there is a unique such $\eta_*$.
\item Divergence Property for $W^\Gamma$: $W^\Gamma\geq 0$ and $W^\Gamma(\mu,\mu)=0$ for all $\mu\in M(S)$. If $\Gamma$ is strictly admissible  then for all $Q,P\in\mathcal{P}(S)$ we have $W^\Gamma(	Q,P)=0$ if and only if $Q=P$.
\item Divergence Property for $D^\Gamma_f$: For $Q,P\in\mathcal{P}(S)$ we have $D_f^\Gamma(Q\|P)\geq 0$ and $D_f^{\Gamma}(P\|P)=0$.  If $f$ and $\Gamma$ are both strictly admissible then $D_f^\Gamma(Q\|P)=0$ if and only if $Q=P$.
\end{enumerate}

\end{theorem}
\begin{proof}
\begin{enumerate}
\item Define $H_1,H_2:C_b(S)\to(-\infty,\infty]$ by
\begin{align}\label{eq:H1_def}
H_1( g)=\inf_{\nu\in\mathbb{R}}\{\nu+E_P[f^*( g-\nu)]\}
\end{align}
 and $H_2( g)=\infty 1_{\Gamma^c}( g)$ (note that $H_1>-\infty$ follows from the bound $f^*(y)\geq y$; see Lemma \ref{eq:f_star_lb}). We first show that $H_1$ and $H_2$ are convex and LSC. To see that $H_1$ is convex, note that convexity of $f^*$ implies that the map $( g,\nu)\mapsto \nu+E_P[f^*( g-\nu)]$ is convex on $C_b(S)\times\mathbb{R}$. Therefore, taking the infimum over $\nu$ results in a convex function of $g$; see Theorem 2.2.6 in \cite{bot2009duality}.   To show lower semicontinuity of $H_1$, first recall the variational formula \eqref{eq:Df_Gibbs_M1}
\begin{align}\label{eq:Lambda_f_var2}
\inf_{\nu\in\mathbb{R}}\{\nu+E_P[f^*( g-\nu)]\}=&\sup_{h\in \mathcal{M}_b(S):E_P[h]=1, E_P[f(h)]<\infty}\{E_P[ g h]-E_P[f(h)]\}\,.
\end{align}
 We can write $E_P[ gh]=\int g dP_h$ where $dP_h\equiv h dP\in M(S)$.  Recalling that $C_b(S)^*=\{\tau_\eta:\eta\in M(S)\}$, $\tau_\eta( g)\equiv \int g d\eta$ we see that $g\mapsto E_P[ g h]$ is continuous on $C_b(S)$. Therefore \req{eq:Lambda_f_var2} expresses $H_1$ as the supremum of a family of continuous functions, thus proving $H_1$  is LSC.  $H_2$ is LSC and convex since $\Gamma$ is closed and convex. We can write $D_f^\Gamma(\mu\|P)$ as the infinite-dimensional convex conjugate of $H_1+H_2$:
\begin{align}
D_f^\Gamma(\mu\|P)=&\sup_{ g\in C_b(S)}\{\tau_\mu( g)-(H_1( g)+H_2( g))\}=(H_1+H_2)^*(\tau_\mu)\,.
\end{align}
   We will now use the theory of infimal convolutions to compute the convex conjugate of $H_1+H_2$.  Under appropriate assumptions, this theory allows one to show 
\begin{align}\label{eq:H_star_goal}
    (H_1+H_2)^*(\tau)=\inf\{H_1^*(\tau_1)+H_2^*(\tau_2):\tau_1+\tau_2=\tau\}\equiv (H_1^*\Box H_2^*)(\tau)\,, \,\,\,\tau\in C_b(S)^*\,,
\end{align}
where $H_1^*\Box H_2^*$ is called the infimal convolution; see, e.g., Chapter 2 in \cite{bot2009duality} for further information on infimal convolutions. Specifically, using    Theorem 2.3.10 in \cite{bot2009duality}, if $\dom H_1\cap \dom H_2\neq \emptyset$ and $H_1^*\Box H_2^*$ is LSC in the weak-* topology on $C_b(S)^*$  then $(H_1+H_2)^*=H_1^*\Box H_2^*$. To show the first condition, note that  $f^*$ is not identically equal to $\infty$ and $0\in\Gamma$; using these facts it is straightforward to show that $0\in \dom H_1\cap \dom H_2$. Therefore if we can prove lower semicontinuity of $H_1^*\Box H_2^*$ then we can conclude \eqref{eq:H_star_goal}. To accomplish this, first   rewrite
\begin{align}\label{eq:Box_H}
H_1^*\Box H_2^*(\tau_\mu)=\inf_{\eta\in M(S)}\{H_1^*(\tau_\eta)+W^\Gamma(\mu,\eta)\}\,.
\end{align} 
Next we show that the infimum in \eqref{eq:Box_H} can be restricted to $\mathcal{P}(S)$. We do this in two steps:
\begin{enumerate}
    \item  $H_1^*(\tau_\eta)=\infty$ when $\eta$ is not positive: To show this, first note that
\begin{align}
H_1^*(\tau_\eta)\geq \sup_{ g\in C_b(S)}\left\{\int  g d\eta-E_P[f^*( g)]\right\}\,.
\end{align}
If there exists a measurable set $F\subset S$ with $\eta(F)<0$ then by the extension of Lusin's theorem found in Appendix D of \cite{dudley2014uniform}, for any $\epsilon>0$ there exists a closed set $E_\epsilon\subset S$ such that $|\eta|(E_\epsilon^c)<\epsilon$ and a $g_\epsilon\in C_b(S)$ such that $0\leq g_\epsilon\leq 1$ and $g_\epsilon=1_F$ on $E_\epsilon$.  For $n\in\mathbb{Z}^+$ define $g_{n,\epsilon}=-ng_\epsilon\in C_b(S)$. The assumption that $f$ is admissible implies $\lim_{y\to-\infty}f^*(y)<\infty$. Therefore $f^*(g_{n,\epsilon})$  is bounded above independent of $n,\epsilon$, and so there exists $D\in\mathbb{R}$ with
\begin{align}
&H_1^*(\tau_\eta)\geq -n\int g_\epsilon d\eta -D=n|\eta(F)|+n\int 1_{E_\epsilon^c}1_Fd\eta-n\int _{E_\epsilon^c}g_\epsilon d\eta-D\\
\geq &n|\eta(F)|-2\epsilon n-D\,.\notag
\end{align}
Letting $\epsilon<|\eta(F)|/2$ and sending $n\to\infty$ proves the claim.

\item  $H_1^*(\tau_\eta)=\infty$ if $\eta(S)\neq 1$:  For $c\in\mathbb{R}$ we can use the fact that $(f^*)^*=f$ to compute
\begin{align}\label{eq:gamma_mass_1}
H_1^*(\tau_\eta)\geq& c\eta(S)-\inf_{\nu\in\mathbb{R}}\{\nu+f^*(c-\nu)\}=c(\eta(S)-1)+\sup_{\nu\in\mathbb{R}}\{c-\nu-f^*(c-\nu)\}\\
=& c(\eta(S)-1)+f(1)=c(\eta(S)-1)\,.\notag
\end{align}
Taking $c\to \pm\infty$ proves the claim.
\begin{remark}\label{remark:Lambda_f}
In light of the variational formula \eqref{eq:Df_var_formula}, one might be motivated to define $D_f^\Gamma$ using  $E_P[f^*( g)]$ in place of $H_1( g)$. However, the property proven in \req{eq:gamma_mass_1} would fail in that case and we would be unable to proceed with our method of proof.  If $\Gamma$ is closed under the shift transformations $ g\mapsto g-\nu$, $\nu\in\mathbb{R}$, then the choice of $H_1( g)$ versus $E_P[f^*( g)]$  does not impact the value of $D^\Gamma_f(\mu\|P)$ when $\mu=Q\in\mathcal{P}(S)$, but it can change the value if $\mu\in M(S)\setminus\mathcal{P}(S)$ and the choice can also impact the performance of numerical computations (see  \cite{Ruderman,birrell2020optimizing} for discussion of this issue in the context of classical $f$-divergences).  
\end{remark}
\end{enumerate}
Having proven the above two properties we can now conclude that  $H_1^*(\tau_\eta)=\infty$ if $\eta\not\in\mathcal{P}(S)$, hence 
\begin{align}\label{eq:boxH_P(S)}
H_1^*\Box H_2^*(\tau_\mu)=\inf_{\eta\in\mathcal{P}(S)}\{D_f(\eta\|P)+W^\Gamma(\mu,\eta)\}\,,
\end{align}
where we used Corollary \ref{cor:Df_LSC} and Remark \ref{remark:Df_nu_shift} to evaluate $H_1^*(\tau_\eta)$. To prove lower semicontinuity of $H_1^*\Box H_2^*$, let $a\in\mathbb{R}$ and take a net $\{\tau_{\mu_\alpha}\}_{\alpha\in A}$ in $\{H_1^*\Box H_2^*\leq a\}$ with $\tau_{\mu_\alpha}\to\tau_\mu$.    For any $\epsilon>0$, \req{eq:boxH_P(S)} implies there exists $\eta_{\alpha,\epsilon}\in\mathcal{P}(S)$ with
\begin{align}
\epsilon+a>D_f(\eta_{\alpha,\epsilon}\|P)+W^\Gamma(\mu_\alpha,\eta_{\alpha,\epsilon})\geq D_f(\eta_{\alpha,\epsilon}\|P)\,,
\end{align}
where in the second inequality we used the fact that $W^\Gamma\geq 0$ (to see this, bound it below by taking $ g=0$ in \eqref{eq:gen_wasserstein_app}).
$D_f(\cdot\|P)$ is lower semicontinuous (see Corollary \ref{cor:Df_LSC}) and has compact sublevel sets in the Prohorov-metric topology (see  Lemma \ref{lemma:Df_compact_sublevel}), so there exists a subnet $\eta_{\alpha_\beta,\epsilon}$, $\beta\in B$ (where $B$ is some directed set) with $\eta_{\alpha_\beta,\epsilon}\to \eta_\epsilon$ in the Prohorov metric (i.e., weakly) and
\begin{align}
D_f(\eta_\epsilon\|P)\leq \liminf_\beta D_f(\eta_{\alpha_\beta,\epsilon}\|P)\equiv \sup_{\tilde \beta}\inf_{\beta\geq\tilde\beta}D_f(\eta_{\alpha_\beta,\epsilon}\|P)\,.
\end{align}

The weak-* topology on $C_b(S)^*$ is generated by $\{\pi_ g: g\in C_b(S)\}$, $\pi_ g(\tau)\equiv \tau( g)$ and we have 
\begin{align}\label{eq:W_Psi}
W^\Gamma(\mu,\eta)=w^\Gamma(\tau_\mu,\tau_\eta)\,,
\end{align}
where $w^\Gamma:C_b(S)^*\times C_b(S)^*\to [0,\infty]$ is given by
\begin{align}\label{eq:w_Gamma_def}
w^\Gamma =\sup_{ g\in\Gamma}\{\pi_ g\circ\pi_1-\pi_ g\circ\pi_2\}
\end{align}
and is therefore LSC in the product topology ($\pi_1$, $\pi_2$ denote the projection maps onto the first and second components).  By assumption, $\tau_{\mu_\alpha}\to\tau_\mu$ in the weak-* topology.  The fact that $\eta_{\alpha_\beta,\epsilon}\to\eta_\epsilon$ weakly implies that $\tau_{\eta_{\alpha_\beta,\epsilon}}\to\tau_{\eta_\epsilon}$ in the weak-* topology as well. Therefore $(\tau_{\mu_{\alpha_\beta}},\tau_{\eta_{{\alpha_\beta},\epsilon}})\to (\tau_\mu,\tau_{\eta_\epsilon})$ in the product topology on $C_b(S)^*\times C_b(S)^*$.  Lower semicontinuity of $w^\Gamma$ and the equality \eqref{eq:W_Psi} then imply
\begin{align}
W^\Gamma(\mu,\eta_\epsilon)\leq\liminf_\beta W^\Gamma(\mu_{\alpha_\beta},\eta_{\alpha_\beta,\epsilon})\,.
\end{align}

Next we need the following simple to prove lemma about nets in $(-\infty,\infty]$: 
Let $x_\beta,y_\beta$, $\beta\in B$ be nets in $(-\infty,\infty]$.  If $x_\beta$ and $y_\beta$ are nondecreasing then 
\begin{align}
\sup_\beta\{x_\beta+y_\beta\}=\sup_\beta x_\beta+\sup_\beta y_\beta\,.
\end{align}
Using this  we have
\begin{align}
\epsilon+a\geq&\sup_{\tilde\beta\in B}\inf_{\beta\geq \tilde \beta}\left\{D_f(\eta_{\alpha_\beta,\epsilon}\|P)+W^\Gamma(\mu_{\alpha_\beta},\eta_{\alpha_\beta,\epsilon})\right\}\\
\geq &\sup_{\tilde\beta\in B}\left\{\inf_{\beta\geq \tilde \beta}D_f(\eta_{\alpha_\beta,\epsilon}\|P)+\inf_{\beta\geq \tilde \beta}W^\Gamma(\mu_{\alpha_\beta},\eta_{\alpha_\beta,\epsilon})\right\}\notag\\
= &\sup_{\tilde\beta\in B}\inf_{\beta\geq \tilde \beta}D_f(\eta_{\alpha_\beta,\epsilon}\|P)+\sup_{\tilde\beta\in B}\inf_{\beta\geq \tilde \beta}W^\Gamma(\mu_{\alpha_\beta},\eta_{\alpha_\beta,\epsilon})\notag\\
=&\liminf_\beta D_f(\eta_{\alpha_\beta,\epsilon}\|P)+\liminf_\beta W^\Gamma(\mu_{\alpha_\beta},\eta_{\alpha_\beta,\epsilon})\notag\\
\geq&D_f(\eta_\epsilon\|P)+W^\Gamma(\mu,\eta_\epsilon)\notag\\
\geq &\inf_{\eta\in \mathcal{P}(S)}\{D_f(\eta\|P)+W^\Gamma(\mu,\eta)\}\notag\\
=&H_1^*\Box H_2^*(\tau_\mu)\,.\notag
\end{align}
This holds for all $\epsilon>0$ and so $\tau_\mu\in\{H_1^*\Box H_2^*\leq a\}$.  This proves that $\{H_1^*\Box H_2^*\leq a\}$ is closed for all $a\in\mathbb{R}$ and hence we have proven lower semicontinuity of $H_1^*\Box H_2^*$.  Therefore we can conclude that
\begin{align}
D_f^\Gamma(\mu\|P)=&(H_1+H_2)^*(\tau_\mu)=H_1^*\Box H_2^*(\tau_\mu)\\
=&\inf_{\eta\in\mathcal{P}(S)}\{D_f(\eta\|P)+W^\Gamma(\mu,\eta)\}\notag
\end{align}
as claimed.

\item  If $D_f^\Gamma(\mu\|P)<\infty$  then there exists $\eta_n\in\mathcal{P}(S)$ such that
\begin{align}\label{eq:D_f_Phi_min_limit}
D_f^\Gamma(\mu\|P)=\lim_n(D_f(\eta_n\|P)+W^\Gamma(\mu,\eta_n))\,,
\end{align}
with $D_f(\eta_n\|P)+W^\Gamma(\mu,\eta_n)$  finite for all $n$. $W^\Gamma\geq 0$   and so \eqref{eq:D_f_Phi_min_limit}   implies $D_f(\eta_n\|P)$ is a bounded sequence, i.e., there exists $M\in\mathbb{R}$ with $\eta_n\in\{Q:D_f(Q\|P)\leq M\}$. Lemma \ref{lemma:Df_compact_sublevel} implies $Q\mapsto D_f(Q\|P)$ has compact sublevel sets, hence there exists a convergence subsequence $\eta_{n_j}\to\eta_*\in\mathcal{P}(S)$. By Corollary \ref{cor:Df_LSC}, $(Q,P)\mapsto D_f(Q\|P)$ is LSC and so $D_f(\eta_*\|P)\leq \liminf_j D_f(\eta_{n_j}\|P)$. $\eta \mapsto W^\Gamma(\mu,\eta)$ is the supremum of a collection of continuous functions on $\mathcal{P}(S)$, and so is also LSC. Therefore $W^\Gamma(\mu,\eta_*)\leq\liminf_j W^\Gamma(\mu,\eta_{n_j})$ and we have
\begin{align}
&D_f^\Gamma(\mu\|P)\leq D_f(\eta_*\|P)+W^\Gamma(\mu,\eta_*)\leq  \liminf_j D_f(\eta_{n_j}\|P)+\liminf_j W^\Gamma(\mu,\eta_{n_j})\\
\leq& \liminf_j( D_f(\eta_{n_j}\|P)+W^\Gamma(\mu,\eta_{n_j}))=D_f^\Gamma(\mu\|P)\,.\notag
\end{align}
This completes the proof of \eqref{eq:inf_conv_existence_app}. If $f$ is strictly convex then  the map $\eta\mapsto D_f(\eta\|P)$ is strictly convex on the set $\{\eta:D_f(\eta\|P)<\infty\}$ (see Lemma \ref{lemma:strictly_convex}).  It is straightforward to see that $\eta\mapsto W^\Gamma(\mu,\eta)$ is convex.  Therefore the objective functional in \eqref{eq:inf_conv_app} is strictly convex and hence has at most one minimizer. 
\item We have already noted that $W^\Gamma\geq 0$.  The property $W^\Gamma(\mu,\mu)=0$ is trivial from the definition. If  $\Gamma$ is strictly admissible and $Q,P\in\mathcal{P}(S)$ with $W^\Gamma(Q,P)=0$ then for $\psi\in\Psi$ we can find $c\in\mathbb{R}$, $\epsilon>0$ such that $c\pm\epsilon\psi\in\Gamma$. Therefore
\begin{align}
    0=W^\Gamma(Q,P)\geq E_Q[c\pm\epsilon\psi]-E_P[c\pm\epsilon\psi]=\pm\epsilon(E_Q[\psi]-E_P[\psi])\,.
\end{align}
From this we can conclude that $E_Q[\psi]=E_P[\psi]$ for all $\psi\in\Psi$ and hence $Q=P$.
\item Both $D_f$ and $W^\Gamma$ are non-negative, hence the infimal convolution formula \eqref{eq:inf_conv_app} implies $D_f^\Gamma\geq 0$.  By taking $\eta=P$ in \eqref{eq:inf_conv_app} it is easy to see that $D_f^\Gamma(P\|P)=0$.  Now suppose $f$ and $\Gamma$ are strictly admissible.  Strict convexity of $f$ at $1$ implies that $D_f$ has the divergence property \cite{LieseVajda}.   If $Q,P\in\mathcal{P}(S)$ with $D_f^\Gamma(Q\|P)=0$ then \req{eq:inf_conv_existence_app} implies there exists $\eta_*\in\mathcal{P}(S)$ with
\begin{align}
0=D_f^\Gamma(Q\|P)=D_f(\eta_*\|P)+W^\Gamma(Q,\eta_*)\,.
\end{align}
Therefore $D_f(\eta_*\|P)=0=W^\Gamma(Q,\eta_*)$.  The first equality implies $\eta_*=P$ and so we have $W^\Gamma(Q,P)=0$.  We have assumed $\Gamma$ is strictly admissible, hence Part 3 implies $Q=P$.
\end{enumerate}
\end{proof}
Under slightly stronger assumptions we find that $D_f^\Gamma(\mu\|P)$  is infinite when $\mu\not\in M_1(S)\equiv\{\mu\in M(S):\mu(S)=1\}$.
\begin{corollary}
Suppose $f$ and $\Gamma$ are admissible and $\Gamma$ contains the constant functions. Then for $P\in\mathcal{P}(S)$, $\mu\in M(S)\setminus M_1(S)$ we have $D_f^\Gamma(\mu\|P)=\infty$.
\end{corollary}
\begin{proof}
We have assumed that $\Gamma$   contains the constant functions. Therefore, for any $\eta\in\mathcal{P}(S)$ we have
\begin{align}
    W^\Gamma(\mu,\eta)\geq c(\mu(S)-\eta(S))=c(\mu(S)-1)
\end{align}
for all $c\in\mathbb{R}$.  If $\mu(S)\neq 1$ then taking $c\to\pm\infty$ implies $ W^\Gamma(\mu,\eta)=\infty$.  This holds for all $\eta\in\mathcal{P}(S)$ and so \req{eq:inf_conv_app} implies $D_f^\Gamma(\mu\|P)=\infty$.
\end{proof}

{ 
 Here we prove that, under appropriate assumptions, the unit ball in a RKHS is an  admissible set and hence falls under the purview of Theorem \ref{thm:f_div_inf_convolution}. See Chapter 4 in \cite{steinwart2008support} for a detailed treatment of the properties of an RKHS, several of which are used below.
\begin{lemma}\label{lemma:RKHS}
Let $X\subset C_b(S)$ be a separable RKHS with  reproducing-kernel $k:S\times S\to\mathbb{R}$. Let $\Gamma=\{g\in X:\|g\|_X\leq 1\}$ be the unit ball in $X$.  Then $\Gamma$ is admissible.  
\end{lemma}
\begin{proof}
We clearly have $0\in\Gamma$ and $\Gamma$ is convex. Therefore we just need to show that $\Gamma$ is a closed subset of $C_b(S)$ under the weak topology generated by $M(S)$: First note that  $\Gamma$ is compact in the weak topology induced by $X^*$; this follows from  Alaoglu's theorem (see, e.g., Theorem 5.18 in \cite{folland2013real}) together with the fact that $X$ is reflexive. 

Next we show that the topology on $\Gamma$ induced by $M(S)$ is the same as the topology induced by $X^*$. To do this, first recall that the assumption  $X\subset C_b(S)$ implies $k$ is bounded, separately continuous, and jointly measurable. This allows us to define the  linear map $\mu_X:M(S)\to X$ by $\mu_X(\nu)=\int k(\cdot,x)\nu(dx)$ that satisfies
\begin{align}
    \tau_\nu(g)\equiv\int g d\nu=\langle g,\mu_X(\nu)\rangle_X 
\end{align}
for all $g\in X$.  This shows that $\tau_\nu\in X^*$ for all $\nu\in M(S)$.  Therefore the topology on $\Gamma$ induced by $M(S)$ is weaker than the topology induced by $X^*$.  The former is Hausdorff (since $M(S)$ separates points) and, as shown above, the latter is compact.  Therefore the two topologies are in fact equal (see, e.g., Proposition 4.28 in \cite{folland2013real}). 

Combining the above two properties we conclude that the weak topology on $\Gamma$ induced by $M(S)$ is compact.  The topology induced by $M(S)$ on $C_b(S)$ is Hausdorff (again because $M(S)$ separates points) and $\Gamma$ is a compact subset of this space, hence we have proven that $\Gamma$ is closed in $C_b(S)$.
\end{proof}

\begin{remark}\label{remark:RKHS_strictly_adm}
By imposing various additional conditions on the kernel (e.g., if it is characteristic or universal) one can ensure that the unit ball in $X$ is measure determining and hence is strictly admissible; see \cite{JMLR:v12:sriperumbudur11a} and references therein.
\end{remark}
}

Now we prove the limiting properties from Theorem \ref{thm:limit}, which are repeated below.
\begin{theorem}\label{thm:limit_app}
Let $Q,P\in\mathcal{P}(S)$ and $\Gamma$, $f$ both be admissible.  Then for all $c>0$ the set $\Gamma_c\equiv\{c g: g\in\Gamma\}$ is admissible and we have the following two limiting formulas.
\begin{enumerate}
\item If $\Gamma$ is strictly admissible then the sets $\Gamma_L$ are strictly admissible for all $L>0$ and
\begin{align}
\lim_{L\to\infty} D^{\Gamma_L}_f(Q\|P)=D_f(Q\|P)\,.
\end{align}
\item If $f$ is strictly admissible then 
\begin{align}
\lim_{\delta\searrow 0}\frac{1}{\delta} D_f^{\Gamma_\delta}(Q\|P)=W^\Gamma(Q,P)\,.
\end{align}
\end{enumerate}
\end{theorem}
\begin{proof}
\begin{enumerate}
    \item From the definition, it is straightforward to see that  $\Gamma_c$ is strictly admissible for all $c>0$ and $W^{\Gamma_c}(\mu,\kappa)=cW^\Gamma(\mu,\kappa)$. To prove that $\lim_{L\to\infty}D_f^{\Gamma_L}(Q\|P)=D_f(Q\|P)$, first suppose that $D_f(Q\|P)<\infty$:  From \req{eq:inf_conv_app} we see that $D^{\Gamma_L}_f(Q\|P)\leq D_f(Q\|P)<\infty$ for all $L>0$ and \eqref{eq:inf_conv_existence_app} implies that there exists $\eta_{*,L}\in\mathcal{P}(S)$ such that  $D_f^{\Gamma_L}(Q\|P)=D_f(\eta_{*,L}\|P)+W^{\Gamma_L}(Q,\eta_{*,L})$.  Take a sequence $L_n\nearrow\infty$. We have
\begin{align}
 D_f(\eta_{*,L_n}\|P)\leq D_f^{\Gamma_L}(Q\|P) \leq D_f(Q\|P)\equiv M<\infty\,,
\end{align}
and so for all $n$ we have $\eta_{*,L_n}\in\{D_f(\cdot\|P)\leq M\}$, a compact set (see Lemma \ref{lemma:Df_compact_sublevel}).  Hence there exists a weakly convergence subsequence $\eta_{*,L_{n_j}}\to\eta_*$.  We can compute
\begin{align}
W^\Gamma(Q,\eta_*)\leq& \liminf_j W^\Gamma(Q,\eta_{*,L_{n_j}})=  \liminf_j  \frac{1}{L_{n_j}}W^{\Gamma_ {L_{n_j} } }(Q,\eta_{*,L_{n_j}})\\
\leq&  \liminf_j  \frac{1}{L_{n_j}}D_f^{\Gamma_{L_{n_j}}}(Q\|P)\leq \liminf_j  \frac{1}{L_{n_j}}D_f(Q\|P)=0\,.\notag
\end{align}
Therefore $W^\Gamma(Q,\eta_*)=0$. $\Gamma$ is strictly admissible, hence $W^\Gamma$ has the divergence property (see Part 3 of Theorem \ref{thm:f_div_inf_convolution}) and so $\eta_*=Q$. Therefore we can use lower semicontinuity of $D_f$ to compute
\begin{align}
\liminf_jD_f^{\Gamma_{L_{n_j}}}(Q\|P)=&\liminf_j(D_f(\eta_{*,L_{n_j}}\|P)+W^{\Gamma_L}(Q,\eta_{*,L_{n_j}}))\\
\geq& \liminf_jD_f(\eta_{*,L_{n_j}}\|P)\geq D_f(Q\|P)\,.\notag
\end{align}
Combining this with the fact that $D_f^{\Gamma_{L_n}}(Q\|P)\leq D_f(Q\|P)$ we see that $\lim_j D_f^{\Gamma_{L_{n_j}}}(Q\|P)=D_f(Q\|P)$.  Therefore we have shown that every sequence $L_n\nearrow \infty$ has a subsequence with $D_f^{\Gamma_{L_{n_j}}}(Q\|P)\to D_f(Q\|P)$.  This implies $D_f^{\Gamma_{L_n}}(Q\|P)\to D_f(Q\|P)$ and so we have proven the result in the  case  where  $D_f(Q\|P)<\infty$.

Now  suppose $D_f(Q\|P)=\infty$: If $\lim_{L\to\infty}D_f^{\Gamma_L}(Q\|P)\neq \infty$ then there exists $R\in\mathbb{R}$ and $L_n\to\infty$ with $D_f^{\Gamma_{L_n}}(Q\|P)\leq R$ for all $n$.  \req{eq:inf_conv_existence_app} implies there exists $\eta_{*,n}\in\mathcal{P}(S)$ such that 
\begin{align}\label{eq:R_Df_bound}
R\geq D_f^{\Gamma_{L_n}}(Q\|P)=D_f(\eta_{*,n}\|P)+W^{\Gamma_{L_n}}(Q,\eta_{*,n})\geq D_f(\eta_{*,n}\|P)\,.
\end{align}
Using compactness of sublevel sets   we again see that there is a convergent subsequence $\eta_{*,{n_j}}\to \eta_*$. Similarly to \eqref{eq:R_Df_bound}, we can compute
\begin{align}
R\geq& D_f^{\Gamma_{L_n}}(Q\|P)=D_f(\eta_{*,n}\|P)+W^{\Gamma_{L_n}}(Q,\eta_{*,n})\geq W^{\Gamma_{L_n}}(Q,\eta_{*,n})=L_nW^{\Gamma}(Q,\eta_{*,n})\,.
\end{align}
This implies
\begin{align}
R/ L_n\geq W^{\Gamma}(Q,\eta_{*,n})\geq 0\,,
\end{align}
and so $W^{\Gamma}(Q,\eta_{*,n})\to 0$. $\Gamma$ is strictly admissible and $\eta_{*,{n_j}}\to \eta_*$ weakly, hence a similar argument to that of the previous case  implies $Q=\eta_*$ and
\begin{align}
D_f(Q\|P)\leq& \liminf_j D_f(\eta_{*,{n_j}}\|P)\leq \liminf_j (D_f(\eta_{*,{n_j}}\|P)+W^{\Gamma_{L_{n_j}}}(Q,\eta_{*,n_j}))\notag\\
=& \liminf_jD_f^{\Gamma_{L_{n_j}}}(Q\|P)\leq R\,.
\end{align}
Therefore $D_f(Q\|P)\leq R<\infty$, a contradiction.  This completes the proof.
\item It is again straightforward to see that $\Gamma_\delta$ is admissible and $W^{\Gamma_\delta}=\delta W^\Gamma$. Fix $Q,P\in\mathcal{P}(S)$ and take a sequence $\delta_n\searrow 0$.  Using the infimal convolution formula \eqref{eq:inf_conv_app} we see that
\begin{align}
\frac{1}{\delta_n}D_f^{\Gamma_{\delta_n}}(Q\|P)=\inf_{\eta\in\mathcal{P}(S)}\{\delta_n^{-1}D_f(\eta\|P)+W^\Gamma(Q,\eta)\}\leq W^\Gamma(Q,P)
\end{align}
and the left-hand-side is nondecreasing in $n$. Therefore
\begin{align}
\lim_{n\to\infty}\frac{1}{\delta_n}D_f^{\Gamma_{\delta_n}}(Q\|P)=\sup_n \frac{1}{\delta_n}D_f^{\Gamma_{\delta_n}}(Q\|P)\leq W^\Gamma(Q,P)\,.
\end{align}
Suppose $\sup_n \frac{1}{\delta_n}D_f^{\Gamma_{\delta_n}}(Q\|P)< W^\Gamma(Q,P)$:  This implies $D_f^{\Gamma_{\delta_n}}(Q\|P)<\infty$ for all $n$, hence \req{eq:inf_conv_existence_app} implies there exists $\eta_{*,n}\in\mathcal{P}(S)$ such that $D_f^{\Gamma_{\delta_n}}(Q\|P)=D_f(\eta_{*,n}\|P)+W^{\Gamma_{\delta_n}}(Q,\eta_{*,n})$. Therefore
\begin{align}
\infty>\sup_n \frac{1}{\delta_n}D_f^{\Gamma_{\delta_n}}(Q\|P)\geq\sup_n\frac{1}{\delta_n} D_f(\eta_{*,n}\|P)\,.
\end{align}
$\delta_n\searrow 0$ and so this implies  $D_f(\eta_{*,n}\|P)$ is uniformly bounded. $\eta\mapsto D_f(\eta\|P)$ has compact sublevel sets (see Lemma \ref{lemma:Df_compact_sublevel}), hence there exists a weakly convergent subsequence $\eta_{*,{n_j}}\to \eta_*$.  The fact that $\sup_n\frac{1}{\delta_n} D_f(\eta_{*,n}\|P)<\infty$ together with $\delta_n\searrow 0$ implies $D_f(\eta_{*,n}\|P)\to 0$.  $D_f$ is LSC and $\eta_{*,{n_j}}\to \eta_*$ so this implies 
\begin{align}
0\leq D_f(\eta_*\|P)\leq \liminf_j D_f(\eta_{*,{n_j}}\|P)=0\,,
\end{align}
i.e., $D_f(\eta_*\|P)=0$. $f$ is strictly convex at $1$, hence $D_f$ has the divergence property and so $\eta_*=P$.  Therefore $\eta_{*,{n_j}}\to P$ weakly and we can  compute
\begin{align}
W^\Gamma(Q,P)>&\sup_n\frac{1}{\delta_{n}}D_f^{\Gamma_{\delta_{n}}}(Q\|P) \geq \liminf_j\frac{1}{\delta_{n_j}}D_f^{\Gamma_{\delta_{n_j}}}(Q\|P)\\
\geq& \liminf_j\frac{1}{\delta_{n_j}}W^{\Gamma_{\delta_{n_j}}}(Q,\eta_{*,{n_j}})=\liminf_jW^\Gamma(Q,\eta_{*,{n_j}})\geq W^\Gamma(Q,P)\,.\notag
\end{align}
This is a contradiction, hence we can conclude that
\begin{align}
\lim_{n\to\infty}\frac{1}{\delta_n}D_f^{\Gamma_{\delta_n}}(Q\|P)=\sup_n \frac{1}{\delta_n}D_f^{\Gamma_{\delta_n}}(Q\|P)= W^\Gamma(Q,P)\,.
\end{align}
$\delta_n\searrow 0$ was arbitrary, therefore
\begin{align}
\lim_{\delta\searrow 0}\frac{1}{\delta}D_f^{\Gamma_{\delta}}(Q\|P) = W^\Gamma(Q,P)\,.
\end{align}
\end{enumerate}
\end{proof}

{ 
Next we prove the convergence and continuity results from Theorem \ref{thm:conv}.
\begin{theorem}\label{thm:conv_app}
Let $f\in\mathcal{F}_1(a,b)$ and $\Gamma\subset\mathcal{M}_b(\Omega)$. Then:
\begin{enumerate}
    \item If there exists $c_0\in \Gamma\cap\mathbb{R}$ then $W^\Gamma(Q_n,P)\to 0 \implies D_f^\Gamma(Q_n\|P) \to 0$ and $D_f(Q_n\|P)\to 0 \implies D_f^\Gamma(Q_n\|P) \to 0$, and similarly if one exchanges $Q_n$ and $P$.
    
    \item Suppose $f$ and $\Gamma$ also satisfy the following:
\begin{enumerate}
\item There exist a nonempty set $\Psi\subset\Gamma$ with the following properties:
\begin{enumerate}
\item $\Psi$ is $\mathcal{P}(\Omega)$-determining.
\item For all $\psi\in\Psi$  there exists $c_0\in\mathbb{R}$, $\epsilon_0>0$ such that $c_0+\epsilon \psi\in\Gamma$ for all $|\epsilon|<\epsilon_0$.
\end{enumerate}
\item $f$ is strictly convex on a neighborhood of $1$.
\item $f^*$ is finite and $C^1$ on a neighborhood of $\nu_0\equiv f_+^\prime(1)$.
\end{enumerate}
Let $P,Q_n\in\mathcal{P}(\Omega)$, $n\in\mathbb{Z}_+$. If $D_f^\Gamma(Q_n\|P)\to 0$ or  $D_f^\Gamma(P\|Q_n)\to 0$ then $E_{Q_n}[\psi]\to E_P[\psi]$ for all $\psi\in\Psi$. 
\item On a metric space $S$, if $f$ is admissible then the map $(Q,P)\in\mathcal{P}(S)\times\mathcal{P}(S)\mapsto D_f^\Gamma(Q\|P)$ is  lower semicontinuous.
\end{enumerate}
\end{theorem}
\begin{proof}
Part 1 follows from the upper bound \eqref{eq:inf_conv_ineq} and the lower bound from Part 3 of Theorem \ref{thm:general_ub}. Now work under the assumptions of Part 2 and suppose $D_f^\Gamma(Q_n\|P)\to 0$. Fix $\delta>0$ and take $N_\delta$ such that for all $n\geq N_\delta$ we have $D_f^\Gamma(Q_n\|P)\leq \delta$. Fix $\psi\in\Psi$ and, per Assumption 2.a.ii, take $c_0\in\mathbb{R}$ and $\epsilon_0>0$ such that $c_0+\epsilon\psi\in\Gamma$ for all $|\epsilon|<\epsilon_0$.  Using \eqref{eq:Df_Gamma_def2} we obtain
\begin{align}
    E_{Q_n}[\nu_0+\epsilon\psi]-E_P[f^*(\nu_0+\epsilon\psi)]\leq D_f^\Gamma(Q_n\|P)\leq \delta
\end{align}
for all $n\geq N_\delta$, $|\epsilon|<\epsilon_0$, where $\nu_0$ is as in \eqref{eq:nu0_identities}. Taylor expanding $f^*$ then gives
\begin{align}
    E_{Q_n}[\nu_0+\epsilon\psi]-E_P[f^*(\nu_0)+(f^*)^\prime(\nu_0)\epsilon\psi+R(\epsilon\psi)\epsilon\psi]\leq \delta
\end{align}
for all $n\geq N_\delta$, $|\epsilon|<\epsilon_0$ (using a possibly smaller $\epsilon_0$),
where the remainder, $R$, is continuous, bounded, and  satisfies $R(0)=0$. The identities  \eqref{eq:nu0_identities} then imply
\begin{align}
        \epsilon(E_{Q_n}[\psi]-E_P[\psi])\leq \delta+\epsilon E_P[R(\epsilon\psi)\psi]
\end{align}
for all $n\geq N_\delta$,  $|\epsilon|<\epsilon_0$. By appropriately choosing the sign of $\epsilon$, we therefore find
\begin{align}
    \sup_{n\geq N_\delta}|E_{Q_n}[\psi]-E_{P}[\psi]|\leq \delta/\epsilon+\|\psi\|_\infty \sup_{[-\epsilon\|\psi\|_\infty,\epsilon\|\psi\|_\infty]}|R|
\end{align}
for all $0<\epsilon<\epsilon_0$. For $\delta$ sufficiently small we can let $\epsilon=\delta^{1/2}$ and therefore find
\begin{align}
    \sup_{n\geq N_\delta}|E_{Q_n}[\psi]-E_{P}[\psi]|\leq \delta^{1/2}+\|\psi\|_\infty \sup_{[-\delta^{1/2}\|\psi\|_\infty,\delta^{1/2}\|\psi\|_\infty]}|R|\to 0
\end{align}
as $\delta\to 0$. Hence $E_{Q_n}[\psi]\to E_P[\psi]$ as claimed.  The proof in the case where $D_f^\Gamma(P\|Q_n)\to 0$ is similar. Finally, Part 3 follows  from \eqref{eq:Df_Gamma_def2}, which show that $D_f^\Gamma(Q\|P)$ is the supremum of functions that  are continuous in $(Q,P)$ (recall that Lemma \ref{lemma:f_star_cont} implies $f^*$ is finite and continuous on $\mathbb{R}$ and hence $f^*( g-\nu)\in C_b(S)$), where the topology on $\mathcal{P}(S)$ is induced by the Prohorov metric.
\end{proof}
}

{ 
Now we derive the data processing inequality from Theorem \ref{thm:data_proc}.
\begin{theorem}[Data Processing Inequality]\label{thm:data_proc_app}
Let $f\in\mathcal{F}_1(a,b)$, $Q,P\in\mathcal{P}(\Omega)$, and $K$ be a probability kernel from $(\Omega,\mathcal{M})$ to $(N,\mathcal{N})$. 
\begin{enumerate}
    \item  Let  $\Gamma\subset \mathcal{M}_b(N)$ be nonempty. Then  
    \begin{align}\label{eq:data_proc1_app}
        D_f^\Gamma\left(K[Q]\|K[P]\right)\leq D_f^{K[\Gamma]}(Q\|P)\,.
    \end{align}
    \item  Let  $\Gamma\subset \mathcal{M}_b(\Omega\times N)$ be nonempty. Then  
    \begin{align}\label{eq:data_proc2_app}
        D_f^\Gamma\left(Q\otimes K\|P\otimes K\right)\leq D_f^{K[\Gamma]}(Q\|P)\,.
        \end{align}
\end{enumerate}
\end{theorem}
\begin{proof}
From \req{eq:Df_Gamma_def2} we have
\begin{align}
    D_f^\Gamma(K[Q]\|K[P])=\sup_{ g\in\Gamma,\nu\in\mathbb{R}}\left\{\int \int   (g(y)-\nu) K_x(dy) Q(dx)-\int\int f^*( g(y)-\nu) K_x(dy) P(dx)\right\}\,.
\end{align}
Using convexity of $f^*$ we can apply Jensen's inequality to find
\begin{align}
   \int f^*(g(y)-\nu)K_x(dy)\geq f^*\left(\int (g(y)-\nu) K_x(dy)\right) 
\end{align}
for all $x\in\Omega$. Hence
\begin{align}
    D_f^\Gamma(K[Q]\|K[P])    \leq & \sup_{ g\in\Gamma,\nu\in\mathbb{R}}\left\{E_Q[ K[g]-\nu]-E_P[ f^*(K[g]-\nu)]\right\}=D_f^{K[\Gamma]}(Q\|P)\,.
\end{align}
This proves \req{eq:data_proc1_app}. The proof of \req{eq:data_proc2_app} is very similar and so we omit it.
\end{proof}
}

Next we prove (a generalization of) Theorem \ref{thm:Gibbs_optimizer}, which gives existence and uniqueness results regarding the dual optimization problem \eqref{eq:Df_Gibbs_var_formula} for the classical $f$-divergences.
\begin{theorem}\label{thm:Gibbs_optimizer_app}
Let $f\in\mathcal{F}_1(a,b)$,  $a\geq 0$, $P\in\mathcal{P}(\Omega)$,  and $ g\in\mathcal{M}_b(\Omega)$.
\begin{enumerate}
\item If $f$ is strictly convex then the optimization problem
\begin{align}\label{eq:Gibb_Q_unique_app}
\sup_{Q\in\mathcal{P}(\Omega)}\{E_Q[ g]-D_f(Q\|P)\}
\end{align} 
has at most one optimizer. 
\item  Suppose there exists $\nu_*\in\mathbb{R}$ such that $\range( g-\nu_*)\subset \{f^*<\infty\}^o$  and
\begin{align}\label{eq:E_f_prime_assump_app}
E_P[(f^*)^\prime_+( g-\nu_*)]=1\,.
\end{align}
Then 
\begin{align}\label{eq:D_f_Q_star_app}
dQ_*\equiv  (f^*)^\prime_+( g-\nu_*)dP
\end{align}
is a probability measure  and 
\begin{align}
&\sup_{Q\in\mathcal{P}(\Omega)}\{E_Q[ g]-D_f(Q\|P)\}= E_{Q_*}[ g]-D_f(Q_*\|P)\\
=&\nu_*+E_P[f^*( g-\nu_*)]=\inf_{\nu\in\mathbb{R}}\{\nu+E_P[f^*( g-\nu)]\}\,.\notag
\end{align}

\item If $f$ is strictly convex on $(a,b)$ and $\{f^*<\infty\}=\mathbb{R}$   then there exists $\nu_*\in\mathbb{R}$ such that
\begin{align}
E_P[(f^*)^\prime( g-\nu_*)]=1\,.
\end{align}
\end{enumerate}
\end{theorem}
\begin{proof}
\begin{enumerate}
\item  We obviously have
\begin{align}\label{eq:gibbs_equality}
\sup_{Q\in\mathcal{P}(\Omega)}\{E_Q[ g]-D_f(Q\|P)\}=\sup_{Q\in\mathcal{P}(\Omega):D_f(Q\|P)<\infty}\{E_Q[ g]-D_f(Q\|P)\}\,,
\end{align} 
and optimizers (if they exist) must be in $\{Q:D_f(Q\|P)<\infty\}$. If $f$ is strictly convex then Lemma \ref{lemma:strictly_convex} implies that  the map $Q\mapsto D_f(Q\|P)$ is strictly convex on the set $\{Q:D_f(Q\|P)<\infty\}$.  The objective functional $Q\mapsto E_Q[ g]-D_f(Q\|P)$ is therefore strictly concave and hence has at most one maximizer. 
\item Lemma \ref{lemma:f_star_inc} implies $f^*$ is nondecreasing, and so $(f^*)^\prime_+\geq 0$. Together with the assumption \eqref{eq:E_f_prime_assump_app}, this implies $Q_*$ is a probability measure. From Lemma \ref{lemma:f_star_formula} we have
\begin{align}\label{eq:LT_deriv_identity}
f((f^*)_+^\prime( g-\nu_*))=( g-\nu_*)(f^*)_+^\prime( g-\nu_*)-f^*( g-\nu_*)
\end{align}
 and so we can compute
\begin{align}
&\sup_{Q\in\mathcal{P}(\Omega)}\{E_Q[ g]-D_f(Q\|P)\}\geq  E_{Q_*}[ g]-D_f(Q_*\|P)\\
=&\nu_*+E_P[( g-\nu_*) (f^*)^\prime_+( g-\nu_*)-f( (f^*)^\prime_+( g-\nu_*))]\notag\\
=&\nu_*+E_P[f^*( g-\nu_*)]\geq\inf_{\nu\in\mathbb{R}}\{\nu+E_P[f^*( g-\nu)]\}\notag\\
=&\sup_{Q:D_f(Q\|P)<\infty}\{E_Q[ g]-D_f(Q\|P)\}\,,\notag
\end{align}
where we used \req{eq:LT_deriv_identity} to go from the second to the third line and we used \req{eq:Df_Gibbs} to obtain the last line. The equality \eqref{eq:gibbs_equality} then completes the proof.
\item Strict convexity of $f$ implies $f^*$ is $C^1$ (see   Theorem 26.3 in \cite{rockafellar1970convex}).  $g$ is bounded and so the dominated convergence theorem implies that the map $h:\mathbb{R}\to\mathbb{R}$, $h(\nu)= E_P[(f^*)^\prime(g-\nu)]$ is continuous. From Lemma \ref{lemma:nu0} we see that $\nu_0\equiv f_+^\prime(1)$ satisfies  $(f^*)^\prime(\nu_0)=1$. Convexity of $f^*$ implies that $(f^*)^\prime$ is nondecreasing, therefore
\begin{align}
h(\|g\|_\infty-\nu_0)\leq (f^*)^\prime(\nu_0)=1
\end{align}
and
\begin{align}
h(-\|g\|_\infty-\nu_0)\geq (f^*)^\prime(\nu_0)=1\,.
\end{align}
Continuity therefore implies there exists $\nu_*\in [-\|g\|_\infty-\nu_0,\|g\|_\infty-\nu_0]$ with $h(\nu_*)=1$ as claimed.
\end{enumerate}
\end{proof}

Finally, we prove the characterization from Theorem \ref{thm:inf_conv_sol} in the more general case of $D_f^\Gamma(\mu\|P)$ where $\mu\in M(S)$. 
\begin{theorem}\label{thm:inf_conv_sol_app}
Let  $\Gamma\subset C_b(S)$ be admissible and $f\in\mathcal{F}_1(a,b)$ be admissible, where   $a\geq 0$.  Fix $P\in\mathcal{P}(S)$, $\mu\in M(S)$.  Suppose we have $ g_*\in\Gamma$ and $\nu_*\in\mathbb{R}$ that satisfy the following:
\begin{enumerate}
\item $f((f^*)^\prime_+( g_*-\nu_*))\in L^1(P)$,
\item  $E_P[(f^*)^\prime_+( g_*-\nu_*)]=1$,
\item $W^\Gamma(\mu,\eta_*)=\int  g_*d\mu-\int g_*d\eta_*$, where  $d\eta_*\equiv (f^*)^\prime_+( g_*-\nu_*)dP$.
\end{enumerate}
Then $\eta_*\in\mathcal{P}(S)$ solves the infimal convolution problem \eqref{eq:inf_conv_app} (\req{eq:inf_conv} for the case of $\mu=Q\in\mathcal{P}(S)$) and
\begin{align}\label{eq:Df_Phi_formula_app}
D_f^\Gamma(\mu\|P)=\int  g_*d\mu-(\nu_*+E_P[f^*( g_*-\nu_*)])\,.
\end{align}
If $f$ is strictly convex then $\eta_*$ is the unique solution to the infimal convolution problem.
\end{theorem}
\begin{proof}
Admissibility of $f$ implies $\{f^*<\infty\}=\mathbb{R}$. Convexity of $f^*$ then implies that the right derivative $(f^*)_+^\prime$   exists everywhere and so  $\range( g_*-\nu_*)\subset \{(f^*)^\prime_+<\infty\}$. Therefore we can use Theorem \ref{thm:Gibbs_optimizer_app} to conclude that
\begin{align}
d\eta_*=  (f^*)^\prime_+( g_*-\nu_*)dP
\end{align}
is a probability measure, $D_f(\eta_*\|P)<\infty$, and 
\begin{align}
\sup_{Q: D_f(Q\|P)<\infty}\{E_Q[ g_*]-D_f(Q\|P)\}= &E_{\eta_*}[ g_*]-D_f(\eta_*\|P)=\nu_*+E_P[f^*( g_*-\nu_*)]\\
=&\inf_{\nu\in\mathbb{R}}\{\nu+E_P[f^*( g_*-\nu)]\}\,.\notag
\end{align}
In particular,
\begin{align}
 D_f(\eta_*\|P)-E_{\eta_*}[ g_*]=-(\nu_*+E_P[f^*( g_*-\nu_*)])\,.
\end{align}

Using Part 1 of Theorem \ref{thm:f_div_inf_convolution_app} we can compute
\begin{align}
D_f^\Gamma(\mu\|P)\leq& D_f(\eta_*\|P)+W^\Gamma(\mu,\eta_*)\\
=&D_f(\eta_*\|P)+\int  g_*d\mu-\int g_*d\eta_*\notag\\
=&\int  g_*d\mu-(\nu_*+E_P[f^*( g_*-\nu_*)])\notag\\
=&\int  g_*d\mu-\inf_{\nu\in\mathbb{R}}\{\nu+E_P[f^*( g_*-\nu)]\}\notag\\
\leq&\sup_{ g\in\Gamma}\left\{\int  g d\mu-\inf_{\nu\in\mathbb{R}}\{\nu+E_P[f^*( g-\nu)]\}\right\}=D_f^\Gamma(\mu\|P)\,.\notag
\end{align}
Therefore $\eta_*$ solves the infimal convolution problem and \eqref{eq:Df_Phi_formula_app} holds.  If $f$ is strictly convex then, as shown in Theorem \ref{thm:f_div_inf_convolution_app}, the solution to the infimal convolution problem is unique.
\end{proof}

\section{Strict Concavity  of the $(f,\Gamma)$-Divergence Objective Functional}\label{app:Taylor}
Here we will (formally) compute the Taylor expansion of the objective functional in \eqref{eq:gen_f_def} for the $(f,\Gamma)$-divergences. The computation is very similar to the  classical $f$-divergence case considered in Appendix C of \cite{birrell2020optimizing}. First define
\begin{align}
    H_f[ g;Q,P]=E_Q[ g]-\inf_{\nu\in\mathbb{R}}\{\nu+E_P[f^*( g-\nu)]\}
\end{align}
and note that this map is concave in $ g$, due to the convexity of $f^*$.  In fact, under weak assumptions it is strictly concave, as we now show. Take a line segment $ g_\epsilon= g_0+\epsilon\psi\in\Gamma$, $\epsilon\in(-\delta,\delta)$. We will compute $\frac{d^2}{d\epsilon^2}|_{\epsilon=0}H_f[ g_\epsilon;Q,P]$.

The optimization problem $\inf_{\nu\in\mathbb{R}}\{\nu+E_P[f^*( g_\epsilon-\nu)]\}$ is solved by $\nu_\epsilon$ that satisfies
\begin{align}
    0=\partial_\nu|_{\nu=\nu_\epsilon}\{\nu+E_P[f^*( g_\epsilon-\nu)]\}\,,
\end{align}
i.e.,
\begin{align}\label{eq:taylor1}
E_P[(f^*)^\prime( g_\epsilon-\nu_\epsilon)]=1
\end{align}
for all $\epsilon$. Differentiating this at $\epsilon=0$ we find
\begin{align}\label{eq:taylor2}
\nu^\prime_0=E_{P_0}[\psi]\,,\,\,\,\,\,dP_0\equiv \frac{(f^*)^{\prime\prime}( g_0-\nu_0)}{E_P[(f^*)^{\prime\prime}( g_0-\nu_0)]}dP\,.
\end{align}
Note that convexity of $f^*$ implies $E_P[(f^*)^{\prime\prime}( g_0-\nu_0)]\geq 0$.  We assume this inequality is strict. We can compute
\begin{align} 
    \frac{d}{d\epsilon}|_{\epsilon=0} H[ g_\epsilon;Q,P]=&E_Q[\psi]-\nu_0^\prime-E_P[(f^*)^\prime( g_0-\nu_0)\psi]+E_P[(f^*)^\prime( g_0-\nu_0)]\nu_0^\prime \notag\\
    =&E_Q[\psi]-E_P[(f^*)^\prime( g_0-\nu_0)\psi]\,,\label{eq:first:deriv:appendix}\\
    \frac{d^2}{d\epsilon^2}|_{\epsilon=0} H_f[ g_\epsilon;Q,P]=&-\nu_0^{\prime\prime}-E_P[(f^*)^{\prime\prime}( g_0-\nu_0)(\psi-\nu_0^\prime)^2]+E_P[(f^*)^\prime( g_0-\nu_0)]\nu_0^{\prime\prime}\notag\\
    =&-E_P[(f^*)^{\prime\prime}( g_0-\nu_0)]\Var_{P_0}[\psi]\,,
    \label{eq:second:deriv:appendix}
\end{align}
where we used \req{eq:taylor1} and \req{eq:taylor2} to simplify and
\begin{align}
    \nu_0={\mbox{argmin}}_{\nu\in\mathbb{R}}\{\nu+E_P[f^*( g_0-\nu)]\}\,.
\end{align}
In particular, the second derivative is strictly negative when $\Var_{P_0}[\psi]\neq 0$, i.e., $H_f[ g;Q,P]$ is strictly concave at $ g_0$ in all directions, $\psi$, of nonzero variance under $P_0$. This can be made more explicit in the KL case. First recall the objective functional from \req{eq:Lambda_KL},
\begin{align}\label{eq:H_KL_def}
    H_{KL}[ g;Q,P]\equiv E_Q[  g]-\inf_{\nu\in\mathbb{R}}\{\nu+E_P[f_{KL}^*( g-\nu)]\}=E_Q[  g]-\log E_P[e^{ g}]\,.
\end{align}
Fixing $ g_0\in\Gamma$ and perturbing in a direction $\psi$ we can compute
\begin{align}\label{eq:HKL_taylor}
    \frac{d^2}{d\epsilon^2}|_{\epsilon=0}H_{KL}[ g_0+\epsilon\psi;Q,P]=&-(E_P[\psi^2e^{ g_0}] E_P[e^{ g_0}]^{-1}-E_P[\psi e^{ g_0}]^2E_P[e^{ g_0}]^{-2})\\
    =&-\Var_{P_0}[\psi],\,\,\,\, dP_0\equiv e^{ g_0}dP/E_P[e^{ g_0}]\,.\notag
\end{align}
Thus we again have strict convexity in all directions $\psi$ of nonzero variance under $P_0$.

{  \section{ Additional Figures} \label{app:extra_figs}}

\newpage

\begin{figure}[h]
  \centering
\begin{subfigure}[b]{1\textwidth}
   \includegraphics[width=1\linewidth]{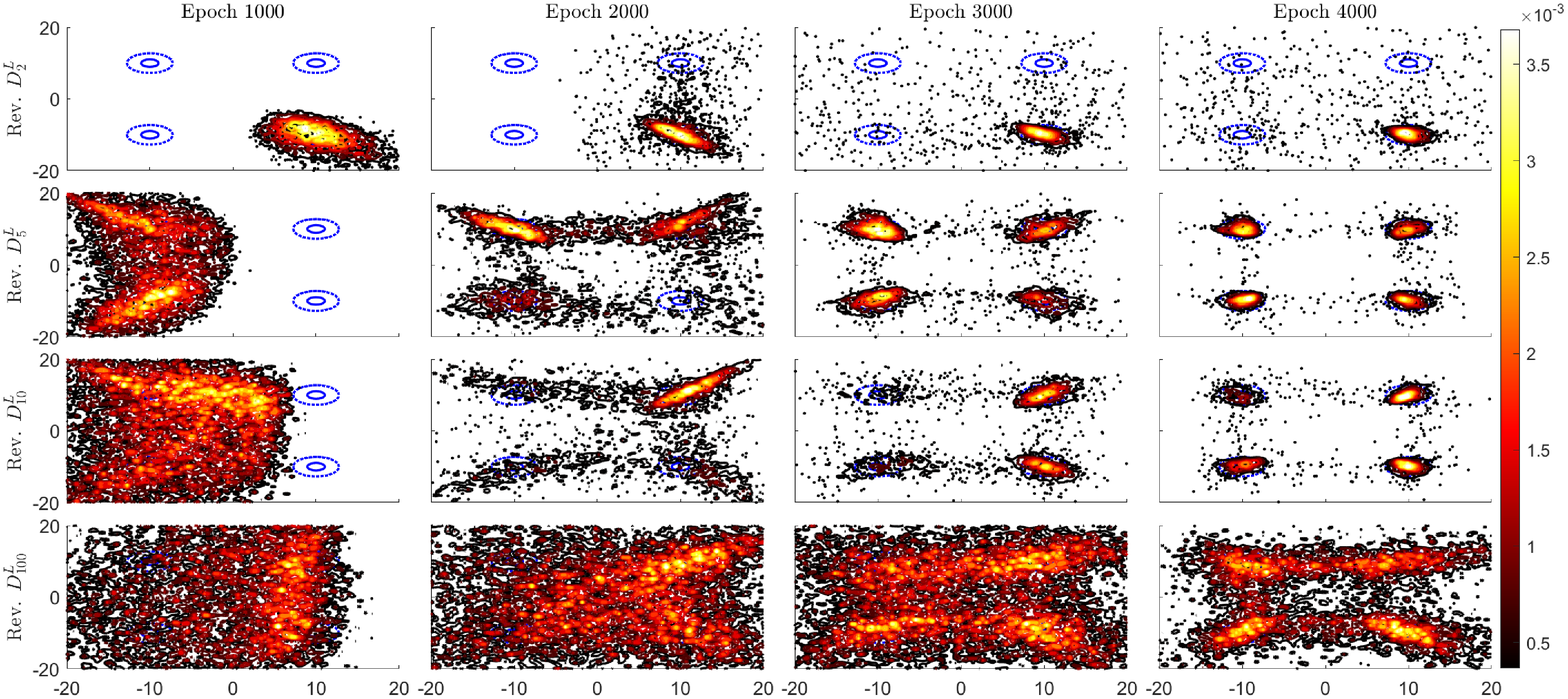}
   \caption{}
\end{subfigure}

\begin{subfigure}[b]{1\textwidth}
   \includegraphics[width=1\linewidth]{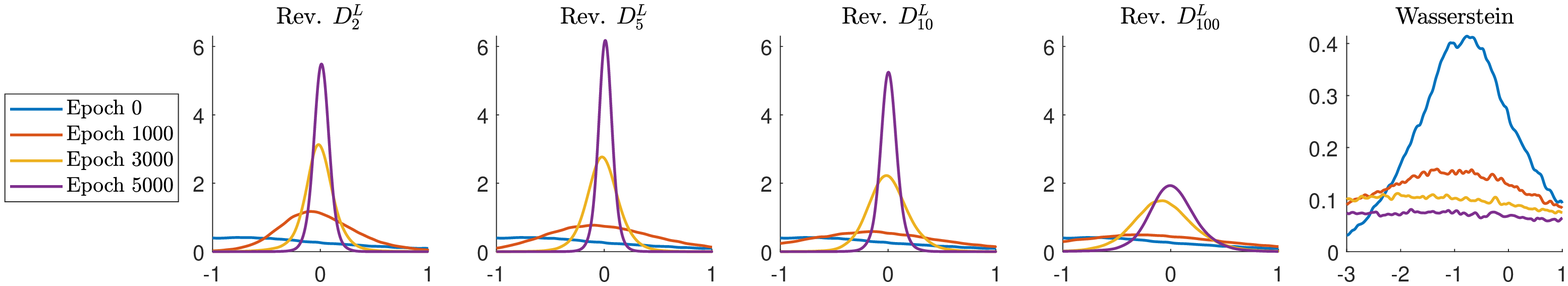}
   \caption{}
\end{subfigure}

\begin{subfigure}[b]{.99\textwidth}
 \begin{minipage}[b]{0.49\linewidth}
  \centering
\includegraphics[scale=.49]{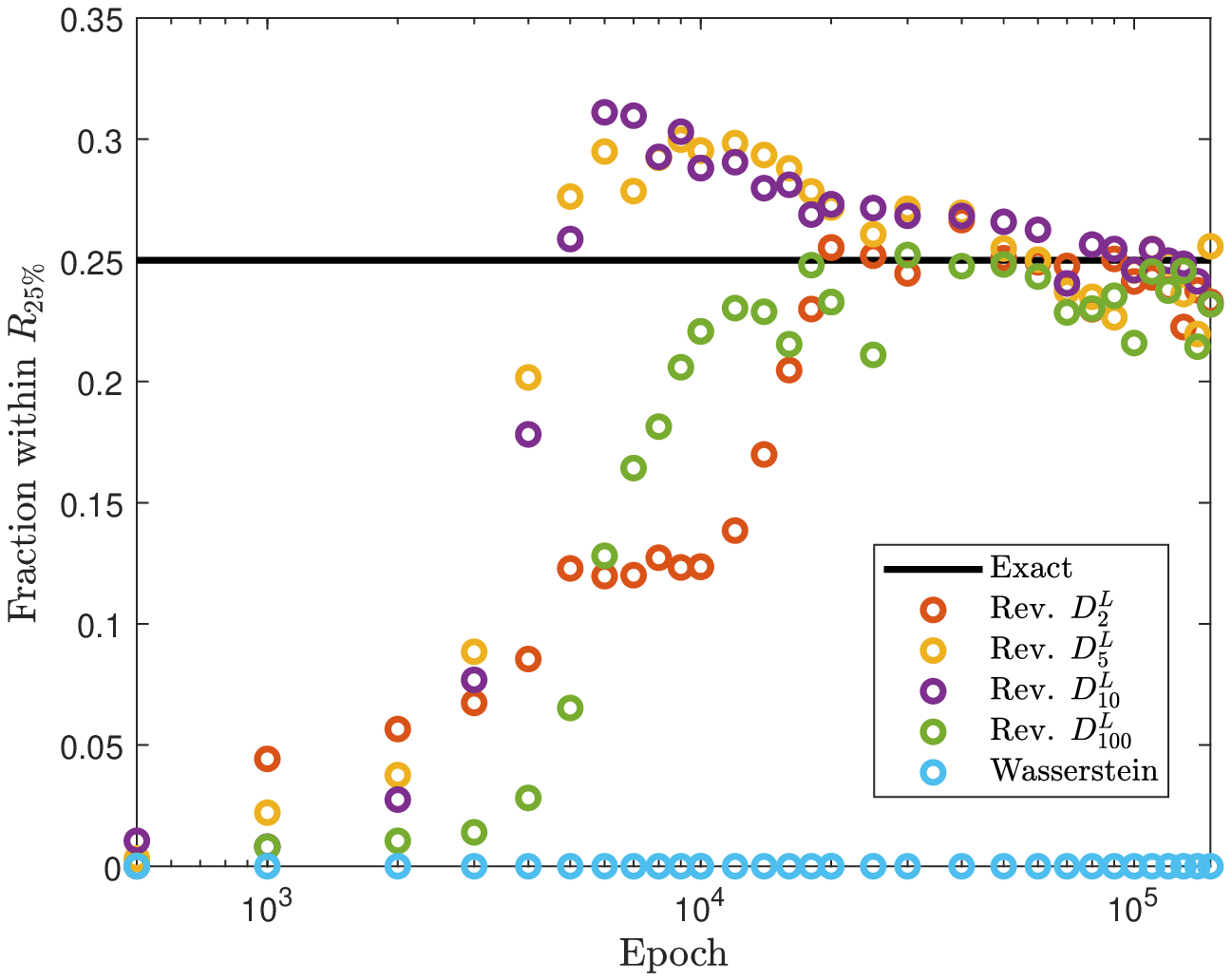}\end{minipage}
\begin{minipage}[b]{0.49\linewidth}
\centering
\includegraphics[scale=.49]{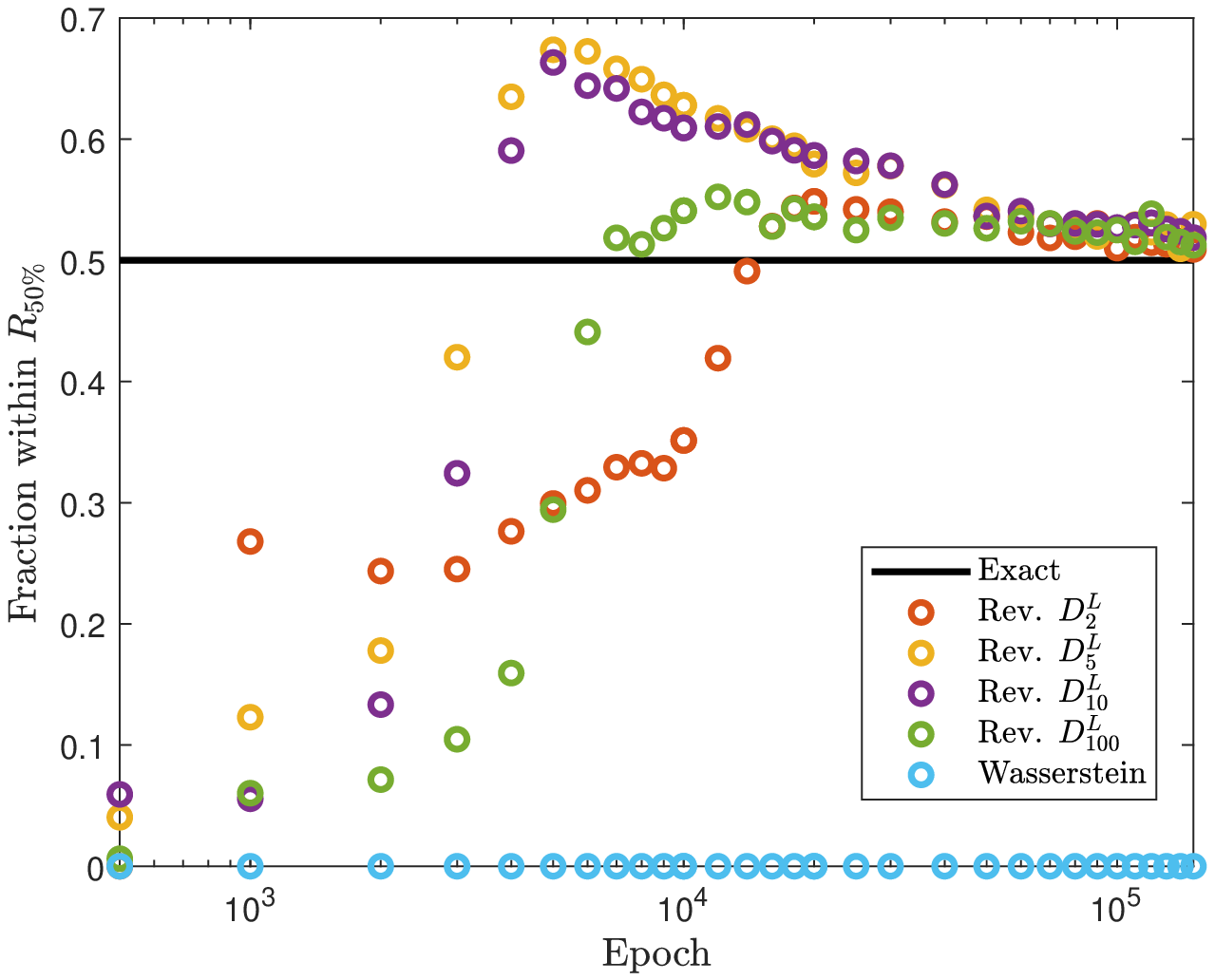} \end{minipage}
   \caption{}
\end{subfigure}
   \caption{{  Here we present generator samples and their statistical behavior from Wasserstein and reverse Lipschitz $\alpha$-GAN methods using the same setup as in Figure \ref{fig:gen_alpha_GAN_2}, except that training was done with a much larger set of samples (100000 samples). We obtain similar results, with the primary difference being that the training converges faster. }
     }\label{fig:gen_alpha_GAN_2}
\end{figure}

\newpage

\begin{figure}[h]
  \centering
\includegraphics[scale=.465]{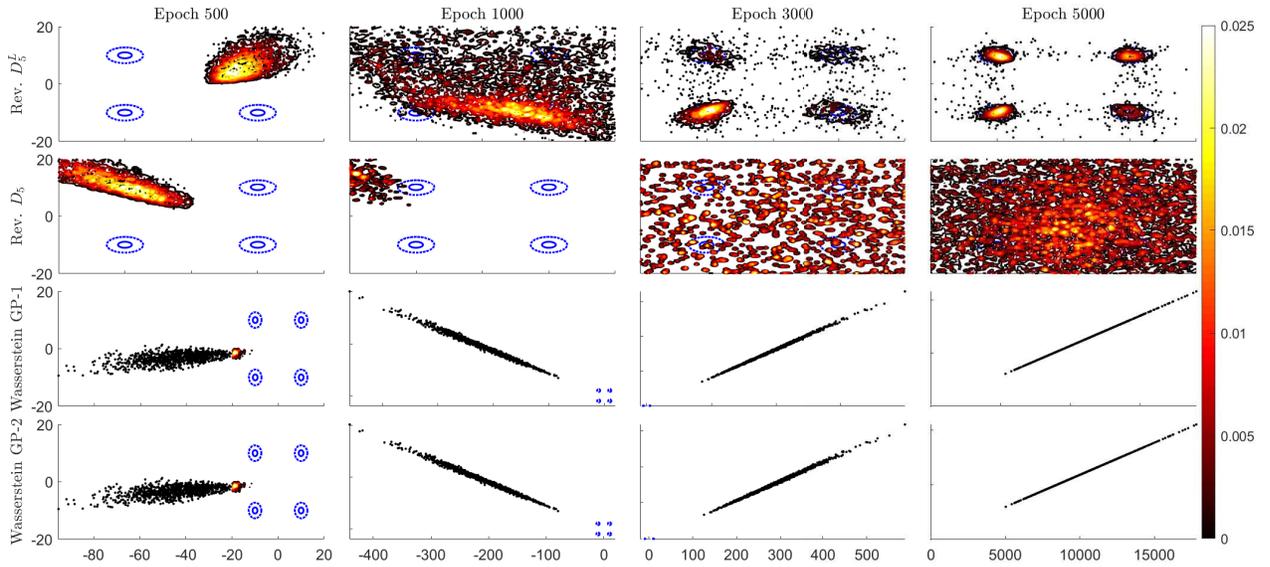} 

\caption{ {  
Here we present generator samples  from WGAN-GP, reverse classical $f$-GAN (denoted $D_\alpha$), and reverse Lipschitz $\alpha$-GAN  using the setup described in Section \ref{sec:submanifold_ex} and Figure \ref{fig:gen_alpha_GAN_2}, except that we do not embed in higher-dimensional space; the results are similar to what was described in Section \ref{sec:submanifold_ex}. In the case of classical $f$-GAN this is intriguing since, unlike in Figure \ref{fig:gen_alpha_GAN_2}, here we have $D_f(P_\theta\|Q)<\infty$ yet the classical $f$-GAN still fails to converge. This suggests that the Lipschitz constraint aids in the stability of the training even in such cases where the classical $f$-divergence is finite. We show the result only for $\alpha=5$ but the behavior for other values of $\alpha$ is similar.
}}\label{fig:2Dstudent_GAN}
   
\end{figure}

\section*{Acknowledgments}
The authors are grateful to Dipjyoti Paul for  providing code to build upon and for helping us to run simulations. The authors also want to acknowledge the anonymous referees for valuable comments, suggestions and insights.  The research of J.B., M.K., and L. R.-B.  was partially supported by NSF TRIPODS  CISE-1934846.   
The research of M. K. and L. R.-B.   was partially supported by the National Science Foundation (NSF) under the grant DMS-2008970
 and by the Air Force Office of Scientific Research (AFOSR) under the grant FA-9550-18-1-0214. 
The research of P.D. was supported in part by the National Science Foundation (NSF) under the grant DMS-1904992 and by the Air Force Office of Scientific Research (AFOSR) under the grant FA-9550-18-1-0214. This work was performed in part using high performance computing equipment obtained under a grant from the Collaborative R\&D Fund managed by the Massachusetts Technology Collaborative.

\bibliography{f_Gamma_v2.bbl}

\end{document}